\documentclass[12pt]{article}
\usepackage[colorlinks=true,linkcolor=blue,citecolor=Red]{hyperref}
\usepackage{breakcites}

\usepackage{mathtools}

\usepackage{subcaption}

\usepackage{algpseudocode}

\usepackage{latexsym}
\usepackage{graphics}
\usepackage{amsmath}
\usepackage[ruled,linesnumbered]{algorithm2e}
\usepackage{xspace}
\usepackage{amssymb}
\usepackage{psfrag}
\usepackage{epsfig}
\usepackage{amsthm}
\usepackage{url}
\usepackage{pst-all}
\DeclareMathOperator*{\argmax}{arg\,max}
\DeclareMathOperator*{\argmin}{arg\,min}

\usepackage{color}

\definecolor{Red}{rgb}{1,0,0}
\definecolor{Blue}{rgb}{0,0,1}
\definecolor{Olive}{rgb}{0.41,0.55,0.13}
\definecolor{Yarok}{rgb}{0,0.5,0}
\definecolor{Green}{rgb}{0,1,0}
\definecolor{MGreen}{rgb}{0,0.8,0}
\definecolor{DGreen}{rgb}{0,0.55,0}
\definecolor{Yellow}{rgb}{1,1,0}
\definecolor{Cyan}{rgb}{0,1,1}
\definecolor{Magenta}{rgb}{1,0,1}
\definecolor{Orange}{rgb}{1,.5,0}
\definecolor{Violet}{rgb}{.5,0,.5}
\definecolor{Purple}{rgb}{.75,0,.25}
\definecolor{Brown}{rgb}{.75,.5,.25}
\definecolor{Grey}{rgb}{.5,.5,.5}

\newcommand{\E}[1]{\mathbb{E}\left[#1\right]}

\newcommand{\Z}{\mathbb{Z}}

\newcommand{\R}{\mathbb{R}}

\renewcommand{\Z}{\mathbb{Z}}

\newcommand{\inner}[1]{\langle #1 \rangle}

\newcommand{\ip}[2]{\left\langle{#1},{#2}\right\rangle}

\renewcommand{\E}[1]{\mathbb{E}\!\left[#1\right]}
\renewcommand{\R}{\mathbb{R}}
\renewcommand{\Z}{\mathbb{Z}}

\newcommand{\distr}{\stackrel{d}{=}}

\newcommand{\ignore}[1]{\relax}

\newlength\myindent
\newtheorem{theorem}{Theorem}[section]
\newtheorem{remark}[theorem]{Remark}
\newtheorem{lemma}[theorem]{Lemma}

\newtheorem{proposition}[theorem]{Proposition}

\newtheorem{Corollary}[theorem]{Corollary}
\newtheorem{claim}[theorem]{Claim}

\newtheorem{definition}[theorem]{Definition}
\newtheorem{Assumption}[theorem]{Assumption}

\newcommand{\risk}[1]{\widehat{\mathcal{L}}(#1)}


\newcounter{parentnumber}

\usepackage{booktabs} 

\makeatletter
\def\BState{\State\hskip-\ALG@thistlm}
\makeatother

\definecolor{Red}{rgb}{1,0,0}
\definecolor{Blue}{rgb}{0,0,1}
\definecolor{Olive}{rgb}{0.41,0.55,0.13}
\definecolor{Green}{rgb}{0,1,0}
\definecolor{MGreen}{rgb}{0,0.8,0}
\definecolor{DGreen}{rgb}{0,0.55,0}
\definecolor{Yellow}{rgb}{1,1,0}
\definecolor{Cyan}{rgb}{0,1,1}
\definecolor{Magenta}{rgb}{1,0,1}
\definecolor{Orange}{rgb}{1,.5,0}
\definecolor{Violet}{rgb}{.5,0,.5}
\definecolor{Purple}{rgb}{.75,0,.25}
\definecolor{Brown}{rgb}{.75,.5,.25}
\definecolor{Grey}{rgb}{.5,.5,.5}
\definecolor{Pink}{rgb}{1,0,1}
\definecolor{DBrown}{rgb}{.5,.34,.16}
\definecolor{Black}{rgb}{0,0,0}

\usepackage{float}

\usepackage{float}
\author{
{\sf Matt Emschwiller}\thanks{MIT; e-mail: {\tt ematt@mit.edu}}
\and 
{\sf David Gamarnik}\thanks{MIT; e-mail: {\tt gamarnik@mit.edu}. Research supported  by the NSF grants CMMI-1335155.}
\and
{\sf Eren C. K{\i}z{\i}lda\u{g}\thanks{MIT; e-mail: {\tt kizildag@mit.edu}}}
\and
{\sf Ilias Zadik}\thanks{NYU; e-mail: {\tt zadik@nyu.edu}. Research supported by a CDS Moore-Sloan Postdoctoral Fellowship.}
}

\begin{document}

\title{Neural Networks and Polynomial Regression.\\ Demystifying the Overparametrization Phenomena}
\date{}

\maketitle

\begin{abstract}
In the context of neural network models, overparametrization refers to the phenomena 
whereby these models appear to generalize 
well on the unseen data, even though the number of parameters significantly exceeds the sample sizes, and the model
perfectly fits  the in-training sample achieving zero in-training error. 
A conventional explanation of this phenomena 
is based on self-regularization properties of algorithms used to train the data. Namely, the algorithms implemented to minimize the training
errors ``prefer'' certain kinds of solutions which generalize well.
In this paper we prove a series of results which provide a somewhat  diverging explanation. Adopting a teacher/student model
where the teacher network is used to generate the predictions and student network is trained on the observed labeled data, and then tested on out-of-sample data, we
show that \emph{any} student network interpolating the
data generated by a teacher network generalizes well, provided that the sample size is at least an explicit quantity controlled by the data dimension and approximation
guarantee alone, \emph{regardless} of the number of internal nodes of either teacher or student network.

Our claim is based on approximating both teacher and student networks by polynomial (tensor) regression models with degree 
which depends on the desired accuracy and network depth only. Such a parametrization  notably does not depend on the number of 
internal nodes. Thus a message implied
by our results is that parametrizing wide neural networks by the number of hidden nodes is misleading, and a more fitting 
measure of parametrization complexity is the number of 
regression coefficients associated with tensorized data.  In particular, this somewhat reconciles the generalization ability of  
neural networks with  more 
classical statistical notions of data complexity and generalization bounds. Our empirical results on MNIST and Fashion-MNIST datasets indeed confirm that tensorized regression achieves a good out-of-sample performance, even when the degree of the tensor is at most two.
\end{abstract}
\pagebreak 
\tableofcontents
\pagebreak

\section{Introduction}
Success of modern learning methods such as neural networks stands in the face of conventional statistical wisdom, which perhaps 
can be summarized informally as the 
following ``commandment'': ``Thou shalt not overfit the training data, lest one does not suffer from a significant generalization
error...''. The statistical rule above is based on very solid practical and theoretical foundations, such as VC-dimension theory,
and is based on the idea that when the dimension
of the model significantly exceeds the training data size,  the overfitting is bound to happen whereby the model picks on idiosyncrasies
of the training data itself. As an example consider a linear regression model where the number of features, namely the dimension of the regression vector exceeds the sample size. Unless some additional regularization assumptions are adopted such as, for example, the sparsity of the regression vector - a common assumption in the literature on compressive sensing - the model is ill-posed, the in-training error minimization process will return a vector perfectly fitting the observed data, but its prediction quality, namely the generalization error will be poor. 
As a result, the  model which has overfitted the training data should generalize
poorly on the unseen data. 

A growing body of recent literature, however, has demonstrated exactly the opposite effect for a broad class of neural network models: 
even though the number of parameters, such as the number
of hidden units of a neural network significantly exceeds the sample size, and a perfect (zero) in-training error is achieved (commonly called
data interpolation), the generalization ability of neural networks remains good. Some partial and certainly very incomplete references to this
point are found in~\cite{belkin2019reconciling}, \cite{gunasekar2018implicit}, \cite{goldt2019dynamics}, \cite{du2018gradient}, \cite{arora2019exact}, \cite{li2018learning}, and \cite{azizanstochastic}. Furthermore, defying  statistical intuition, it was established empirically in~\cite{belkin2019reconciling}
that beyond a certain point increasing the number of parameters results  in ever increasing out of sample accuracy.

Explaining this conundrum is arguably one of the most vexing current problem in the field of machine learning. A predominant
explanation of this phenomena which has emerged recently is based on the idea of self-regularization. Specifically, it is argued that even though there is an abundance
of parameter choices perfectly fitting (interpolating) the data and thus achieving zero in-training error, the algorithms used in training
the models, such as Gradient Descent and its many variants such as Stochastic Gradient Descent, Mirror Descent, etc. tend to find solutions
 which are regularized according to some additional criteria, such as small norms, thus introducing algorithm dependent inductive bias. 
The use of these  solutions for model building in particular
is believed to result in low generalization errors. Thus a significant research effort was devoted to the analysis of the end results of the implementation
of such algorithms. While  we do not debate the merits of this explanation, we establish in this paper a series of theoretical results, supported
by empirical tests on real life data, which offer a somewhat different explanation. Roughly speaking we establish that for neural network models 
based on common choices of activation functions, including  ReLU and Sigmoid functions, provided the sample size is at least a certain
value $N^*$ which is entirely controlled by the data dimension $d$ and target generalization error $\epsilon$, \emph{every} network perfectly interpolating
the training data results in at most $\epsilon$ generalization error, regardless of the complexity of model such as the number of internal nodes.
The dependence of $N^*$ on the data dimension $d$, depth $L$ of the network, and target error $\epsilon$ is of the form $d^{O\left(\left(L/\sqrt{\epsilon}\right)^{L}\right)}$,
for the ReLU activation function, and $d^{O\left(\log\left(L/\epsilon\right)^{L}\right)}$ for the Sigmoid activation function.
While this scaling is certainly poor (doubly exponential
in network depth $L$), it is in line of similar bounds found elsewhere in the literature (see below). It is also worth noting that several publications assume population based training, namely effectively infinite sample size. For neural networks with one hidden layer in particular,
the dependence is of more benign  form $d^{O\left(1/\epsilon\right)}$ and $d^{O\left(\log\left(1/\epsilon\right)\right)}$ for ReLU and Sigmoid, respectively.

\subsubsection*{Summary of results}
We now describe our results in more level of details. First 
in order to address  directly the issue of ``over'' in the ``overparametrization'' concept, we employ the framework of teacher-student setup where one
network (teacher) is used to generate the labels $Y$ associated with data $X$, and another network (student), which potentially has a significantly
larger number of parameters (hence overparametrization), is trained on the observed data, and its performance on the unseen data is being judged.
A very similar teacher-student neural network model was used recently in~\cite{goldt2019dynamics} for the purpose of the analysis of the performance
of the Gradient Descent algorithm. Suppose now $(X^{(i)},Y_i)\in \R^d\times \R, 1\leqslant i\leqslant N$ is the training data set 
with labels $Y$ obtained from a ground truth
$L$-layer \emph{teacher} neural network $f(\mathcal{W}^*,X)$ of width $m^*$. Thus the dimension of the parameter space of the teacher network is $(L-1)(m^*)^2+dm^*+m^*$. 
We assume that the data $X^{(i)}$, $1\leqslant i\leqslant N$ is an i.i.d.\ sequence generated according to a distribution on $\R^d$ corresponding to a product distribution
with support $[-1,1]$. Namely, for each $i$, the sequence $X^{(i)}=(X_{j}^{(i)}, 1\leqslant j\leqslant d)$ is i.i.d. with support $[-1,1]$. The bound $1$ on the support
is adopted for convenience and can be replaced by any fixed constant. The rows of the matrices $W^*$ constituting $\mathcal{W}^*$ are assumed to be normalized to the unit norm, except for  the last layer. This is without the loss of generality for the case of homogeneous activations such as ReLU, since the parameter
scales in this case can be ``pushed" to the last layer $L$. The last layer $W^*_{L}$ is just an $m^*$-vector, which instead we denote by $a^*\in\R^{m^*}$.
We assume that the norm of $a^*$ is bounded by one, and again any constant bound on the norm of $a^*$ can be adopted instead. 
With these assumptions
it is easy to verify that the norm of the label $Y=f(\mathcal{W}^*,X)$ is also constant $O(1)$ and does not grow with either $d$ or $m^*$. This is a reasonable
assumption as $Y$ is the predicted label of the model denoting the inherent quantity which is being predicted, such as a classification of an image 
and thus should not scale with the size of the network parameters such as $m^*$ or $L$.

Suppose now a \emph{student} network $f(\widehat{\mathcal{W}},X)$ with the same number of layers $L$,
but with potentially different and possibly much larger width $\widehat{m}$ than $m^*$
interpolates the data $(X^{(i)},Y_i), 1\leqslant i\leqslant N$, so that $f(\widehat{\mathcal{W}},X^{(i)})=f(\mathcal{W}^*,X^{(i)})$ for all $i$. Assume the rows of (the matrices constituting) $\widehat{\mathcal{W}}$ are also normalized
to be unit length, with motivation provided above. Our first result assumes that in addition the last layer of both the teacher and student network satisfy $\|a^*\|_1,\|\widehat{a}\|_1\leqslant 1$. 
The upper bound $1$ here is chosen again for convenience and can be replaced by any constant. 
The result, then, states that the generalization error of the student network is at most $\epsilon$ provided the sample
size $N$ is at least $N^*=d^{O\left(\left(L/\sqrt{\epsilon}\right)^{L}\right)}$ for the case of ReLU activation function 
and $N^*=d^{O\left(\log^{L}\left(L/\epsilon\right)\right)}$ for the case of sigmoid activation function. 
For the case of one hidden layer teacher/student network these again translate to sample complexities of $d^{O\left(1/\sqrt{\epsilon}\right)}$,
and $d^{O\left(\log\left(1/\epsilon\right)\right)}$, respectively.
In particular, the network generalizes well, \emph{regardless} of the number of hidden nodes $m^*$ and $\widehat{m}$; 
provided that the sample size is at least a quantity $N^*$ which depends only on the data dimension $d$, network depth $L$,  and target accuracy level $\epsilon$.
This is the subject of Theorem~\ref{thm:demyst-overparam-regular}. 

An important corollary of our result is that not only the overparametrized network generalizes well,
but it does so \emph{regardless} of the solution (parameters of the student network) achieving the interpolation. 
This somewhat contradicts the conventional wisdom  based on self-regularization of  interpolating solutions idea discussed above, 
as per our result \emph{every} interpolating solution generalizes well. 
In fact in order to stress this point, we note that 
the following result holds as well, which is a corollary of our main result described above: even if one forces the value $\widehat{m}$ to be arbitrarily large, 
thus forcing the total norm of the interpolating solution to be at least 
$L\widehat{m}$, the resulting generalization error is still small. Namely, the explanation of the overparametrization phenomena based on choosing say
''small norm'' solutions $\widehat{W}$ does not seem to apply in this case. This is to say that small generalization error for large $\widehat{m}$ cannot be explained
say by turning some of activation weights of $\widehat{W}$ to zero thus effectively reducing the width $\widehat{m}$ of the network. In our setting the norms
of (rows of) $\widehat{W}$ are restricted to be unit values. Likewise, ''small norm'' explanation cannot be explained say by having the norm of the last layer
$\widehat{a}$ of the student network being small: an inspection of our results reveals that even if one forces the norm of $\widehat{a}$ to be any fixed constant the result regarding the small generalization error still holds. 

An interesting question arises as to what extent the interpolation property itself places a restriction on the student network. We partially address
this question for the case of neural networks with one hidden layer ($L=1$) and homogeneous activation functions in 
the following result, which is summarized in Theorem~\ref{thm:interpolant-norm-poly-bd}. When sample size is at least of the order
$\exp(O(d\log d))$, and the weights of the outward layer $\widehat{a}$ of the student network is coordinate-wise non-negative, 
the norm of $\widehat{a}$ is at most $d^{O(1)}$ times the norm of $a^*$ of the teacher network. 
A more precise bound is $O(d^{2})$ for the case of ReLU activation function. As a consequence, our result regarding the generalization ability
of a student network holds regardless of the norm $\widehat{a}$ of the outward layer of the student network. 
This result is summarized in Theorem~\ref{thm:overparam:ReLU-case}. Admittedly this holds under a more stringent
assumption on the sample size, in particular it requires a sample size which is exponential in the dimension $d$, as opposed polynomial in $d$ size when 
the norm of $\widehat{a}$ is bounded by a constant. We note that non-negativity assumption on $\widehat{a}$ is necessary in some sense as without
it one can effectively mimic the teacher network with many student networks by playing with the signs of $a^*$ and introducing many cancellations. 

\subsubsection*{Proof techniques: Polynomial (tensor) regression approximation}
Our approach is based on using the polynomial regression method with carefully  controlled degree as a model relating data $X$ to labels $Y$. Specifically, fixing 
any degree parameter $M$, we consider the model 
\begin{align*}
Y=\sum_{0\leqslant r\leqslant M}\sum_{ \alpha \in \mathbb{N}^d: \sum_{i=1}^d\alpha_i=r}q_{\alpha}\prod_{i=1}^d X^{\alpha_i}_i+\eta,
\end{align*}
where $q_{\alpha}$ are associated regression coefficients and $\eta$ is a noise parameter.  Alternatively,
one can think of this model as tensorized linear regression (linear regression on tensorized data $X^{d\otimes M}$, 
or kernel regression with a trivial kernel). Our main technical result is based on showing that polynomial regression model above serves as a good
approximation model of a neural network model with $L$ layers both for the in-training error and generalization error, provided
that the degree $M$ is of a fixed order which depends on the dimension $d$ and the target generalization error guarantee $\epsilon$. 
In particular, it is $M=O\left(\left(L/\sqrt{\epsilon}\right)^{L}\right)$ for ReLU and $M=O\left(\log^{L}\left(L/\epsilon\right)\right)$ for 
the Sigmoid activations. 

More specifically, we establish the following result. Suppose $(X^{(i)},Y_i), 1\leqslant i\leqslant N$ are
generated according to a neural network with $L$ layers and the same assumptions as above. Namely, the $\ell_1$-norms of internal node weights $W_{j,\ell}$ 
are at most unity, the outward layer $a$ has $\ell_1$-norm at most unity, and the data $X_i$ are i.i.d.\ with i.i.d.\ components supported in $[-1,1]$. 
Fix $M$ and let $q=(q_{\alpha})_{ \alpha \in \mathbb{N}^d, 1\leqslant \sum_{i=1}^d \alpha_i \leqslant M}$ be a solution to the mean squared error minimization problem associated
with $(X_i,Y_i), 1\leqslant i\leqslant N$:
\begin{align*}
\min_\alpha \sum_{i=1}^N\left(Y_i-\sum_{0\leqslant r\leqslant M}\sum_{ \alpha \in \mathbb{N}^d: \sum_{i=1}^d\alpha_i=r}q_{\alpha}\prod_{j=1}^d \left(X_i\right)^{\alpha_j}_{j}\right)^2.
\end{align*}
This solution $\alpha$ can be used for the out-of-sample prediction for any new data $X$ in an obvious way: 
\begin{align}\label{eq:poly-regression}
\widehat{Y}=\sum_{0\leqslant r\leqslant M}\sum_{ \alpha \in \mathbb{N}^d: \sum_{i=1}^d\alpha_i=r}q_{\alpha}\prod_{i=1}^d X^{\alpha_i}_i
\end{align}
which can be compared with the ground truth prediction $Y^*=f(X,\mathcal{W}^*)$. We call the difference $(\widehat{Y}-Y^*)^2$ the (squared) generalization error. We establish that
for every $\epsilon>0$, when the degree $M$ is bounded as above in terms of $d$ and $\epsilon$, and the sample size is $N=\Omega \left( d^M \right)$, the expected error (with expectation wrt distribution of a fresh sample $X$)
is at most $\epsilon$, with high probability with respect to the training sample $(X_i, 1\leqslant i\leqslant N)$. The proof of this result is 
fairly involved, and in particular is based on a) analyzing instead a polynomial orthogonal to the tensors of the data in an appropriate sense, and b) bounding the condition number of the expected $O(d^M)\times O(d^M)$ sample covariance matrix of the tensorized data $X^{d\otimes M}$, a result of potential independent interest in the study of random tensors \cite{vershynin2019concentration}.

A fundamental feature of approximating by the polynomial regression model above is that it does not involve the width $m^*$ of the network, suggesting that, at least when the network is wide enough (large $m^*$), its width is not the right measure of complexity of the learning class, and rather dimension $d$ and polynomial degree $M$ are better measures of complexity. This contrasts with the VC-dimension theory which bounds the complexity by $L(m^*)^2$ (roughly).  However, when the network is not too wide, in particular when $m^*$ is small compared to $d$, the VC-dimension can still be a more fitting measure of complexity. Identifying more precisely as to when VC-dimension stops describing the true model complexity is an interesting open question. 
An independent contribution of our paper is an improved understanding of the strength of the polynomial regression method as a learning model to begin
with. In some sense our results show that the power of the regression method is no less at least theoretically than the power 
of neural networks. 
In order to test this hypothesis in practice we have conducted a series of experiments on MNIST and Fashion-MNIST datasets. These are reported in Section \ref{appendix:realdata}. Due to the explosion of the number of coefficients in the polynomial regression model, and inspired by the practice of neural networks, we have adopted a \emph{convolutional} version of polynomial regression where only certain contiguous 
products of variables corresponding to near-by pixels of an image are used. 
As demonstrated in this section, the polynomial regression method even with degree at most two (involving products of at most two variables), achieves very high-quality prediction error only about 3\% smaller than typical convolutional neural networks used for those tasks. 
An advantage of the polynomial regression model as a learning model is interpretability: the regression coefficients $q_{\alpha}$ 
are amenable to easy interpretation which is not available in the context of neural networks. 

\subsubsection*{Prior literature on tensor regression methods}
The idea of approximating neural networks by polynomial regression, more commonly either by tensor regression or kernel regression,
has been very popular recently. We list here only some of the relevant results. Tensor regression methods were proposed by \cite{janzamin2015beating}, under the assumption that the distribution of the data $X$ has density
which is differentiable up to certain orders. An upper bound on the width of the network also had to be assumed for some of the results. 
\cite{goel2017learning} 
use approximation by polynomials similar to our polynomial regression method to learn $L=1$ hidden layered neural network in polynomial 
time and more general monotone non-linear outward layer assumption. The guarantee bound we obtain when specialized for $L=1$ case are similar
to the ones in~\cite{goel2017learning}. Recently a great deal of attention has been devoted to the so-called neural tangent kernel 
method \cite{jacot2018neural}, \cite{arora2019exact}, \cite{liang2019risk}. Here the motivation is to understand the dynamics of the gradient descent
algorithm on training neural network when initialized with Gaussian distribution on infinitely wide network. Tensor regression estimates somewhat
similar to ours are obtained in \cite{liang2019risk}, where the sampling size $n$ is of the form $d^a$ with $a>1$.
The invertibility of the sample covariance of the tensorized data vector $X$ is also analyzed.
For technical reasons the sampling size has to be also at most this value and also the additional assumption is that the coefficients associated with
first $k$  order Taylor expansion are non-negative, which in our setting would correspond to non-negativity of the coefficients $q_{\alpha}$,
something we do not assume and something that need not hold for general neural networks. 
\subsubsection*{Notation}

The set of real numbers, non-negative real numbers, and natural numbers are denoted respectively by $\R$, $\R_{\geqslant 0}$, and $\mathbb{N}$ (or $\mathbb{Z}^+$). For any $k\in\mathbb{N}$, the set $\{1,2,\dots,k\}$ is denoted by $[k]$. For any matrix $A$, $\sigma_{\min}(A)$ denotes its smallest singular value, and $\sigma_{\max}(A)$ (or $\|A\|_2$, or $\|A\|$) denotes its largest singular value.  
Given any $v\in\R^n$, its Euclidean $\ell_p-$norm $(\sum_{i=1}^n |v_i|^p)^{1/p}$ is denoted by $\|v\|_p$ (or $\|v\|_{\ell_p}$).  For any two vectors $x,y\in \R^n$, their Euclidean inner product, $\sum_{i=1}^n x_i y_i$, is denoted by $\ip{x}{y}$. 
For any set $S$, its cardinality is denoted by $|S|$.  $\Theta(\cdot),O(\cdot)$, $o(\cdot)$, and $\Omega(\cdot)$ are standard (asymptotic) order notations. For a distribution $P$ on the real line, the smallest (in terms of inclusion) closed set $S\subseteq \R$ with $\mathbb{P}(X_i\in S)=1$ (where $X_i$ is distributed according to $P$) is called the support of $P$, denoted by $S={\rm Support}(P)$. For any distribution $P$, $P^{\otimes d}$ is a product distribution, i.e. the distribution of a $d$-dimensional random vector with i.i.d.\ coordinates from $P$. ReLU ($\max\{x,0\}$) and sigmoid ($(1+\exp(-x))^{-1}$) functions are two commonly used activation functions in the literature on neural networks.
\subsubsection*{Paper Organization}
The remainder of the paper is organized as follows. Preliminaries, in particular the activation functions and distributions we consider and the tensor notation we adopt, are the subject of Section \ref{sec:preliminaries}. The algorithm we propose for learning such networks and its guarantees are provided in Section \ref{sec:PolyReg}. The overparametrization phenomena (in the context of teacher/student networks), and the self-regularity of the interpolants are addressed in Section \ref{sec:overParam-Gen}. Our empirical findings are reported in Section \ref{appendix:realdata}. Our results on the condition number of a certain expected covariance matrix, which could be of independent interest, are provided in Section \ref{sec:bound-cond-no}. Finally, the complete proofs of our results are found in Section \ref{sec:pfs-main}.

\section{Model and Preliminaries}\label{sec:preliminaries}
\paragraph{Model.}
Neural network architectures are described by a parameter set
\begin{equation}\label{eq:deepnn-weights}
\mathcal{W}^*= \{a^*\} \cup \{W_p^*:1\leqslant p\leqslant L\}
\end{equation}
where $a^*\in \R^{m^*}$; $W_1^*\in\R^{m^*\times d}$, and $W_p^*\in\R^{m^*\times m^*}$ for any $2\leqslant p\leqslant L$. Namely, $\mathcal{W}^*$ is the set of weights, $m^*\in\mathbb{N}$ is the width of each internal layer; $L\in\mathbb{N}$ is the depth of the network, and $d\in\mathbb{N}$ is the dimension of input data. Here, the $j^{th}$ row of $W_p^*$, denoted by $W_{p,j}^*$, carries the weights associated to $j^{th}$ neuron at $p^{th}$ layer. 

\noindent Such a network, for each input $X\in\R^d$, computes the function
\begin{equation}\label{eq:deepnn-main}
f(\mathcal{W}^*,X) = (a^*)^T \sigma(W_L^* \sigma(W_{L-1}^*\cdots \sigma(W_1^* X))\cdots),
\end{equation}
where  the activation function $\sigma(\cdot)$ acts coordinate-wise. In particular, when $L=1$, this model corresponds to a model with one hidden layer, also known as the {\em shallow} neural network model.

In literature, this model is also known as the ``realizable model"; since the ``labels" $f(\mathcal{W}^*,X)$ are generated using a neural network with planted weights $\mathcal{W}^*$.

\paragraph{Admissible activation functions and distributions} 
Our results hold for a large class of activation functions, provided that either they are polynomials or they can be ``well-approximated" by polynomials. To formalize the latter notion, we restrict ourselves to the case where the domain of the activations is the compact interval $[-1,1]$. 
\begin{definition}\label{def:approximation-order}
Let $\sigma:[-1,1]\to \R$ be an arbitrary function, and $\mathbb{R}[x]$ be the set of all polynomials with real coefficients. Define 
$$
S(r,\sigma,\epsilon)= \left\{P(x)\in\R[x]:{\rm deg}(P)=r,\sup_{x\in[-1,1]} |P(x)-\sigma(x)|\leqslant \epsilon\right\}.
$$
The function $\varphi_{\sigma}:\R^+\to\Z^+\cup\{\infty\}$, defined by $\varphi_{\sigma}(\epsilon) = \inf\{r\geqslant 1:S(r,\sigma,\epsilon)\neq\varnothing\}$
is called the {\bf order of approximation} for $ \sigma(\cdot)$ over $[-1,1]$, where $\varphi_\sigma(\epsilon)=+\infty$, if $\bigcup_{r\in\mathbb{Z}^+} S(r,\sigma,\epsilon)=\varnothing$.
\end{definition}
Namely, for any $\epsilon>0$, $\varphi_\sigma(\epsilon)$ is the smallest possible degree $d$ required to uniformly approximate $\sigma(\cdot)$ over $[-1,1]$ by a polynomial of degree $d$ to within an additive error of at most $\epsilon$. By the Stone-Weierstrass theorem \cite{rudin1964principles}, $\varphi_{\sigma}(\epsilon)<\infty$ for every $\epsilon>0$, provided $\sigma$ is continuous.   

The activations that we consider are the {\bf admissible} ones, which we formalize next.
\begin{definition}\label{def:admissible-activation}
An activation function $\sigma:[-1,1]\to\R$ is called {\bf admissible} if: ${\rm (i)}$ for every $\epsilon>0$, $\varphi_\sigma(\epsilon)<\infty$, ${\rm (ii)}$ for every $\epsilon>0$ there exists a polynomial $P(x)\in S(\varphi_\sigma(\epsilon),\sigma,\epsilon)$ such that $P([-1,1])\subseteq [0,1]$, ${\rm (iii)}$ $\sigma(\cdot)$ is 1-Lipschitz, that is, for every $x,y\in[-1,1]$, $|\sigma(x)-\sigma(y)|\leqslant |x-y|$, and ${\rm (iv)}$ $\sup_{x\in[-1,1]}|\sigma(x)|\leqslant 1$.
\end{definition}
It turns out that, this definition encompasses popular choices of activation functions.
\begin{proposition}\label{prop:activation-admissible}
ReLU and sigmoid activations are both admissible, as per Definition \ref{def:admissible-activation}. The corresponding order of approximation is $O(1/\epsilon)$ for the ReLU, and $O(\log(1/\epsilon))$ for the sigmoid.
\end{proposition}
The proof of Proposition \ref{prop:activation-admissible} is provided in Section \ref{sec:pf-prop:activation-admissible}. 

In this paper, we consider {\bf admissible} input distributions,  formalized next.
\begin{definition}\label{def:admissible-distribution}
A random vector $X=(X_i:1\leqslant i\leqslant d)\in[-1,1]^d$, is said to have an {\bf admissible} input distribution, if the coordinates $X_i$ are i.i.d.\ samples of a non-constant distribution $P$. 
\end{definition}
In particular, the uniform distribution on the cube $[-1,1]^d$ and on the Boolean cube $\{-1,1\}^d$ are both {\bf admissible}. The rationale for considering distributions on the bounded domain $[-1,1]^d$ is to be consistent with the fact that the domain of activations we consider are also restricted to $[-1,1]$.  

\paragraph{Improper learning.}
We now discuss the learning setting. Suppose the data $X^{(i)}$, $1\leqslant i\leqslant N$, is drawn from a certain distribution supported on a domain $S^d\subseteq \mathbb{R}^d$, and the corresponding ``labels" $Y_i$, $1\leqslant i\leqslant N$, are  generated per (\ref{eq:deepnn-main}),  that is $Y_i=f(\mathcal{W}^*,X^{(i)})$. Given the knowledge of training data, $\{(X^{(i)},Y_i):i\in[N]\}$, the goal of the learner is to output a function $\widehat{h}(\cdot):S^d\to \R$, called a {\em hypothesis} (or {\em predictor}), for which the ``distance" between the true labels $f(\mathcal{W}^*,X)$  and their predicted values $\widehat{h}(X)$ is small, namely, what is known as the testing error
\begin{equation}\label{eq:test-error}
\mathcal{L}\left(\widehat{h}(\cdot),f(\mathcal{W}^*,{}\cdot{})\right)\triangleq\mathbb{E}\left[\left(\widehat{h}(X)-f(\mathcal{W}^*,X)\right)^2\right]
\end{equation}
remains small \cite{shalev2014understanding}, where the expectation is taken with respect to a fresh sample drawn from the same distribution generating $X^{(i)}$. The testing error is also commonly called as the generalization error; and in this paper we refer to it as \textbf{generalization error} as well. Our focus is on the so-called {\em improper learning} setting, that is, the output hypothesis $\widehat{h}(\cdot)$ need not be a neural network.

We now provide the tensor notation that we use throughout. 
For any $s,t\in\mathbb{N}$ with $1\leqslant s\leqslant t$, define the set of multiplicities
\begin{align}\label{fst}
\mathcal{F}_{s,t}:=\left\{(\alpha_1,\dots,\alpha_d)\in\mathbb{Z}^d:\sum_{i=1}^d \alpha_i=t;\,0\leqslant \alpha_i\leqslant s\text{ for all $i$} \right\}.
\end{align} The dependence of $\mathcal{F}_{s,t}$ on $d$ is suppressed for convenience. The following (elementary) observation will be proven useful to later sections of the paper.
\begin{remark}\label{rem:fst-empty}
It holds $\mathcal{F}_{s,t}\not =\emptyset$ if and only if $t \leq ds.$ 
\end{remark}

Given any $v\in\mathbb{R}^d$, and $\alpha\in\mathbb{Z}^d$, define
$T_\alpha(v) \triangleq \prod_{i=1}^d v_i^{\alpha_i}$. We refer to
\begin{equation}\label{eq:st-tensor}
T(v)=(T_\alpha(v):\alpha\in\mathcal{F}_{s,t})\in\R^{|\mathcal{F}_{s,t}|}
\end{equation} as an $(s,t)-$tensor. 

Finally, we collect below a set of assumptions that we use frequently in our results. 
\begin{Assumption}\label{sump:main}
\begin{itemize}
    \item $X^{(i)}\in\R^d$, $1\leqslant i\leqslant N$, is an i.i.d. sequence of data generated by an admissible distribution $\mathcal{D}=P^{\otimes d}$ per Definition \ref{def:admissible-distribution}, where $S=\mathrm{Support}(P) \subseteq [-1,1]$.
    \item $\sigma(\cdot)$ is either an admissible activation per Definition \ref{def:admissible-activation} with the corresponding order of approximation $\varphi_\sigma(\cdot)$; or an arbitrary polynomial with real coefficients.
    \item The labels $Y_i$ are generated using a teacher neural network of depth $L$ consisting of activations $\sigma(\cdot)$. That is, $Y_i=f(\mathcal{W}^*,X^{(i)})$, per Equation (\ref{eq:deepnn-main}), for every $i\in[N]$.
\end{itemize}
\end{Assumption}
  This model is referred to as the ``teacher network"; since it generates the training data. 

\section{Main Results}

\subsection{Learning Neural Networks with Polynomial Regression}\label{sec:PolyReg}
In this section, we establish that neural networks can be learned by applying the ordinary least squares (OLS) procedure to a polynomial regression setting, provided we have access to sufficiently many training data.

We first describe the reduction to the regression setting. Suppose that $X^{(i)}\in\R^d$, $1\leqslant i\leqslant N$, is a sequence of i.i.d. training data coming from an admissible distribution supported on some (possibly finite) set $S^d$ and $Y_i=f(\mathcal{W}^*,X^{(i)}),$ $1\leqslant i\leqslant N$ is the corresponding label, generated per Equation (\ref{eq:deepnn-main}). We start with the case where the activation is a polynomial function, something we refer to as {\em polynomial networks}. 
We define the following sequence of tensors which are functions of an input vector $X\in\R^d$.

Let $S$ be an (possibly finite) input set, $t \in \mathbb{N}$; and $X \in S^d$. The following quantity would be of instrumental importance in the results that follow: \begin{equation}\label{eq:omega-Pt}
\omega(S,t)\triangleq \min\{|S|-1,t\}\in\mathbb{Z}^+.
\end{equation} In the cases where $X$ is drawn from a product distribution $P^{\otimes d}$ with a slight abuse of notation, but with the goal to reduce notational complexity, we make use of the notation
\begin{equation}\label{eq:omega-Pt-P}
\omega(P,t)\triangleq \omega(\mathrm{Support}(P),t) =\min\{|\mathrm{Support}(P)|-1,t\}\in\mathbb{Z}^+.
\end{equation}

Now let us define the $(\omega(S,t),t)$-tensor $\Xi^{(t)}(X) \in \mathbb{R}^{|\mathcal{F}_{\omega(S,t),t}| }$, as per Equation (\ref{eq:st-tensor}), where 
\begin{align}\label{eq:tensor-Dfnksi}
\Xi^{(t)}_{\alpha}\left( X \right) \triangleq \left(X_{1}\right)^{\alpha_1}\left(X_{2}\right)^{\alpha_2}\cdots \left(X_{d}\right)^{\alpha_d},\quad\text{for each}\quad \alpha=(\alpha_i:i\in[d])\in \mathcal{F}_{\omega(S,t),t}.
\end{align}
Namely, $\Xi^{(t)}(X)$ is a vector of $X_1^{\alpha_1}\cdots X_d^{\alpha_d}$, where $\alpha_i\in\mathbb{Z}$; $\sum_{i=1}^d \alpha_i =t$; and $\alpha_i\leqslant |S|-1$.

In the case where the activation function is a polynomial of degree $k$, it is easy to see (details are in Lemma \ref{lem:poly2}) that when $X \in S^d$  that there exists a sequence of vectors $\mathcal{T}^{(t)} \in \mathbb{R}^{|\mathcal{F}_{\omega(S,t),t}| }
$, indexed by $t\in[k^L]$, for which 
 \begin{equation}\label{eq:tensor-3New2}
f(\mathcal{W}^*,X) = \sum_{t=0}^{k^L} \inner{\Xi^{(t)}\left( X \right),\mathcal{T}^{(t)}},\quad \text{for every}\quad X\in S^d.
\end{equation} 
Now, Equation (\ref{eq:tensor-3New2}) indeed demonstrates that for the case of polynomial networks, the input-output relationship, $f(\mathcal{W}^*,X^{(i)})$, is an instance of a (noiseless) polynomial regression problem: the `measurement vector' consists of the entries of tensors $\Xi^{(t)}\left(X^{(i)} \right)$, $t\in[k^L]$, which, in turn, are obtained from the data; and `coefficients' to be recovered are the entries of $\mathcal{T}^{(t)}$, $t\in[k^L]$. 

For the case where the network consists of any arbitrary admissible activation, the mechanics of the aforementioned reduction is as follows. We establish in Theorem \ref{thm:approximate-deep-with-poly} that any such network can be uniformly approximated by a polynomial network of the same depth, under a certain boundedness assumption on its input and weights. Combining this with the observation recorded in Equation (\ref{eq:tensor-3New2}); we then obtain that the input-output relationship for this case is still an instance of a regression problem with a noise-like term originating from the approximation error. 

More concretely, for learning such networks we propose T-OLS algorithm (Algorithm \ref{algo:tensor-non-poly}),  which is based on the ordinary least squares procedure.
\begin{algorithm}[h!] 
\SetAlgoLined
 \KwIn{Data $(X^{(i)}, Y_i) \in \mathbb{R}^{d } \times \mathbb{R}, 1\leqslant i\leqslant N$, degree $M$ of approximation, cardinality $|S| \in \mathbb{R} \cup \{+\infty\}$ of a set $S$ with $X^{(i)} \in S^d$ for $i=1,2,\ldots,N.$}
 \KwOut{Prediction function $\widehat{h}(\cdot)$. }

For $t\in[M]\cup\{0\}$, $1\leqslant i\leqslant N$, compute
$\Xi^{(t)}\left(X^{(i)}\right)$, based on Equation (\ref{eq:tensor-Dfnksi}).


Solve for $\mathcal{T}^{(t)} \in \mathbb{R}^{|\mathcal{F}_{\omega(S,t),t}|}, $ $t=0,1,2,\ldots,M$ the unconstrained OLS problem $$\min  \sum_{i=1}^N \left(Y_i-
\sum_{t=0}^{M} \left\langle\Xi^{(t)}\left( X^{(i)} \right),\mathcal{T}^{(t)}\right\rangle\right)^2,$$
where the minimization is over $\left(\mathcal{T}^{(t)}:0\leqslant t\leqslant M\right)$.

Output the prediction function $\widehat{h}(x)=\sum_{t=0}^{M} \inner{\Xi^{(t)}\left(x \right),\mathcal{T}^{(t)}}$.


 \caption{Tensor OLS (T-OLS) Algorithm for Learning Neural Networks}
\label{algo:tensor-non-poly}
\end{algorithm}

We are now in a position to state our main results. 

\subsubsection*{Polynomial Networks}
We start with the case of polynomial networks. In this setting, not surprisingly, one can achieve zero generalization error, using T-OLS algorithm.
\begin{theorem}\label{thm:the-main-for-poly}
Let $\sigma(x)=\sum_{t=0}^k\beta_t  x^t$ be an arbitrary polynomial with real coefficients whose degree $k$ is known to the learner; $X^{(i)}$, $1\leqslant i\leqslant N$, be the training data; and $Y_i =f(\mathcal{W}^*,X^{(i)})$ be the corresponding label. 
\begin{itemize}
    \item[(a)] Suppose that Assumption \ref{sump:main} holds. Let $\widehat{h}(\cdot)$ be the output of T-OLS algorithm with inputs $(X^{(i)},Y_i)\in\R^d\times \R$, $1\leqslant i\leqslant N$; $k^L$, and $|S|$ with $|S| = |\mathrm{Support}(P)|$. Then, with probability at least $1-\exp(-c'N^{1/4})$ over $X^{(i)}$, $1\leqslant i\leqslant N$,
    $$
    \widehat{h}(x) = f(\mathcal{W}^*,x),
    $$
    for every $x\in\R^d$, provided 
    \begin{equation}\label{eq:sample-complex-N1}
    N>N_1^*\triangleq d^{24k^L}C^4\left(P,k^L\right).
    \end{equation}
    \item[(b)] Suppose that the coordinates of $X^{(i)}$ are only jointly continuous (and in particular not necessarily admissible). Let $\widehat{h}(\cdot)$ be the output of T-OLS algorithm with inputs $(X^{(i)},Y_i)\in\R^d\times \R$, $1\leqslant i\leqslant N$; $k^L$, and $+\infty$. Then, as soon as 

    \begin{equation}\label{eq:sample-complex-N1-hat}
    N\geqslant \widehat{N}_1^*\triangleq\sum_{t=0}^{k^L}\left|\mathcal{F}_{t,t}\right| = \sum_{t=0}^{k^L}\binom{d+t-1}{t},
    \end{equation}
    it holds
    $$
    \widehat{h}(x) = f(\mathcal{W}^*,x),    
    $$
    for every $x\in\R^d$, with probability one over the randomness of $X^{(i)}$, $1\leqslant i\leqslant N$.
\end{itemize}
\end{theorem}
Here, $c'>0$ is some absolute constant; and $C(\cdot,\cdot)$ is a constant depending only on the input distribution $P$, the depth of the network $L$, and the degree of the polynomial $k$. In particular, it is independent of the dimension $d$ of data. In Section \ref{sec:bound-cond-no}, we show how to calculate $C(P,\cdot)$ exactly: it turns out that for any discrete distribution $P$ and any $k \in \mathbb{N}$, $C(P,k)=\exp(O(k\log k))$, where constant in $O(\cdot)$ depends on the distribution $P$. It is also worth noting that the constant $1/4$ appearing in high probability bound is chosen for convenience, and can potentially be improved.

The implication of Theorem \ref{thm:the-main-for-poly} is as follows. In the case of polynomial networks, it is possible to output a hypothesis $\widehat{h}(\cdot)$ which has zero generalization error provided that the learner has access to $d^{\Theta(k^L)}$ samples. The fact that the generalization error can be driven to zero relies on the following observation: for the case of polynomial networks, the reduction from the problem of learning neural networks to the regression problem, as described above, leads to a noiseless regression setting. An important feature of Theorem \ref{thm:the-main-for-poly} is that the the sample complexity lower bound on $N_1^*$ (and $\widehat{N}_1^*$) does not depend on the width $m^*$ of the internal layers. We exploit this feature later to address the overparametrization phenomena. 

The proof of Theorem \ref{thm:the-main-for-poly} is provided in Section \ref{sec:pf-main-poly-act}. 

\subsubsection*{Networks with Arbitrary Admissible Activations}
Our next main result addresses the problem of learning neural networks with arbitrary admissible activations, and gives the performance guarantee of the T-OLS algorithm.
\begin{theorem}\label{thm:the-main}
Suppose that Assumption \ref{sump:main} holds where $\sigma(\cdot)$ is an admissible activation; and the weights of the teacher network satisfy: $\|a^*\|_{\ell_1}\leqslant 1$, and $\|W_{p,j}^*\|_{\ell_1}\leqslant 1$ for any $p\in[L]$ and $j$. 

Let $\widehat{h}(\cdot)$ be the output of T-OLS algorithm with inputs $(X^{(i)},Y_i)\in \R^d\times\R$, $1\leqslant i\leqslant N$, $M=\varphi_\sigma^L(\sqrt{\epsilon/4L^2})\in\mathbb{N}$; and $|S|$ with $|S| = |\mathrm{Support}(P)|$. Then, 
with probability at least $1-\exp(-cN^{1/4})$
with respect to the training data $(X^{(i)},Y_i)\in\R^d\times \R$, $1\leqslant i\leqslant N$, it holds $$\mathcal{L}\left(\widehat{h}({}\cdot{}),f(\mathcal{W}^*,{}\cdot{})\right)\leqslant \epsilon,$$  where $\mathcal{L}(\cdot,\cdot)$ is defined as per Equation (\ref{eq:test-error}),
provided
\begin{equation}\label{eq:main-sample-complexity}
 N\geqslant N_2^* \triangleq 2^{12}\epsilon^{-6}\exp\left(96\log(d)\varphi_\sigma^L(\sqrt{\epsilon/4L^2})\right) C^{18}\left(P,\varphi_\sigma^L(\sqrt{\epsilon/4L^2})\right).
\end{equation}

\end{theorem}
Here, $c>0$ is an absolute constant.
$C(\cdot,\cdot)$ is a constant depending only on the input distribution $P$, the depth of the network $L$, and the target accuracy $\epsilon>0$, and in particular is independent of the input dimension $d$, as in Theorem \ref{thm:the-main-for-poly}. Details on how to compute $C(\cdot,\cdot)$ are elaborated in Section \ref{sec:bound-cond-no}, and in particular it turns out for any discrete distribution $P$ and $k\in\mathbb{N}$, $C(P,k)=\exp(O(k\log k))$, same as before. Similar to Theorem \ref{thm:the-main-for-poly}, the constant $1/4$ appearing in the high probability bound is chosen for convenience, and can potentially be improved. 

An important remark is that the sample complexity bound $N_2^*$ in (\ref{eq:main-sample-complexity}) does not depend on the width $m^*$ of the internal layers. This will turn out to be of paramount importance for the overparametrization results that follow.

We now highlight the implications of our result in two important cases. Suppose that the activation function $\sigma(\cdot)$ is ReLU for which the corresponding order of approximation is $\varphi_\sigma(\epsilon)=O(1/\epsilon)$ (see Proposition \ref{prop:activation-admissible}). Then, according to our theorem, provided $N$ is roughly of the order $d^{\Theta((L/\sqrt{\epsilon})^L)}$, the generalization error of  $\widehat{h}(\cdot)$ is at most $\epsilon$. Suppose next that the activation function $\sigma(\cdot)$ is sigmoid, for which the corresponding order of approximation is $\varphi_\sigma(\epsilon)=O(\log(1/\epsilon))$.  Then, the generalization error of $\widehat{h}(\cdot)$ is at most $\epsilon$, provided $N$ is roughly of the order $d^{\Theta(\log(L/\epsilon)^L)}$.
The proof of Theorem \ref{thm:the-main} is deferred to Section \ref{sec:pf-thm-main}. 
Before we close this section, we make several remarks. An inspection of Theorems \ref{thm:the-main-for-poly} and \ref{thm:the-main}  reveals that the learner needs only to know the data $(X^{(i)},Y_i)$, the degree  $M$ of approximation; and the cardinality of the superset $S$, as required by the T-OLS algorithm. In particular, the learner does not need the exact knowledge of the underlying architecture and the activations.
An upper bound on the depth $L$ and the degree of approximation $M$ suffice. The latter, for the case of polynomial networks, translates as an upper bound on the degree of activation; and for the case of networks with admissible activation translates as an upper bound on the order of approximation $\varphi_\sigma(\cdot)$. 

Yet another remark, regarding the architecture generating the training data, is the following. The hidden layers need not have the same width. Moreover, the activations need not be the same across all nodes. For instance, the thesis of Theorem \ref{thm:the-main-for-poly} for the polynomial networks still remains valid, if each neuron is equipped with a polynomial activation, not necessarily the same, of degree at most $k$. Similarly, the thesis of Theorem \ref{thm:the-main} for the general networks still remains valid, even when each neuron is equipped with an (admissible) activation $\sigma^{i,j}$ (where  $i$ ranges over the depth, and $j$ ranges over the width), not necessarily the same across all neurons: in this case, it suffices to have $\varphi_\sigma(\cdot)$ as a pointwise upper bound on the functions $\varphi_{\sigma^{i,j}}(\cdot)$, that is, for every $\epsilon>0$, $i\in[L]$, and $j\in[m^*]$, it suffices to ensure $\varphi_\sigma(\epsilon)\geqslant \varphi_{\sigma^{i,j}}(\epsilon)$. 
The proofs can be adapted to these more general cases with straightforward modifications. We do not, however, pursue these herein for the sake of simplicity.

\subsection{Overparametrization vs Generalization}\label{sec:overParam-Gen}
In this section, we study the interplay between overparametrization and generalization in the context of teacher/student networks. In short, we establish that any student network which is arbitrarily overparametrized with respect to the teacher network obtains arbitrarily small generalization error, provided it interpolates enough data, regardless of how the training is done. 

Suppose that $(X^{(i)},Y_i)\in\R^d\times \R$, $1\leqslant i\leqslant N$, is a sequence of data, where $X^{(i)}$ are i.i.d.\ samples of an admissible distribution per Definition \ref{def:admissible-distribution}; and the labels $Y_i$ are generated by a ``teacher" network of depth $L$, and width $m^*$, according to Equation (\ref{eq:deepnn-main}). For convenience, we refer to this network as $\mathcal{N}_1$, and denote $Y_i = f_{\mathcal{N}_1}(\mathcal{W}^*,X^{(i)})$. We now introduce the student network. Fix an arbitrary positive integer $\widehat{m}\geqslant m^*$. Let $\widehat{\mathcal{W}}$ be the set of weights of a wider ``student" network $\mathcal{N}_2$ of depth $L$ and width $\widehat{m}$, trained on $(X^{(i)},Y_i)$, $1\leqslant i\leqslant N$. In this case, if 
$f_{\mathcal{N}_1}(\mathcal{W}^*,X^{(i)})=f_{\mathcal{N}_2}(\widehat{\mathcal{W}},X^{(i)})$, for every $i\in[N]$, we say that the student network interpolates the data.

\subsubsection*{Networks with Polynomial Activations}
We start with the case of {\em polynomial networks}, and establish that any overparametrized ``student" network enjoys zero generalization error (for predicting a ``teacher" network), provided it interpolates a sufficient amount of data.  
\begin{theorem}\label{thm:overly-general-poly}
Suppose that the assumptions of Theorem \ref{thm:the-main-for-poly} hold. 
Fix a positive integer $\widehat{m}\geqslant m^*$, and let $\widehat{\mathcal{W}}$ be the set of weights of a student network $\mathcal{N}_2$ consisting of the same polynomial activations, with width $\widehat{m}$ and depth $L$ interpolating the data; namely $Y_i=f_{\mathcal{N}_2}(\widehat{\mathcal{W}},X^{(i)})$, $i\in[N]$. Then, we have the following.
\begin{itemize}
    \item[(a)] With probability at least $1-2\exp(-c'N^{1/4})$ over $X^{(i)}$, $1\leqslant i\leqslant N$ it holds:
    $$
    f_{\mathcal{N}_1}(\mathcal{W}^*,x)=f_{\mathcal{N}_2}(\widehat{\mathcal{W}},x),
    $$
    for every $x\in\R^d$, provided $N>N_3^* \triangleq d^{24k^L}C^4(P,k^L)$.
    \item[(b)] If, in addition, $X^{(i)}$ has jointly continuous coordinates, then with probability one over $X^{(i)}$, $1\leqslant i\leqslant N$, it holds
    $$
    f_{\mathcal{N}_1}(\mathcal{W}^*,x)=f_{\mathcal{N}_2}(\widehat{\mathcal{W}},x),
    $$
    for every $x\in\R^d$, provided 
    $$
    N \geqslant \widehat{N}_3^*\triangleq  \sum_{t=0}^{k^L}\binom{d+t-1}{t}.
    $$
\end{itemize}
\end{theorem}
The thesis of Theorem \ref{thm:overly-general-poly} holds provided $\mathcal{N}_2$ interpolates, regardless of $\widehat{m}$ and how the training is done. The sample complexity $\widehat{N}_3^*$ required by part ${\rm (b)}$ turns out to be at most $d^{2k^L}$. 

The proof of Theorem \ref{thm:overly-general-poly} is given in Section \ref{sec:pf-thm:overly-general-poly}. 
\subsubsection*{Networks with Arbitrary Admissible Activations}
We now present our next generalization result addressing the networks with non-polynomial activations, where the weights of the student network $\mathcal{N}_2$ enjoy a boundedness assumption like those of the teacher network $\mathcal{N}_1$.

\begin{theorem}\label{thm:demyst-overparam-regular}
Suppose that the assumptions of Theorem \ref{thm:the-main} hold. 
Let $\mathcal{N}_2$ be a student network of width $\widehat{m}\geqslant m^*$, whose set  $\widehat{\mathcal{W}}$ of weights satisfies ${\rm (i)}$ $\widehat{a}\in\R^{\widehat{m}}$, $\widehat{W}_1\in \R^{\widehat{m}\times d}$, and $\widehat{W}_k\in \R^{\widehat{m}\times \widehat{m}}$ for $2\leqslant k\leqslant L$, ${\rm (ii)}$ $\|\widehat{a}\|_{\ell_1}\leqslant 1$, $\|\widehat{W}_{k,j}\|_{\ell_1}\leqslant 1$, for $1\leqslant k\leqslant L$, and all $j$; where $\widehat{W}_{k,j}$ is the $j^{th}$ row of matrix $\widehat{W}_k$, and ${\rm (iii)}$ $f_{\mathcal{N}_1}(\mathcal{W}^*,X^{(i)})=f_{\mathcal{N}_2}(\widehat{\mathcal{W}},X^{(i)})$, for every $i\in[N]$. 

Then, with probability at least $1-\exp(-cN^{1/4})$
with respect to the training data $(X^{(i)},Y_i)\in\R^d\times \R$, $1\leqslant i\leqslant N$ $$\mathcal{L}\left(f_{\mathcal{N}_1}(\mathcal{W}^*,{}\cdot{}),f_{\mathcal{N}_2}(\widehat{\mathcal{W}},{}\cdot{})\right)\leqslant \epsilon,$$  
provided
\begin{equation}\label{eq:sample-cpmplex-4}
 N\geqslant N_4^*\triangleq  2^{24}\epsilon^{-6}\exp\left(96\log(d)\varphi_\sigma^L(\sqrt{\epsilon/16L^2})\right) C^{18}\left(P,\varphi_\sigma^L(\sqrt{\epsilon/16L^2})\right).
\end{equation}
\end{theorem}
Namely, an estimator $f_{\mathcal{N}_2}(\widehat{\mathcal{W}},\cdot)$, which takes the form of a potentially larger network generalizes well; provided it interpolates a sufficient number of random data. Theorem \ref{thm:demyst-overparam-regular} is based on the fact that the guarantees of Theorem \ref{thm:the-main} are completely independent of the widths $m^*$ and $\widehat{m}$ of teacher and student networks. The (OLS) estimator $\widehat{h}(\cdot)$ studied under Theorem \ref{thm:the-main} is used herein as an auxiliary device to relate two networks and establish a generalization bound. It is important to note that Theorem \ref{thm:demyst-overparam-regular} is completely oblivious to how the training is done: it suffices that $\widehat{\mathcal{W}}$ interpolates the data,  and the norm of the vectors in $\widehat{\mathcal{W}}$ are bounded by one. 

The proof of Theorem \ref{thm:demyst-overparam-regular} is deferred to Section \ref{sec:pf-main-generalization}.

We now consider a setup where the weights of the student network $\mathcal{N}_2$ are not explicitly bounded. That is, the weights need not enjoy a bounded norm condition. In particular, we consider the following setting: the labels $Y_i$ are generated by a shallow teacher network $\mathcal{N}_1$ of width $m^*$, that is, $Y_i=f_{\mathcal{N}_1}(a^*,W^*,X^{(i)})=\sum_{j=1}^{m^*}a_j^*\sigma(\ip{W_j^*}{X^{(i)}}$, where the activation function $\sigma(\cdot)$ is positive homogeneous, that is, there exists a $\kappa>0$ such that for any $x\in\R$ and $c\in\R^+$, $\sigma(cx)=c^\kappa \sigma(x)$. Here, $W^*\in\R^{m^*\times d}$ is the planted weight matrix with rows $W_j^*\in\R^d$, $1\leqslant j\leqslant m^*$ (carrying the weights of $j^{th}$ neuron), and $a^*\in\R^{m^*}$ is the vector of output weights. 

Now, fix a positive integer $\widehat{m}\geqslant m^*$, and let $(\widehat{a},\widehat{W})\in\R^{\widehat{m}}\times \R^{\widehat{m}\times d}$ be the weights of a (potentially) wider, otherwise arbitrary student network $\mathcal{N}_2$ of width $\widehat{m}$ interpolating the data, that is, 
$$
Y_i=f_{\mathcal{N}_2}(\widehat{a},\widehat{W},X^{(i)})=\sum_{j=1}^{\widehat{m}}\widehat{a_j}\sigma(\langle \widehat{W_j},X^{(i)}\rangle),
$$
for all $i$. Let $\theta>0$ be arbitrary. Since $\sigma(\cdot)$ is assumed to be homogeneous, it holds that for any $X\in\R^d$, 
$$
\widehat{a_j}\sigma(\langle\widehat{W_j},X\rangle) = \widehat{a_j}\theta^{-\kappa}\sigma(\langle\theta\widehat{W_j},X\rangle).
$$
That is, one can effectively ``push" the parameter scales to the output layer, and assume without loss of generality that $\|\widehat{W}_j\|_{\ell_2}=d^{-1/2}$ (and thus, $\|\widehat{W_j}\|_{\ell_1}\leqslant 1$) for every $j\in[\widehat{m}]$. The bound on the $\ell_2-$norm is a technical requirement needed for the proof details for the results that follow. 

The question now arising is whether interpolation implies that the norm $\|\widehat{a}\|_{\ell_1}$ of $\widehat{a}$ is well-controlled. In full generality, this is not necessarily the case, which we now demonstrate. Let $(a^*,W^*)\in\R^{m^*}\times \R^{m^*\times d}$ be the planted weights. We now construct a wider network interpolating the data, whose vector of output weights has arbitrarily large norm, by introducing many cancellations. Fix a $z\in\mathbb{N}$,  a non-zero vector $v\in\R^d$; and a constant $\nu>0$. Construct a new network $(\widehat{a},\widehat{W})$ with $m^*+2z$ nodes, as follows. For $1\leqslant j\leqslant m^*$, set $\widehat{a}_j=a_j^*$, and $\widehat{W}_j=W_j^*$. For any $m^*+1\leqslant j\leqslant m^*+2z$, set $\widehat{a}_{j} = \nu$ if $j$ is even, and $-\nu$, if $j$ is odd. At the same time, set $\widehat{W}_j =v$, for $m^*+1\leqslant j\leqslant m^*+2z$. This network still interpolates the data, while $\|\widehat{a}\|_{\ell_1} = \|a^*\|_{\ell_1}+2z\nu$. Since $\nu>0$, and $z$ is arbitrary, $\|\widetilde{a}\|_{\ell_1}$ can be made arbitrarily large by making $z$ (or $\nu$) arbitrarily large. For this reason, we restrict ourselves to the case of non-negative output weights, that is, $\widehat{a}_j,a_j^*\geqslant 0$. To handle yet another technicality, we also adopt an additional assumption on the data distribution of the coordinates of $X\in\R^d$. In particular, we assume the existence of the density, bounded away from zero.

In this case, it turns out that a certain self-regularization indeed takes place for any ``interpolator" of sufficiently many data, regardless of how the training is done. 

\begin{theorem}\label{thm:interpolant-norm-poly-bd}
Suppose that $X^{(i)}$, $1\leqslant i\leqslant N$ are i.i.d.\ samples of an admissible distribution per Definition \ref{def:admissible-distribution}, where the coordinates of the data have a density bounded away from zero. Let $Y_i$ be the corresponding label generated by a shallow teacher network $\mathcal{N}_1$ of width $m^*$, that is, $Y_i = f_{\mathcal{N}_1}(a^*,W^*,X^{(i)})=\sum_{j=1}^{m^*}a_j^* \sigma(\langle W_j^*,X^{(i)}\rangle)$, where ${\rm (i)}$ for $1\leqslant j\leqslant m^*$, $\|W_j^*\|_{\ell_1}\leqslant 1$ and $a_j^*\geqslant 0$; and ${\rm (ii)}$ $\sigma(\cdot)$ is non-negative, non-decreasing on $[0,\infty)$, and is positive homogeneous with constant $\kappa$. Let $\widehat{m}\geqslant m^*$ be an integer, and $(\widehat{a},\widehat{W})\in\R_{\geqslant 0}^{\widehat{m}}\times \R^{\widehat{m}\times d}$ be the weights of a student network $\mathcal{N}_2$ interpolating the data, that is, $Y_i=f_{\mathcal{N}_2}(\widehat{a},\widehat{W},X^{(i)})$ for every $i\in[N]$. Suppose $\|\widehat{W}_j\|_{\ell_2}=1/\sqrt{d}$.  Then, with probability at least $1-\exp(-cN^{1/3})$ over $X^{(i)}$, $1\leqslant i\leqslant N$, we have
$$
\|\widehat{a}\|_{\ell_1}\leqslant d^{\kappa+1}2^{\kappa+1} \|a^*\|_{\ell_1},
$$
provided \begin{equation}\label{eq:sample-complex-N5-star}
    N>N_5^*\triangleq \exp(3d\log(d)).
\end{equation}
\end{theorem}
\begin{remark}
Note that since the activation $\sigma(\cdot)$ is homogeneous, the assumption $\|\widehat{W}_j\|_{\ell_2}=1/\sqrt{d}$ is without any loss of generality.
\end{remark}
For the ReLU activation, we thus obtain $\|\widehat{a}\|_{\ell_1}\leqslant 4d^{2} \|a^*\|_{\ell_1}$ for any such interpolator, with high probability. 

The proof of Theorem \ref{thm:interpolant-norm-poly-bd} is provided in Section \ref{sec:pf-thm-interpolant-norm}. 

Having controlled $\|\widehat{a}\|_{\ell_1}$ (and also the norm of $\widehat{W}_j$ due to homogeneity) for any interpolator, we are now in a position to state our final main result, addressing the generalization under the aforementioned setting:
\begin{theorem}\label{thm:overparam:ReLU-case}
Suppose that the assumptions of Theorem \ref{thm:interpolant-norm-poly-bd} hold, and let $(\widehat{a},\widehat{W})\in\R_{\geqslant 0}^{\widehat{m}}\times \R^{\widehat{m}\times d}$ be the weights of  a student network $\mathcal{N}_2$ of width $\widehat{m}$ interpolating the data. Then, with probability at least $1-2\exp(-cN^{1/4})$ over data $X^{(i)}$, $1\leqslant i\leqslant N$, it holds: $$\mathcal{L}\left(f_{\mathcal{N}_1}(a^*,W^*,{}\cdot{}),f_{\mathcal{N}_2}(\widehat{a},\widehat{W},{}\cdot{})\right)\leqslant \epsilon,$$  provided 
\begin{equation}\label{eq:sample-complex-N6}
N>N_6^*\triangleq 2^{12\kappa+24}\epsilon^{-6}d^{12\kappa+12}\exp\left(96\log(d)\varphi_\sigma\left(\sqrt{\epsilon 2^{-2\kappa-7}d^{-2\kappa-2}}\right)\right) C^{18}\left(P,\varphi_\sigma\left(\sqrt{\epsilon 2^{-2\kappa-7}d^{-2\kappa-2}}\right)\right).
\end{equation}

\end{theorem}
In spirit, Theorem \ref{thm:overparam:ReLU-case} is similar to Theorem \ref{thm:demyst-overparam-regular}: an estimator of the form of a ``student" network of width $\widehat{m}$ generalizes to a ``teacher" network of width $m^*$ (where $\widehat{m}\geqslant m^*$) to within an error of at most $\epsilon$, provided it interpolates a sufficient number of data, irrespective of how the interpolation is done. The weights, however, are not explicitly bounded this time: Theorem \ref{thm:interpolant-norm-poly-bd} yields that any interpolating weights are necessarily self-regularized. 

For the case of ReLU networks (for which $\varphi_\sigma(\epsilon)=O(1/\epsilon)$ and $\kappa=1$) Theorem \ref{thm:overparam:ReLU-case} yields a sample complexity bound $N_6^*$ of the order $\exp(\Theta(d^2\log d/\sqrt{\epsilon}))$. 
The proof of Theorem \ref{thm:overparam:ReLU-case} is provided in Section \ref{sec:proof-of-overparam:ReLU-case}.

\section{Testing Polynomial Regression Method on Real Datasets} \label{appendix:realdata}
In this section, we present our main computational findings. 
Our polynomial regression method is implemented on the MNIST \cite{lecun1998mnist} and Fashion-MNIST \cite{xiao2017fashion} datasets, both of which are predominantly approached using neural networks and serve as widely used testbeds for deep learning algorithms. Both datasets contains $N=60,000$ training images, each consisting of $28\times 28$ gray-scale pixels, that is, $d=784$. We study both datasets as a classification task with 10 classes. 
\paragraph{Main Prediction Model} Our prediction model is constructed by stacking 10 different polynomial regression models, one per class, according to (\ref{eq:poly-regression}). In particular, the choice of coefficients of $q$ for each class is made independently. The regression model
 then outputs, for any test data, a $\widehat{Y}\in \R^{10}$ where $\widehat{Y}_i$ is the predicted value for the class corresponding to index $i$.  Finally, we predict the class based on the rule $\argmax_{1\leqslant i\leqslant 10} \widehat{Y}_i$. Our focus is on the case of polynomial regression with degree $M=2$.

\paragraph{Two modifications}
Even the case of $M=2$ is already computationally prohibitive to the resources of a personal computer: computing the standard OLS estimator requires a time that is quadratic in the number of features; and the number $d+\binom{d}{2}$ of features becomes $307,720$ if we consider all terms with degree at most $2$. To circumvent this complication and successfully implement our algorithm, we apply the following two modifications.
\begin{itemize}
    \item[(a)] We resort first to a ``convolutional" version of the regression where the only products of  variables $X_iX_j$ that are used correspond to nearby pixels. More specifically, we only take pairs of variables that can be covered simultaneously by a filter of size $3\times 3$. This step drops the number of features to $18,740$.
    
    \item[(b)] We resort to a batching and averaging heuristic. Fix an $n$ and $B$. At each step $i\in[n]$, we sample $B$ samples from the collection of $N$ samples, uniformly at random; and run the ``convolutional" version of our prediction model for these samples. This yields $n$ different sets of coefficients. These coefficients are then averaged over $n$ to form our final predictor. 
\end{itemize}

\protect\begin{figure}
\protect\begin{subfigure}[h]{0.5\textwidth}
\protect\includegraphics[width=\textwidth]{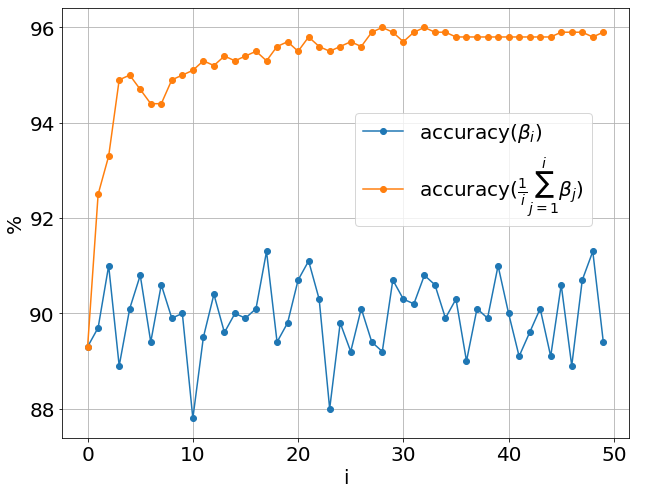}
\protect\caption{MNIST digits}
\protect\end{subfigure}
\hfill
\protect\begin{subfigure}[h]{0.5\textwidth}
\protect\includegraphics[width=\textwidth]{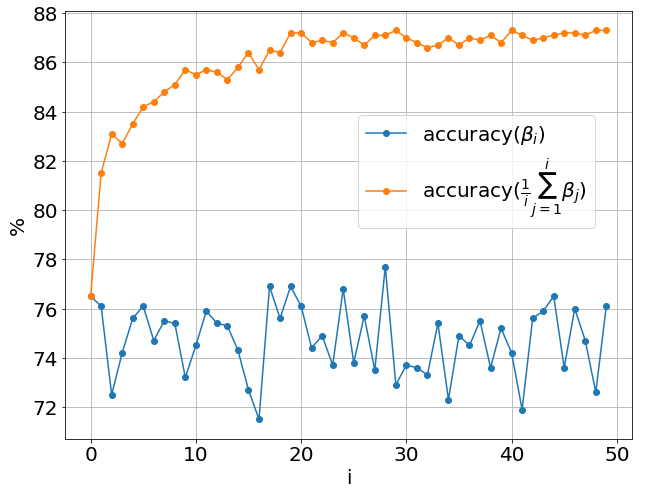}
\protect\caption{Fashion-MNIST}
\protect\end{subfigure}%
\protect\caption{Out-of-sample accuracy of the multi-dimensional polynomial regression method, with $B = 1000$ and $n = 50$.}
\protect\label{fig:result}
\protect\end{figure}

\paragraph{Results} The results of the aforementioned heuristic batch training process for MNIST and Fashion-MNIST datasets are reported in Figure \ref{fig:result}. We have chosen $n=50$ and  $B=1000$. The rationale behind this choice is to ensure that the number $nB$ of samples we have used throughout is comparable with the size $N$ of the whole dataset. The accuracy we plot is evaluated out-of-sample. The bottom curve corresponds to the prediction results on each individual batch $i$ (one 10-dimensional regression) and the top curve reports the results when aggregated over batches as $i$ increases from $1$ to $n$ ($i$ $10$-dimensional regression coefficients averaged together). This procedure takes 26 minutes, and the out-of-sample accuracies are reported in Table \ref{tab:results}.

\paragraph{Comparison} The results we obtained are certainly below the state-of-the-art reported in papers \cite{stateofartmnist,stateofartfashion}. 

However, in order to ensure a fair comparison on the same platform, we have implemented a Convolutional Neural Network (CNN) model with 2 layers, each composed of a convolution, a ReLU activation; and a batch normalization followed by a fully connected layer of size $1000$. This is a standard architecture for document recognition, similar to LeNet-5 \cite{cnnarchitecture}.
This architecture was trained for a similar amount of time (27.5 minutes), using Stochastic Gradient Descent for $10$ epochs with a batch of size $100$. Our results are also reported in Table \ref{tab:results}. 

Despite being below the state-of-the-art, our method is: 1) comparable with sophisticated deep network architectures, 2) conceptually simple, and 3) is fully interpretable. We also expect the performance to improve further by increasing $M$.

\begin{table}[h!]
  \begin{center}
    \begin{tabular}{c|c|c|c}
      \toprule
      \textbf{Dataset} & \textbf{Polynomial regression} & \textbf{Convolutional Neural Net} & \textbf{State-of-the-art}\\
      \midrule
      \textbf{MNIST} & 95.94\% & 99.03\% & 99.79\%\\
      \textbf{Fashion-MNIST} & 87.11\% & 90.80\% & 96.35\%\\
      \bottomrule
    \end{tabular}
    \caption{Out-of-sample accuracy comparison between methods.}
    \label{tab:results}
  \end{center}
\end{table}
\subsection*{Noise robustness}
 We now explore the noise robustness of the polynomial regression algorithm. For this purpose, we modify the image to be classified in two different ways: 1) change {\bf every} pixel by some amount by adding random noise; and 2) change the pixels {\bf only in a smaller region}, where each pixel is allowed to change by a ``large" amount. For convenience, we refer to the former modification as the {\em global noise} and the latter as the {\em local noise}. These modifications, in fact, are in line with existing literature on adversarial examples, see \cite{shamir2019simple} and references therein. 

In a nutshell, our results show essentially the following: while the CNN architecture still achieves a higher prediction accuracy when the noise level is small, polynomial regression method starts outperforming the CNN architectures once the noise level is increased beyond a certain level.
We now report the details of our computational findings in a greater detail.
\begin{itemize}
    \item[(a)]  We first consider the {\em global noise} setting, where the details are as follows. We perturb the image to be classified by adding, to each pixel, a random noise drawn from Gaussian distribution with mean zero and standard deviation $\sigma$, namely from $\mathcal{N}(0,\sigma^2)$. The noise added is i.i.d.\ across each pixel. We report our findings by providing a plot of the accuracy against the standard deviation $\sigma$ of the noise.
    
    Now, let $N_{i,j}$, $1\leqslant i,j\leqslant 28$, be i.i.d.\ samples of $\mathcal{N}(0,\sigma^2)$. For every pair $(i,j)$, we first add the noise to $X_{i,j}$:
    $$
      \widetilde{X}_{i,j} = X_{i,j} + N_{i,j}.
    $$
    We then truncate $\widetilde{X}_{i,j}$ so as to ensure the resulting image is still gray-scale, that is, the resulting pixel values are between $0$ and $1$. For this we consider $\max\{0,\min\{1,\widetilde{X}_{i,j}\}\}$.
    \item[(b)] We next consider the {\em local noise} setting. To the image being classified, we add a rectangular patch, consisting of black pixels; whose center is chosen randomly. More specifically, for a given patch area $A$, we first choose a random pixel location $(i,j)$ uniformly around the center of the image ($6 \leqslant i,j \leqslant 22$),  and  an aspect ratio $D$, which is distributed uniformly on $(1/2,2)$. We then create a patch centered at $(i,j)$, of width $w = \lfloor D \sqrt{A} \rfloor$ and height $h = \lfloor \sqrt{A} / D \rfloor$,  where $\lfloor r\rfloor$ denotes the largest integer not exceeding $r\in\mathbb{R}$. For each pixel covered by the patch, we set its value to $0$. That is, we make that pixel black. 
    
    The rationale for restricting the center of patch to $6\leqslant i,j\leqslant 22$ is to ensure that on average a patch does not cover the dark background region of the images, but covers the region with more informative content, e.g. the pixels making up the digit to be recognized. 
\end{itemize}
 
\protect\begin{figure}[!h]
\protect\begin{subfigure}[h]{0.5\textwidth}
\protect\includegraphics[width=\textwidth]{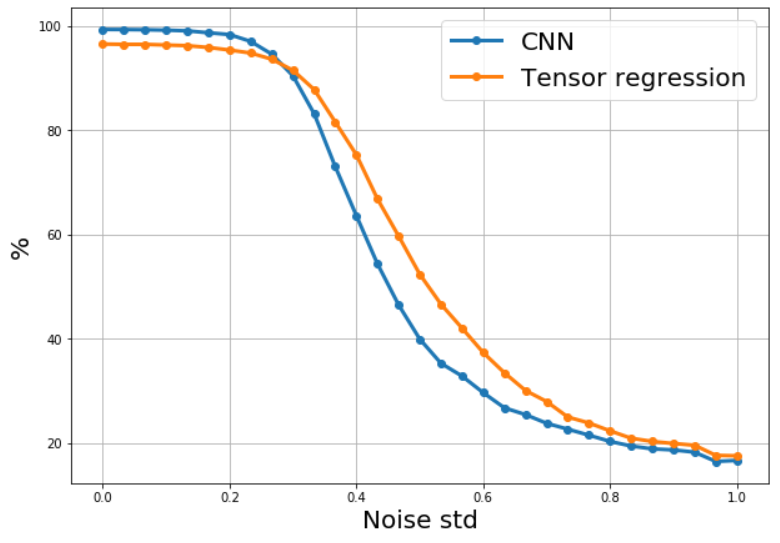}
\protect\caption{MNIST}
\protect\end{subfigure}
\hfill
\protect\begin{subfigure}[h]{0.5\textwidth}
\protect\includegraphics[width=\textwidth]{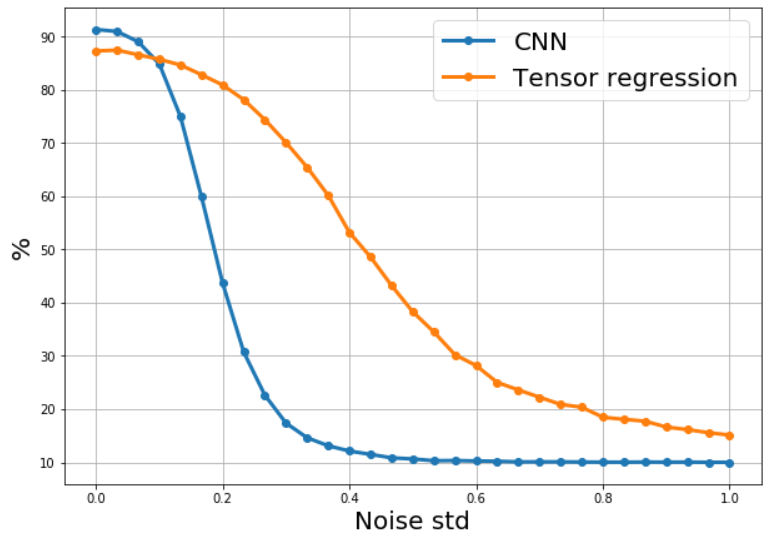}
\protect\caption{Fashion-MNIST}
\protect\end{subfigure}%
\protect\caption{Robustness of models against normal noise}
\protect\label{fig:rob1}
\protect\end{figure}


\protect\begin{figure}[!h]
\protect\begin{subfigure}[h]{0.5\textwidth}
\protect\includegraphics[width=\textwidth]{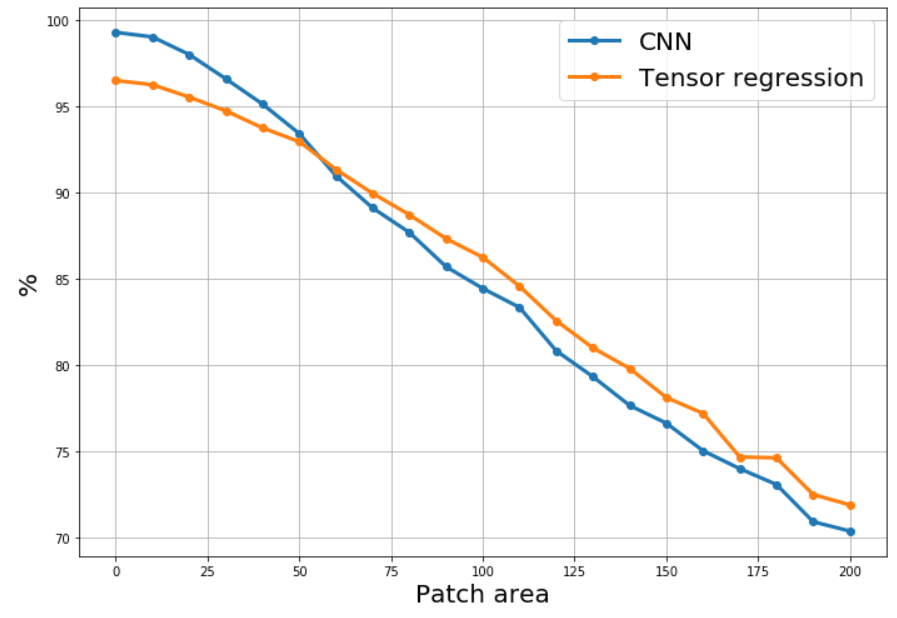}
\protect\caption{MNIST}
\protect\end{subfigure}
\hfill
\protect\begin{subfigure}[h]{0.5\textwidth}
\protect\includegraphics[width=\textwidth]{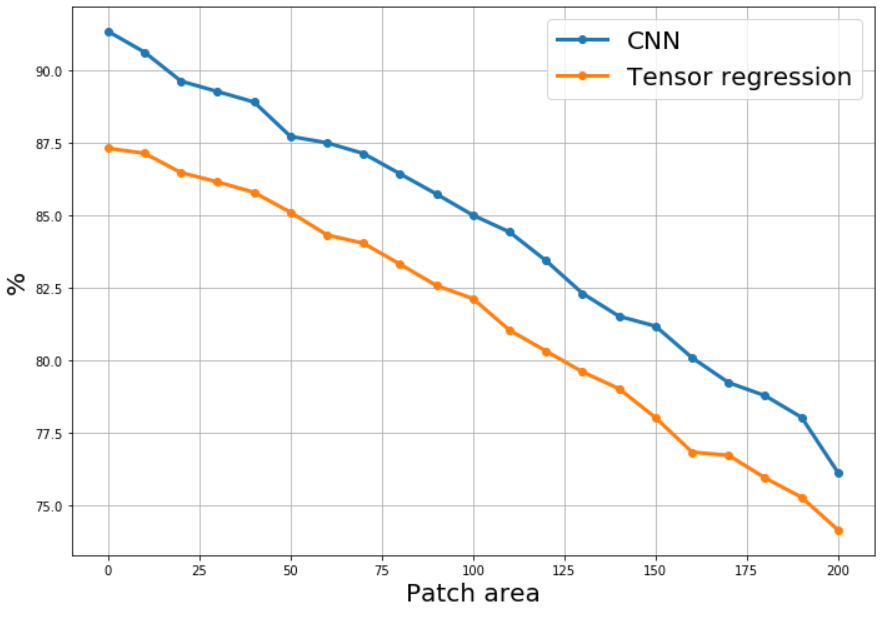}
\protect\caption{Fashion-MNIST}
\protect\end{subfigure}%
\protect\caption{Robustness of models against black patching}
\protect\label{fig:rob2}
\protect\end{figure}

Our results are reported in Figures \ref{fig:rob1} and \ref{fig:rob2}. The details are as follows.
  \begin{itemize}
     \item[(a)] Our computational findings pertaining the MNIST dataset under the {\em global noise} setting are reported in Figure \ref{fig:rob1}${\rm (a)}$. It is evident from the plot that while the CNN architectures outperform the polynomial regression (in terms of prediction accuracy) for a noise standard deviation up to $0.3$ , the polynomial regression starts outperforming the CNN architectures, if the noise standard deviation is pushed beyond 0.3. For Fashion-MNIST, reported in Figure \ref{fig:rob1}${\rm (b)}$, the difference between the robustness of two models is even more dramatic. 
     
     We believe that the aforementioned findings are a consequence of the simplicity of the polynomial regression algorithm, in contrast to CNN architectures. 
     \item[(b)] We now investigate the case of {\em local noise} where a black patch with a random  center and aspect ratio as noted above is inserted within the image. We start with the MNIST dataset, reported in Figure \ref{fig:rob2}${\rm (a)}$. While the CNN architectures outperform the polynomial regression algorithm for the patch area up to 55 pixels (which corresponds roughly to the $7\%$ of the image), the polynomial regression starts outperforming if the patch area is increased beyond 55 pixels. For the Fashion-MNIST dataset, however, the situation is different, as reported in Figure  \ref{fig:rob2}${\rm (b)}$: while the accuracy drop for the both models appear to follow the same trend as the patch area is increased, polynomial regression  clearly performs uniformly worse than CNN architectures. That is, for all values of the patch area we have investigated, the prediction accuracy of polynomial regression is  worse than that of CNN architectures. 
 \end{itemize} 



\section{General Bounds on the Condition Number}\label{sec:bound-cond-no}
As we have mentioned earlier, the correctness of T-OLS algorithm (Algorithm \ref{algo:tensor-non-poly}) heavily relies on properly bounding the condition number of a certain expected covariance matrix of the tensorized data. In  this section, we present our bounds to this end. We believe that our results are of potential independent interest in the study of random tensors \cite{vershynin2019concentration}.


\paragraph{Multiplicities set} We first define {\em multiplicities set}. Recall Equation (\ref{fst}). In this section we focus on input $X$, drawn from some product distribution $P^{\otimes d}$. That is,  $P$ is the distribution of the coordinates of $X$. For this reason, we use throughout this section the notation per \eqref{eq:omega-Pt-P}
$$\omega(P,T):=\min\{|\mathrm{Support}(P)|-1,T\}\quad\text{where}\quad T\in\mathbb{N}.$$

Let $M\in\mathbb{N}$ and $P$ an input distribution. The set $\mathcal{C}$ is called as the $(d,M)-$multiplicities set, if
\begin{equation}\label{def:multiplicities}
\mathcal{C} = \bigcup_{0\leqslant t\leqslant M'} \mathcal{F}_{\omega(P,M),t},
\end{equation}
where $\mathcal{F}_{\omega(P,M),t} :=  \{(0,0,\dots,0)\}$ for $t=0$ and $M':=\min\{M,d\left(|\mathrm{Support}(P)|-1\right)\}.$ 

The definition and use of $M'$ is based on Remark \ref{rem:fst-empty} and ensures based on an elementary calculation that all components of the union in the definition of $\mathcal{C}$ are non-empty, that is for all $0\leqslant t\leqslant M'$ it holds
\begin{equation}\label{eq:fst-nonempty}|\mathcal{F}_{\omega(P,M),t}| \geq 1.\end{equation} The proof of this fact follows by combining Remark \ref{rem:fst-empty} with the observation that $M'\leqslant d \omega \left(P,M\right).$

\paragraph{Further Notation} In this section, we will utilize multiplicities set introduced in \eqref{def:multiplicities} with $k$ being entered in place of $M$ and by $P$ the distribution on the (iid) input's coordinates. By $\mathcal{C}$ we refer to the $(d,k)$-multiplicities set, \eqref{def:multiplicities}. 
For $n=0,1,2,\ldots,k,$ we denote $c_n:=\E{X^n}$ (where $X$ is drawn from $P$) the $n$-th moment of $P$ and \begin{align}\label{detn}D_n:=\det\left( (c_{i+j})_{i,j=0,1,2,\ldots,n}\right).\end{align}   

\subsection{Main Result}
\begin{theorem}\label{cond:main_2}
Let $d,k \in \mathbb{N}$. Suppose input $X \in [-1,1]^d$ follows an admissible distribution $\mathcal{D}=P^{\otimes d}$, and $\mathcal{X} = ((X_{\alpha})_{\alpha \in \mathcal{C}})^T\in\R^{|\mathcal{C}|}$ the associated column vector of monomials indexed by $\mathcal{C}$. Then for $$\Sigma = \E{\mathcal{X} \mathcal{X}^T}$$ there exists $f\left(P,k\right),c\left(P,k\right)>0$ independent of $d$ for which it holds
\begin{align}\label{minEgein2}
\lambda_{\min}(\Sigma) \geqslant c\left(P,k\right) \left(\sum_{i=0}^k \binom{d+i-1}{i}\right)^{-1}.
\end{align}
 and \begin{align}\label{maxEgein2}
\lambda_{\max}(\Sigma) \leqslant f\left(P,k\right) \sum_{i=0}^k \binom{d+i-1}{i}.
\end{align}

\end{theorem}
The following corollary can be immediately deduced,
\begin{Corollary}\label{coro:number}
Let $d,k \in \mathbb{N}$ with $d \geqslant 4$. Suppose $X_i, i=1,2,\ldots,d$ are iid random variables coming from a distribution $P$, $\mathcal{X} = ((X_{\alpha})_{\alpha \in \mathcal{C}})^T\in\R^{|\mathcal{C}|}$ the associated column vector of monomials. Then for $$\Sigma = \E{\mathcal{X} \mathcal{X}^t}$$
for some constant $C\left(P,k\right)>0$ independent of $d$ for which,
\begin{align}\label{maxCond}
\kappa(\Sigma) = \frac{\lambda_{\max}(\Sigma)}{\lambda_{\min}(\Sigma)} \leqslant C\left(P,k\right) d^{3k}.
\end{align}
where we can take $C\left(P,k\right)= \frac{f\left(P,k\right)}{c\left(P,k\right)}$ for $f\left(P,k\right),c\left(P,k\right)$ defined in Theorem \ref{cond:main_2}.

\end{Corollary}
\begin{proof}{(of Corollary \ref{coro:number})}
It suffices to establish
$$
\left(\sum_{i=0}^k \binom{d+i-1}{i}\right)^2\leqslant d^{3k}.
$$
Observe that for any $t \in \mathbb{N},$ $\binom{d+t-1}{t}$ is the number of $t-$tuples $(i_1,\dots,i_t)$ with $1\leqslant i_1\leqslant i_2\leqslant \cdots\leqslant i_t\leqslant d$, and thus is trivially upper bounded by $d^t$. With  this, we have
$$
\left(\sum_{i=0}^k \binom{d+i-1}{i}\right)^2 \leqslant \left(\sum_{i=0}^k d^i\right)^2 \leqslant \left(k d^k\right)^2 \leqslant d^{3k},
$$ where for the last inequality we used that $d \geqslant 4$ and $2^k \geqslant k$ for all $k \in \mathbb{N}.$ The desired conclusion follows.
\end{proof}

Several remarks are now in order.
\begin{remark}
In Theorem \ref{cond:main_2} one can take 
\begin{align}\label{cpk} c\left(P,k\right)=\min\{\min_{i=0,1,\ldots,\omega(P,k)}\left(\frac{D_i}{D_{i-1}}\right)^{\frac{k}{2}},1\} \end{align} and \begin{align}\label{dpk}f\left(P,k\right)=\max\{c^{k}_{2\omega(P,k)},1\} \frac{\max\{\max_{i=0,1,\ldots,\omega(P,k)}\left(\frac{D_i}{D_{i-1}}\right)^{k},1\}}{\min\{\min_{i=0,1,\ldots,\omega(P,k)}\left(\frac{D_i}{D_{i-1}}\right)^{k},1\} },\end{align}where $D_n$ is defined in (\ref{detn}). Later in Lemma \ref{ortho1} we establish that for all non-constant $P$, for all $n \geqslant 0,$ $D_n>0$ which justifies the use of $D_n$ in the denominator. 
\end{remark}

\begin{remark}
If $P$ is a discrete distribution, $\omega(P,k)\leqslant |\mathrm{Support}(P)|-1=O(1).$ Hence, from \eqref{cpk} and \eqref{dpk}, since the maxima and minima are defined over bounded many terms, we conclude that $$c\left(P,k\right)=\exp\left(-O\left(k\right) \right)\quad\text{and}\quad f(P,k)=\exp\left(O\left(k\right) \right).$$ In particular, $$
C\left(P,k\right)=\exp\left(O\left(k\right) \right).
$$
\end{remark}

\begin{remark}
Albeit the focus of the present paper is on admissible input distributions $P^{\otimes d}$ where $P$ is supported on $[-1,1]$, Theorem \ref{cond:main_2} remains true under the weaker assumption that $P$ has finite moments of all orders. The proof of  Theorem \ref{cond:main_2} provided in the following subsections transfer identically under this weaker assumption, which could be of independent interest.
\end{remark}
\begin{remark}\label{remark:on-order-between-c-C-f}
An inspection of (\ref{cpk}) and (\ref{dpk}) reveal immediately that, $c(P,k)\leqslant 1$; and $f(P,k)\geqslant 1$. This, in particular, yield the following for $C(P,k)=\frac{f(P,k)}{c(P,k)}$, the constant defined in Corollary \ref{coro:number}: $C(P,k)\geqslant 1$ and $C(P,k)\geqslant \max\{f(P,k),c(P,k)^{-1}\}$.
\end{remark}

 The remaining
 subsections are devoted to the proof of Theorem \ref{cond:main_2}. The proof utilizes the theory of orthogonal polynomials. We start with a series of auxiliary results.

\subsection{Auxilary Results: Orthogonal Polynomials}
First we define $\mathcal{P}_{\omega(P,k)}$ the vector space of univariate polynomials of degree at most $\omega(P,k)$, which is naturally spanned by the basis
$\mathcal{B}_1=\{1,x,x^2,\ldots,x^{\omega(P,k)}\}.$

Note that since $\omega(P,k) \leqslant |\mathrm{Support}\left(P\right)|-1$ the following property holds: if $p \in \mathcal{P}_{\omega(P,k)}$ satisfies $\E{p(X)^2}=0$, where $X$ is drawn from $P$, then $p$ is the zero polynomial. This follows since any such $p$ must satisfy $p(x)=0$ for $|\mathrm{Support}\left(P\right)|$-many points and $|\mathrm{Support}\left(P\right)| \geqslant k+1 >\mathrm{degree}(p).$ For this reason the operation defined for $p,q \in \mathcal{P}_{\omega(P,k)}$ $\inner{p,q}_P:=\E{p(X)q(X)}$, where $X$ is drawn from $P$, is a well-defined non-degenerate inner product defined on the space $\mathcal{P}_{\omega(P,k)}.$

Using this definition of an inner product we can orthogonalize using the Gram-Schmidt operation the basis of monomials $\mathcal{B}_1$ and get the orthonormal basis $\mathcal{B}_2=\{T_0,T_1,\ldots,T_{\omega(P,k)}\}$. We know that for $i,j=0,1,\ldots,\omega(P,k)$ it holds
\begin{align}\label{orthonorm1}
\E{T_i(X)T_j(X)}=1(i=j).
\end{align}

We define for any $i=0,1,2,\ldots,\omega(P,k)$ and any $x \in \mathbb{R}$ 
$$
P_i(x):=\left[ {\begin{array}{ccccc}
   c_0 & c_1 & c_2 & \cdots & c_{i} \\
   c_1 & c_2 & c_3 & \cdots & c_{i+1} \\
   \cdots & \cdots & \cdots& \cdots & \cdots \\
   c_{i-1} & c_{i} & c_{i+1} & \cdots & c_{2i-1} \\
   1 & x & x^2 & \cdots &x^{i} \\
  \end{array} } \right] \in \mathbb{R}^{(i+1)\times (i+1)}.
$$
\begin{lemma}\label{ortho1}The following properties hold for $\mathcal{B}_2$ and for a non-constant distribution $P$ with finite $2\omega(P,k)$-moment.
\begin{itemize}
  \item[(1)] $D_0=1$ and $T_0(x)=1$ for every $x \in \mathbb{R}$. Furthermore, for $i=1,\ldots,\omega(P,k)$,
  $$D_i=\frac{1}{(i+1)!}\E{\prod_{m,n=0,1,\ldots,i} \left(X_m-X_n\right)^2}>0,$$where $X_m,m=0,1,2,\ldots,i$ are iid draws of $P$ and $$T_i(x)=\frac{1}{(D_iD_{i-1})^{\frac{1}{2}}}\det\left( 
  P_i(x) \right),$$for every $x \in \mathbb{R}.$
    \item[(2)]  For every $i,j=0,1,\ldots,\omega(P,k)$ with $i<j$,   $$\E{X^iT_j(X)}=0.$$
    \item[(3)]  For every $i=1,\ldots,\omega(P,k)$, $$\E{X^iT_i(X)}=\sqrt{\frac{D_i}{D_{i-1}}} .$$
\end{itemize}
\end{lemma}

\begin{proof}
$D_0=1$ and $T_0(x)=1$ for every $x \in \mathbb{R}$ follows straightforwardly from the definition of the first element of the basis $\mathcal{B}_1$, which is the constant monic polynomial, and Gram-Schmidt operation.  

Since the monomials $\mathcal{B}_1$ correspond to a basis in $\mathcal{P}_{\omega(P,k)}$, the two displayed identities in the first part of the Lemma are standard results on orthogonal polynomials on the real line, see \cite[Section~2]{szeg1939orthogonal}. More specifically, the second identity is described identically in \cite[Equation~2.1.5]{szeg1939orthogonal}. The first identity  can be derived in a straightforward way from the identity \cite[Equation~2.1.6]{szeg1939orthogonal}, using the standard determinant formula of a Vandermonde matrix.  In both cases, in \cite{szeg1939orthogonal}'s notation we choose the choice of basis $f_i(x)=x^i,i=0,1,2,\ldots,\omega(P,k)$ and the choice of measure $\alpha(x)$ in \cite{szeg1939orthogonal}'s notation  correspond to the CDF of our distribution $P$.

To establish the rest two parts of the Lemma, let us first fix $i,j=0,1,\ldots,\omega(P,k)$. We denote for $n=1,\ldots,j,j+1$ by $P_{j,n}$ the $j \times j $ matrix which is produced by the matrix $P_j(x)$ after we delete the $j+1$-th row and the $n+1$-th column. Note that we do not indicate any dependence on $x$ of $P_{j,n}$ as all rows, except the $j+1$-th row, of the matrix $P_j(x)$ are constant with respect to $x$. By expanding the determinant of $P_j(x)$ with respect to the last row we have for all $x \in \mathbb{R},$
  \begin{align} \label{detS1}
  T_j(x)=\frac{1}{(D_jD_{j-1})^{\frac{1}{2}}}\sum_{n=1}^{j+1} (-1)^{n+j+1}x^{n-1} \det\left(P_{j,n}\right).
  \end{align}
  Using (\ref{detS1}) we have
  \begin{align} \label{detS1-2}
 \E{X^iT_j(X)}& =\frac{1}{(D_jD_{j-1})^{\frac{1}{2}}}\E{X^i\sum_{n=1}^{j+1} (-1)^{n+j+1}X^{n-1} \det\left(P_{j,n}\right)}\\
 &=\frac{1}{(D_jD_{j-1})^{\frac{1}{2}}} \sum_{n=1}^{j+1} (-1)^{n+j+1}\E{X^{i+n-1}} \det\left(P_{j,n}\right)\\
  &=\frac{1}{(D_jD_{j-1})^{\frac{1}{2}}} \sum_{n=1}^{j+1} (-1)^{n+j+1}c_{i+n-1} \det\left(P_{j,n}\right).
  \end{align}We recognize the last quantity as the expansion of $$\det\left( \left[ {\begin{array}{ccccc}
   c_0 & c_1 & c_2 & \cdots & c_{j} \\
   c_1 & c_2 & c_3 & \cdots & c_{j+1} \\
   \cdots & \cdots & \cdots& \cdots & \cdots \\
   c_{j-1} & c_{j} & c_{j+1} & \cdots & c_{2j-1} \\
   c_i & c_{i+1} & c_{i+2} & \cdots & c_{i+j} \\
  \end{array} } \right]\right)$$which implies that for every $i,j=0,1,\ldots,\omega(P,k)$  $$\E{X^iT_j(X)}=\frac{1}{(D_jD_{j-1})^{\frac{1}{2}}} \det\left(C_{i,j}\right)$$where $$C_{i,j}= \left[ {\begin{array}{ccccc}
   c_0 & c_1 & c_2 & \cdots & c_{j} \\
   c_1 & c_2 & c_3 & \cdots & c_{j+1} \\
   \cdots & \cdots & \cdots& \cdots & \cdots \\
   c_{j-1} & c_{j} & c_{j+1} & \cdots & c_{2j-1} \\
   c_i & c_{i+1} & c_{i+2} & \cdots & c_{i+j} \\
  \end{array} } \right].$$
  
  To prove the second part of the Lemma, let us note that $i<j$ the $j+1$-th row is repeated again as the $i-1$-row in the matrix, implying that the determinant of $C_{i,j}$ equals to zero.
  
  To prove the third part follows, note that for all $i=1,2,\ldots,\omega(P,k)$, $C_{i,i}=(c_{m+n})_{m,n=0,1,2,\ldots,i}.$ Hence, we simply have $\det\left(C_{i,i}\right)=D_i$ from the definition of $D_i$. Hence,
  $$\E{X^iT_i(X)}=\frac{1}{(D_iD_{i-1})^{\frac{1}{2}}} D_i=\sqrt{\frac{D_i}{D_{i-1}}}.$$

\end{proof}
Next we consider the following space. Let $\mathcal{P}_{\mathcal{C},d}$ be the vector space of $d$-variate polynomials defined on $X:=(X_1,\ldots,X_d)$, spanned by the set $\mathcal{B}_{1,d}=\{X_{\alpha},\alpha \in \mathcal{C}\}.$ 

\begin{lemma}\label{fund}
Under the assumptions of Theorem \ref{cond:main_2} the following holds. If $p(x)=\sum_{\alpha \in \mathcal{C}} \lambda_{\alpha}x_{\alpha}$ defined for $x \in \mathbb{R}^d$ satisfies that $p(x)=0$ for all $x \in \mathrm{Support}(P)$, then $\lambda_{\alpha}=0$ for all $\alpha \in \mathcal{C}.$

\end{lemma}
\begin{proof}

We proceed by induction on $d$. If $d=1$, as discussed earlier the result follows as a standard application of the fundamental theorem of algebra since $\mathrm{degree}(P) \leqslant \omega(P,k)<|\mathrm{Support}\left(P\right)|$.

We assume now that the statement is true for all $1 \leqslant d \leqslant D$ for some $D \in \mathbb{N}$ and proceed to prove it for $d=D+1.$ We decompose $p$ as following, \begin{align*}
p(x)&=\sum_{u=0}^{\omega(P,k)} \sum_{\alpha \in \mathcal{C} :\alpha_{D+1}=u} \lambda_{\alpha}x_{\alpha}\\
&=\sum_{u=0}^{\omega(P,k)} \sum_{\alpha \in \mathcal{C} :\alpha_{D+1}=u} \left(\lambda_{\alpha}\left(\prod_{i=1}^D x_i^{\alpha_i}\right) x^u_{D+1}\right)\\
&=\sum_{u=0}^{\omega(P,k)} p_u \left(x_{-(D+1)}\right)  x^u_{D+1},
\end{align*}
where $x_{-(D+1)}=(x_1,\ldots,x_D)$; and
\begin{align}\label{pu}
p_u(x_{-(D+1)})=\sum_{\alpha \in \mathcal{C} :\alpha_{D+1}=u} \lambda_{\alpha}\prod_{i=1}^D x_i^{\alpha_i}.
\end{align}

For arbitrary $y \in \mathrm{Support}\left(P\right)^{D-1}$ let us fix $x_{-(D+1)}=y \in \mathrm{Support}\left(P\right)^{D-1}.$ Then we know $p(y,x)=0$ for all $x \in \mathrm{Support}\left(P\right)$ or equivalently $$\sum_{u=0}^{\omega(P,k)} p_u \left(y\right)  x^u=0,$$for all $x \in \mathrm{Support}\left(P\right).$ By the induction base case of $d=1$ we have \begin{align}\label{puzero2}p_u \left(y\right) =0,\end{align} for all $u=0,1,2,\ldots,\omega(P,k).$ Using now that $y$ was arbitrary we conclude \begin{align}\label{puzero}p_u \left(y\right) =0,\end{align} for all $y \in \mathrm{Support}\left(P\right)$ and for all $u=0,1,2,\ldots,\omega(P,k).$ Using now the induction hypothesis for $d=D-1$ to $p_u$ for each $u=0,1,2,\ldots,\omega(P,k)$ from its definition in we have that for all $u=0,1,2,\ldots,\omega(P,k)$ for any $\alpha \in \mathcal{C}$ with $\alpha_{d+1}=u$ it holds $\lambda_{\alpha}=0.$ Since $u$ can only take these $\omega(P,k)+1$ values we have $\lambda_{\alpha}=0$ for all $\alpha \in \mathcal{C},$as we wanted.

\end{proof}
Under the assumptions of Theorem \ref{cond:main_2} the following holds. \begin{itemize} \item[(a)] the set $\mathcal{B}_{1,\mathcal{C}}$ is a basis for the vector space $\mathcal{P}_{\mathcal{C},d}$. \item[(b)] $\mathcal{P}_{\mathcal{C},d}$ can be equipped with the non-degenerate inner product defined for $p,q \in \mathcal{P}_{\mathcal{C},d}$, given $\inner{p,q}_{P^{\otimes k}}:=\E{p(X)q(X)}$, where $X$ is drawn from the product measure $P^{\otimes k}$.
\end{itemize} 
Indeed, for part ${\rm (a)}$, note that $\mathcal{B}_{1,d}$ is an independent set in $\mathcal{P}_{\mathcal{C},d}$. Indeed if $\sum_{\alpha \in \mathcal{C}} \lambda_{\alpha}x_{\alpha}=0$ for all $x \in \mathbb{R}^d$ then 
$\sum_{\alpha \in \mathcal{C}} \lambda_{\alpha}x_{\alpha}=0$ for all $x \in \mathrm{Support}\left(P\right)^d$. The result then follows from Lemma \ref{fund}.

For part ${\rm (b)}$, it suffices to show the non-degeneracy property as the multilinearity of $\inner{\cdot,\cdot}_P$ follows from the linearity of expectation. Let $p \in \mathcal{P}_{\mathcal{C},d}$ with $\E{p(X)^2}=0$, where $X$ is drawn from the product measure $P^{\otimes k}$. Then $p(x)=0$ for all $x \in \mathrm{Support}\left(P\right)^d$ which according to  Lemma \ref{fund} implies that it is the zero polynomial as desired.

We now define for every $\alpha \in \mathcal{C}$, the following $d$-variate polynomials 
\begin{align}\label{generalT}
    T_{\alpha}(X)=\prod_{i=1}^d T_{\alpha_i}(X_i)
\end{align} where $(T_n)_{n=0,1,\ldots,\omega(P,k)} \in \mathcal{B}_2$ and \begin{align}
    \mathcal{T}:=\{T_{\alpha}, \alpha \in \mathcal{C}\}.
\end{align}Note that from the definition of $\mathcal{F},$ $$\mathcal{T} \subseteq \mathcal{P}_{\mathcal{C},d}.$$

Furthermore we have the following lemma.
\begin{lemma}\label{OrthoMulti}The following properties hold for $\mathcal{T}$ and $X$ drawn from the product measure $P^{\otimes d}.$
\begin{itemize}
\item[(1)] For $\alpha,\beta \in \mathcal{C}$, \begin{align}\label{orthonorm2}
\E{T_{\alpha}(X)T_{\beta}(X)}=1(\alpha=\beta).
\end{align}
    \item[(2)] $\mathcal{T}$ is an orthonormal basis of $\mathcal{P}_{\mathcal{C},d}$, with respect to the inner product $\inner{\cdot,\cdot}_{P^{\otimes k}}$.
    \item[(3)] For every $\alpha, \beta \in \mathcal{C}$ with $\alpha_i<\beta_i$ for some $i=0,1,2,\ldots,d$, $$\E{X_{\alpha}T_{\beta}(X)}=0.$$
    \item[(4)] For every $\alpha \in \mathcal{C}$ , $$\E{X_{\alpha}T_{\alpha}(X)}=\prod_{i: \alpha_i \not = 0} \left(\frac{D_{\alpha_i}}{D_{\alpha_i-1}}\right)^{\frac{1}{2}}.$$
    \item[(5)] It holds $$\min_{\alpha \in \mathcal{C}} \E{X_{\alpha}T_{\alpha}(X)} \geqslant \min\{\min_{i=0,1,\ldots,\omega(P,k)}\left(\frac{D_i}{D_{i-1}}\right)^{\frac{k}{2}},1\}$$ and \begin{align}\label{secondG}\max_{\alpha \in \mathcal{C}} \E{X_{\alpha}T_{\alpha}(X)} \leqslant \max\{\max_{i=0,1,\ldots,\omega(P,k)}\left(\frac{D_i}{D_{i-1}}\right)^{\frac{k}{2}},1\}.\end{align}  
    \item[(6)] It holds
    $$\max_{\alpha,\beta \in \mathcal{C}} \E{X_{\alpha}T_{\beta}(X)}^2 \leqslant \max \{ c^{k}_{2\omega(P,k)},1\}.$$ 
\end{itemize}
\end{lemma}
\begin{proof} 

 Notice that because $X=(X_1,\ldots,X_d)$ has i.i.d.\ coordinates drawn from $P$,
 \begin{align}\label{OrthoProd0} \E{T_{\alpha}(X)T_{\beta}(X)}=\prod_{i=1}^d \E{T_{\alpha_i}(X_i)T_{\beta_i}(X_i)},
\end{align}and\begin{align}\label{OrthoProd} \E{X_{\alpha}T_{\beta}(X)}=\prod_{i=1}^d \E{X^{\alpha_i}_{i}T_{\beta_i}(X_i)},
\end{align}

The first part follows from (\ref{OrthoProd0}) and (\ref{orthonorm1}).

Notice that the first part implies that $\mathcal{T}$ is an orthonormal family of polynomials in  $\mathcal{P}_{\mathcal{C},d}$ with respect to the inner product $\inner{\cdot,\cdot}_{P^{\otimes k}}.$ Furthermore, since $|\mathcal{T}|=|\mathcal{B}_{1,\mathcal{C}}|$ we conclude that $\mathcal{T}$ is an orthonormal basis of $\mathcal{P}_{\mathcal{C},d}$. This proves the second part of the Lemma.

The third part of the Lemma follows from (\ref{OrthoProd}) and the second part of Lemma \ref{ortho1}, while the fourth part of the Lemma follows from (\ref{OrthoProd}) and the third part of Lemma \ref{ortho1}.

For the fourth part notice that any $\alpha \in \mathcal{C}$ it holds by definition \begin{align}\label{inftynorm}\|\alpha\|_{\infty} \leqslant \omega(P,k) \end{align} and second, as all $\alpha_i$'s are nonegative integers, \begin{align}\label{zeronorm}\|\alpha\|_0 \leqslant |\alpha| \leqslant k.\end{align} Hence, for any $\alpha \in \mathcal{C}$, using the third part
\begin{align*} \E{X_{\alpha}T_{\alpha}(X)} 
& = \prod_{i: \alpha_i \not = 0} \left(\frac{D_{\alpha_i}}{D_{\alpha_i-1}}\right)^{\frac{1}{2}}\\
& \geqslant  \prod_{i: \alpha_i \not = 0} \min_{j=1,2,\ldots,\omega(P,k)}\left(\frac{D_{j}}{D_{j-1}}\right)^{\frac{1}{2}}\\
& \geqslant \prod_{i: \alpha_i \not = 0} \min\{\min_{j=1,2,\ldots,\omega(P,k)}\left(\frac{D_{j}}{D_{j-1}}\right)^{\frac{1}{2}},1\}\\
&=\min\{\min_{i=0,1,\ldots,\omega(P,k)}\left(\frac{D_i}{D_{i-1}}\right)^{\frac{k}{2}},1\}.
\end{align*}

Similarly one concludes (\ref{secondG}).

For the sixth part, notice that for all $i$, using Cauchy-Schwartz inequality and the first part of the present Lemma in that order, we have \begin{align}\label{CS}\E{X_{\alpha}T_{\beta}(X)}^2 \leqslant \E{X^2_{\alpha}}\E{T^2_{\beta}(X)}=\E{X^2_{\alpha}}.\end{align}
Now by independence,
\begin{align} \E{X^2_{\alpha}}&=\prod_{i=1}^d \E{X^{2\alpha_i}}\\
&=\prod_{i=1}^d c_{2\alpha_i} \notag \\
&=\prod_{i: \alpha_i \not =0} c_{2\alpha_i} \notag  \\
&\leqslant \prod_{i: \alpha_i \not =0} \max\{c_{2\omega(P,k)},1\}, \text{, by Holder's inequality and (\ref{inftynorm})} \notag \\
& \leqslant \max\{c^k_{2\omega(P,k)},1\} \text{, by (\ref{zeronorm})}. \label{CS2}
\end{align}

\end{proof}

\subsection{Auxilary Results: Decomposition of $\Sigma$ }
Let us define $V$ the $|\mathcal{C}| \times |\mathcal{C}|$ matrix with $$V_{\alpha,\beta}=\frac{1}{\E{X_{\beta}T_{\beta}(X)}}\E{X_{\alpha}T_{\beta}(X)}, \alpha, \beta \in \mathcal{C}.$$
Furthermore, let us define $D$ the diagonal $|\mathcal{C}| \times |\mathcal{C}|$ matrix with $$D_{\alpha,\beta}=\E{X_{\alpha}T_{\alpha}(X)}^2 1(\alpha=\beta), \alpha, \beta \in \mathcal{C}.$$ 

 \begin{lemma}\label{cond:sigmbound2} It holds that
\begin{itemize}
\item[(1)] \begin{align}\label{cond:decomp} 
\Sigma=VDV^t,
\end{align}
\item[(2)]  \begin{align}\label{cond:eigenbound31} 
\lambda_{\min}(\Sigma) \geqslant  \min_{\alpha \in \mathcal{C}} \E{X_{\alpha}T_{\alpha}(X)}^2  \sigma^2_{\min}(V),
\end{align}and 
\item[(3)]  \begin{align}\label{cond:eigenbound31-2} 
\lambda_{\max}(\Sigma) \leqslant  \max_{\alpha \in \mathcal{C}} \E{X_{\alpha}T_{\alpha}(X)}^2 \sigma^2_{\max}(V).
\end{align}
\end{itemize}

\end{lemma}
\begin{proof}
Consider the column vector $\tau:=(T_{\alpha}(X))_{\alpha \in \mathcal{C}}$. Using the second part of Lemma \ref{OrthoMulti} and standard properties of an orthonormal basis of a vector space we have for every $\alpha \in \mathcal{C}$, $$X_{\alpha}=\sum_{\beta \in \mathcal{C}} \E{X_{\alpha}T_{\beta}(X)}T_{\beta}(X).$$ Hence we can immediately deduce that $\mathcal{X}$ defined in the statement of Theorem \ref{cond:main_2} satisfies \begin{align}\mathcal{X}=VD^{\frac{1}{2}}\tau. \end{align}Hence,
$$\E{\mathcal{X} \mathcal{X}^t}=VD^{\frac{1}{2}}\E{\tau \tau^t}D^{\frac{1}{2}}V^t.$$ From the first part of Lemma \ref{OrthoMulti}  we have $$\E{\tau \tau^t}=\E{(T_{\alpha}(X)T_{\beta}(X))_{\alpha,\beta \in \mathcal{C}}}= I_{|\mathcal{C}| \times |\mathcal{C}|}.$$Hence,
\begin{align*} 
\Sigma=VD^{\frac{1}{2}} D^{\frac{1}{2}}V^t=VDV^t,
\end{align*} which is (\ref{cond:decomp}) as desired.

Now we have that for each $\alpha \in \mathcal{C},$ $D_{\alpha,\alpha} \geqslant \min_{\alpha \in \mathcal{C}} \E{X_{\alpha}T_{\alpha}(X)}^2. $ Therefore the matrix $D-\min_{\alpha \in \mathcal{C}} \E{X_{\alpha}T_{\alpha}(X)}^2I_{|\mathcal{C}| \times |\mathcal{C}|}$ is a diagonal matrix with non-negative entries. In particular, $$D - \min_{\alpha \in \mathcal{C}} \E{X_{\alpha}T_{\alpha}(X)}^2 I_{|\mathcal{C}| \times |\mathcal{C}|},$$is a positive semidefinite matrix. Using the simple property that if $A$ is a positive semi-definite matrix then for any square matrix $B$, the matrix $B^tAB$ is also positive semi-definite we conclude that the matrix $$V(D - \min_{\alpha \in \mathcal{C}} \E{X_{\alpha}T_{\alpha}(X)}^2 I_{|\mathcal{C}| \times |\mathcal{C}|})V^t=\Sigma-\min_{\alpha \in \mathcal{C}} \E{X_{\alpha}T_{\alpha}(X)}^2 VV^t,$$is also positive semi-definite.

Now notice that if $A-B$ is a positive semi-definite where $A,B$ are both symmetric square matrices, then \begin{align}\label{goalEigen}\lambda_{\min}(A) \geqslant \lambda_{\min}(B).\end{align} Indeed for a unit-norm $v$ such that $v^tAv=\lambda_{\min}(A)$ it holds $$v^t(A-B)v \geqslant 0$$or equivalently $$\lambda_{\min}(A) \geqslant v^tBv.$$By the Courant-Fischer principle, for any unit norm $v,$ $$v^tBv \geqslant \lambda_{\min}(B).$$Combining the last two displayed equation yields (\ref{goalEigen}). Hence as $$\Sigma-\min_{\alpha \in \mathcal{C}} \E{X_{\alpha}T_{\alpha}(X)}^2 VV^t$$ is positive semi-definite we conclude that \begin{align}\label{cond:eigenbound2} 
\lambda_{\min}(\Sigma) \geqslant \min_{\alpha \in \mathcal{C}} \E{X_{\alpha}T_{\alpha}(X)}^2\lambda_{\min}(VV^t)= \min_{\alpha \in \mathcal{C}} \E{X_{\alpha}T_{\alpha}(X)}^2\sigma^2_{\min}(V),
\end{align}as we wanted.

The third part follows by a similar argument by observing that for each $\alpha \in \mathcal{C}$ $D_{\alpha,\alpha} \leqslant \max_{\alpha \in \mathcal{C}} \E{X_{\alpha}T_{\alpha}(X)}^2$ and therefore the matrix $\max_{\alpha \in \mathcal{C}} \E{X_{\alpha}T_{\alpha}(X)}^2I_{|\mathcal{C}| \times |\mathcal{C}|}-D$ is a positive semi-definite diagonal matrix.
\end{proof}
\begin{lemma}\label{cond:eigenfinal2}
The following holds for the matrix $V$.
\begin{itemize}
    \item[(1)]

\begin{align}\label{det2} \mathrm{det}(V)=1.\end{align}

\item[(2)] 
\begin{align}\label{singular2} 
\sigma_{\min}^2(V) \geqslant \frac{|\mathcal{F}_{\omega(P,k),k}|}{|\mathcal{C}|} \geqslant \frac{1}{|\mathcal{C}|}.
\end{align}

\item[(3)] 
\begin{align}\label{singular3} 
\sigma_{\max}^2(V) \leqslant |\mathcal{C}|^2\|V\|^2_{\infty}.
\end{align}

\end{itemize}
\end{lemma}

\begin{proof}
To compute the determinant and bounds on the singular values of $V$ we can reorder the columns and rows of $V$ using a common permutation at will. For this reason for the purposes of this proof, we choose the following ordering. We order the columns/rows by first placing the elements of $\mathcal{F}_{\omega(P,k),k}$, then the elements of $\mathcal{F}_{\omega(P,k),k-1}$, and downwards all the way down to $\mathcal{F}_{\omega(P,k),0}.$ We choose the ordering between the elements of each of the $k+1$ groups arbitrarily.

From the definition of $V$ and part (3) of Lemma \ref{OrthoMulti} we have that $V_{\alpha,\beta}=0$ unless $\alpha_i \geqslant \beta_i $ for all $i=1,2,\ldots,d$. Hence, according to this choice of ordering makes $V$ an upper triangular matrix. To see this is let us consider $\alpha,\beta$ with $V_{\alpha,\beta} \not = 0$. We have proven that it necessarily holds $|\alpha| \geqslant |\beta|.$ Let us distinguish now two cases to proceed. If $|\alpha|=|\beta|$, then the ordering constraint implies that it should be true that $\alpha=\beta$. If $|\alpha|>|\beta|$, then as $\alpha \in \mathcal{F}_{\omega(P,k),|\alpha|}$ and $\beta \in \mathcal{F}_{\omega(P,k),|\beta|}$ by the definition of the ordering the column/row $\alpha$ is places earlier compared to the column/row $\beta$. Therefore the entry $(\alpha,\beta)$ is indeed placed above the diagonal, and $V$ is proven to be upper triangular. Finally, as $V_{\alpha,\alpha}=1$ for all $\alpha$, $V$ is an upper triangular matrix with diagonal elements equal to one. The first equality (\ref{det2}) follows.

The ordering implies that the first $|\mathcal{F}_{\omega(P,k),k}|$ columns of $V$ have one non-zero entry equal to one and the rest equal to zero. This follows from the fact that for any $\beta \in \mathcal{C}$ with $|\beta|=k$ the only $\alpha \in \mathcal{C}$ with $\alpha_i \geqslant \beta_i$ for all $i=1,2,\ldots,d$, must satisfy $\alpha=\beta$ as otherwise $|\alpha|$ should have been strictly larger than $|\beta|=k$. Hence, if $c_i \in \mathbb{R}^{\mathcal{C}}, i=1,2,\ldots,|\mathcal{C}|$ is an enumeration of the columns of $V$, it holds \begin{align}\label{cond:columns}\|c_i\|^2_2=1 \text{, for } i=1,\ldots,|\mathcal{F}_{\omega(P,k),k}|. \end{align}

We now consider two cases. If $\sigma_{\max}(V)=\sigma_{\min}(V)$, then all the singular values of $V$ are equal. Since $\mathrm{det}(V)=1$ we conclude from elementary linear algebra $\prod_{i=1}^{\mathcal{C}}\sigma_i(V)=1$. Combined with the observation that all singular values in this case are equal to each other, we conclude $\sigma_{\min}(V)=1$. In particular as $|\mathcal{F}_{\omega(P,k),k}| \leqslant |\mathcal{C}|,$ (\ref{singular2}) holds in this case.

We now consider the case where $\sigma_{\max}(V) \not =\sigma_{\min}(V)$. Then using \cite [Theorem 3.1. (a)]{hlavackova2010new} and that $\mathrm{det}(V)=1$ the following identity holds, 
\begin{align}\label{singular}
\sigma_{\min}^2(V) \geqslant \frac{\sigma^2_{\max}(V)-\frac{\|V\|^2_F}{|\mathcal{C}|}}{\sigma^2_{\max}(V)-1}, \end{align}where $\|V\|^2_F:=\sum_{i,j=1}^{|\mathcal{C}|}V^2_{i,j}.$ Using the max-min principle for singular values and $e_i,i=1,\ldots,|\mathcal{C}|$ the standard basis elements of $\mathbb{R}^{|\mathcal{C}|}$ we have \begin{align}\label{cond:inf}\sigma_{\max}(V)=\max_{x:\|x\|_2=1}\|Vx\|_2 \geqslant \max_{i=1,\ldots,|\mathcal{C}|} \|Ve_i\|_2= \max_{i=1,\ldots,|\mathcal{C}|} \|c_i\|_2 \end{align}

Now we have the following inequalities:
\begin{align}
\frac{\|V\|^2_F}{|\mathcal{C}|}&=\frac{1 }{|\mathcal{C}|}\sum_{i=1}^{|\mathcal{C}|} \|c_i\|^2_2 \notag \\
&=\frac{1 }{|\mathcal{C}|}\sum_{i=1}^{\sum_{i=0}^k|\mathcal{F}_{\omega(P,k),i}|} \|c_i\|^2_2, \quad\quad\left(\text{using } |\mathcal{C}|=  \notag \sum_{i=0}^k|\mathcal{F}_{\omega(P,k),i}| \notag \right)\\
& \leqslant \frac{1}{|\mathcal{C}|} \left(\sum_{i=1}^{|\mathcal{F}_{\omega(P,k),k}|} \|c_i\|^2_2+\sum_{i=1}^{\sum_{i=0}^{k-1} |\mathcal{F}_{\omega(P,k),i}|}\max_{i=1,\ldots,|\mathcal{C}|} \|c_i\|^2_2 \right) \notag \\
& \leqslant \frac{1}{|\mathcal{C}|}\left( |\mathcal{F}_{\omega(P,k),k}|+\sigma^2_{\max}(V) \sum_{i=0}^{k-1} |\mathcal{F}_{\omega(P,k),i}| \right), \quad\quad\left(\text{using (\ref{cond:columns}, \ref{cond:inf})}\right) \notag\\
&=\frac{|\mathcal{F}_{\omega(P,k),k}|}{|\mathcal{C}|}+\sigma^2_{\max}(V) \left(1-\frac{|\mathcal{F}_{\omega(P,k),k}|}{|\mathcal{C}|}\right),
\end{align}
which equivalently gives 
\begin{align}
\sigma^2_{\max}(V)- \frac{\|V\|^2_F}{|\mathcal{C}|}  \geqslant \frac{|\mathcal{F}_{\omega(P,k),k}|}{|\mathcal{C}|}\left(\sigma^2_{\max}(V)-1\right) \label{cond:ineqFrob}
\end{align}
Combining (\ref{singular}) with (\ref{cond:ineqFrob}) implies the first inequality of (\ref{singular2}). The second inequality of (\ref{singular2}) follows from the way $\mathcal{C}$ is defined and in particular \eqref{eq:fst-nonempty}.  This completes the proof of the second part of the lemma Lemma.

For the third part, notice that by a standard result by Schur \cite{schur1911bemerkungen,nikiforov2007revisiting} we have $$\sigma^2_{\max}(V) \leqslant \max_{i,j=1,\ldots,d}\|V^te_i\|_1\|Ve_j\|_1.$$ Using the crude bound that for all $i,j=1,\ldots,d$ we have $\max\{\|V^te_i\|_1, \|Ve_j\|_1\} \leqslant |\mathcal{C}| \|V\|_{\infty}$ yields (\ref{singular3}).


\end{proof}

\subsection{Proof of Theorem \ref{cond:main_2}}
We now proceed with the proof of the Theorem \ref{cond:main_2}.
\begin{proof}[Proof of Theorem \ref{cond:main_2}]

We start with two elementary counting arguments.
First notice that $\mathcal{C}$ is contained in the set of all $\alpha \in \mathbb{N}^d$ with $|\alpha| \leqslant k$. Hence, by standard counting arguments 
\begin{align}\label{count1}
    |\mathcal{C}| \leqslant \sum_{i=0}^k \binom{d+i-1}{i}.
    \end{align}

Now to prove the first part of the Theorem we first combine the second part of Lemma \ref{cond:sigmbound2}, the second part of Lemma \ref{cond:eigenfinal2} and  the fifth part of Lemma \ref{OrthoMulti} to conclude,
\begin{align}\label{step1}
    \lambda_{\min}(\Sigma) \geqslant \frac{1}{|\mathcal{C}|} \min\{\min_{i=0,1,\ldots,\omega(P,k)}\left(\frac{D_i}{D_{i-1}}\right)^{\frac{k}{2}},1\} 
\end{align}
Now 
Combining (\ref{count1},\ref{step1}) we conclude (\ref{minEgein2}) , that is the first part of the Theorem \ref{cond:main_2}.

For the second part we use the third part of Lemma \ref{cond:sigmbound2}, the third part of Lemma \ref{cond:eigenfinal2} and conclude
\begin{align}\label{step1_2}
    \lambda_{\max}(\Sigma) \leqslant \max_{\alpha \in \mathcal{C}}\E{X_{\alpha}T_{\alpha}(X)}^2 \left( \frac{\max_{\alpha,\beta \in \mathcal{C}} \E{X_{\alpha}T_{\beta}(X)}}{\min_{\alpha \in \mathcal{C}} \E{X_{\alpha}T_{\alpha}(X)} }|\mathcal{C}|\right)^2
\end{align}
Now using the fifth and sixth parts of Lemma \ref{OrthoMulti} we have 
\begin{align}\label{step2_2}
    \lambda_{\max}(\Sigma) \leqslant \max\{c^{k}_{2\omega(P,k)},1\}  \left( \frac{\max\{\max_{i=0,1,\ldots,\omega(P,k)}\left(\frac{D_i}{D_{i-1}}\right)^{k},1\}}{\min\{\min_{i=0,1,\ldots,\omega(P,k)}\left(\frac{D_i}{D_{i-1}}\right)^{k},1\} }\right)|\mathcal{C}|^2
\end{align}
Combining (\ref{step2_2},\ref{count1}) we conclude (\ref{maxEgein2}) , that is the second part of the Theorem \ref{cond:main_2}.

\end{proof}
\section{Proofs}\label{sec:pfs-main}
\subsection*{Auxiliary Results}
Our results are crucially based on controlling the spectrum of a certain random matrix, constructed using the tensorized data. The main tool that we use is the following concentration result by Vershynin \cite{vershynin2010introduction}, which we record herein verbatim for convenience. 
\begin{theorem}(\cite[Theorem~5.44]{vershynin2010introduction})\label{thm:vershy}
Let $A$ be an $N\times n$ matrix whose rows $A_i$ are independent random vectors in $\R^n$ with the common second moment matrix $\Sigma=\mathbb{E}[A_iA_i^T]$. Let $m$ be a number such that $\|A_i\|_2\leqslant \sqrt{m}$ almost surely for all $i$. Then, for every $t\geqslant 0$, the following inequality holds with probability at least $1-n\cdot \exp(-ct^2)$:
$$
\left\|\frac{1}{N} A^T A - \Sigma\right\|\leqslant \max\left(\|\Sigma\|^{1/2}\delta,\delta^2\right)\quad\text{where}\quad \delta=t\sqrt{m/N}.
$$
Here, $c>0$ is an absolute constant. 
\end{theorem}

Finally, some of our results will also make use of the following version of Hoeffding's inequality.
\begin{theorem}(\cite{hoeffding1994probability})\label{thm:vershy-2}
Let $X_1,\dots,X_N$ be independent centered random variables, where there exists a $K>0$ such  that $|X_i|\leqslant K$ almost surely, for every $i$. Then for every $a=(a_1,\dots,a_N)\in\R^N$ and every $t\geqslant 0$, we have
$$
\mathbb{P}\left(\left|\sum_{i=1}^N a_iX_i\right|\geqslant t\right)\leqslant e\cdot\exp\left(-\frac{ct^2}{K^2\|a\|_2^2}\right),
$$
where $c>0$ is an absolute constant. 
\end{theorem}

\subsection{Proof of Proposition \ref{prop:activation-admissible}}\label{sec:pf-prop:activation-admissible}
\begin{proof}
Recall that ${\rm ReLU}(x)=\max(x,0)=(x+|x|)/2$.  Jackson's theorem \cite{newman1964rational,goel2016reliably} states that the absolute value function $|x|$ has order of approximation $O(1/\epsilon)$. It then follows that, $\varphi_{\rm ReLU}(\epsilon) = O(1/\epsilon)$. Clearly, $\sigma(\cdot)$ is $1-$Lipschitz. Finally, as shown in \cite[Lemma~2.12]{goel2016reliably} for every $\epsilon>0$ there exists a $P(x)\in S(\varphi_\sigma(\epsilon),\sigma,\epsilon)$ such that $P([-1,1])\subseteq [0,1]$.

For the sigmoid activation, $\sigma(x)=\exp(x)/(1+\exp(x))$, the approximation order is known to be $\varphi_{\sigma}(\epsilon) = O(\log(1/\epsilon))$ (see \cite{goel2017learning} and references therein). Furthermore, mean value theorem yields that $\sigma(\cdot)$ is $1-$Lipschitz: note that, for any $x<y$,
$$
\sigma'(x) = \frac{e^x}{(1+e^x)^2}\leqslant 1\Rightarrow |\sigma(x)-\sigma(y)|\leqslant \sup_{z\in[x,y]}|\sigma'(z)|\cdot |x-y|\leqslant |x-y|.
$$
Moreover, any $P(x)\in S(\varphi_\sigma(\epsilon),\sigma,\epsilon)$ also enjoys $P([-1,1])\subseteq [0,1]$ provided $\epsilon\leqslant \sigma(-1)= 1/(1+e)$: let $P(x)$ be such that $\sup_{x\in[-1,1]}|\sigma(x)-P(x)|\leqslant \epsilon$. Then,
$0\leqslant \sigma(-1)-\epsilon\leqslant P(x)\leqslant \sigma(1)+\epsilon\leqslant 1$, for every $x\in[-1,1]$, using the observation that $\sigma(1)+\sigma(-1)=1$.
\end{proof}

\subsection{Proof of Theorem \ref{thm:the-main-for-poly}}\label{sec:pf-main-poly-act} 
This section is devoted to the proof of Theorem \ref{thm:the-main-for-poly}. 
We first record the following auxiliary result.
\begin{theorem}\label{thm:auxiliary}
Let $p\in\mathbb{Z}^+$, and $P(x):\R^p\to R$ be a polynomial. Then, either $P(x)=0$ identically, or the set $\{x\in\R^p: P(x)= 0 \}$ has (Lebesgue) measure $0$. 
\end{theorem}
Namely, a polynomial of real variables is either identically zero, or almost surely non-zero. This fact is an essentially folklore result. See \cite{caron2005zero}  for  a proof.

The next result establishes that the input-output relationship for a polynomial network is essentially an instance of a (noiseless) regression problem. 
\begin{lemma}\label{lem:poly2}
Let $X \in S^d$ for some (possible finite) input $S \subseteq \mathbb{R}$, and let the activation function be a polynomial with degree $k$. Let $\omega(S,k^L)=\min\{|S|-1, k^{L}\},$ per Equation \eqref{eq:omega-Pt}. Then there exists a sequence of vectors $\mathcal{T}^{(t)} \in \mathbb{R}^{|\mathcal{F}_{\omega(S,k^L),t}| }$, $t\in[k^L]$, such that for every $X\in S^d$
\begin{equation}\label{eq:tensor-3New22}
f(\mathcal{W}^*,X) = \sum_{t=0}^{k^L}\ip{\Xi^{(t)}(X)}{\mathcal{T}^{(t)}}.
\end{equation}
\end{lemma}
\begin{proof}
Observe that $f(\mathcal{W}^*,X)$ is a polynomial of degree $k^L$. Thus, there exists $\widetilde{\mathcal{T}}^{(t)}_{i_1,\dots,i_t}$, $0\leqslant t\leqslant k^L; 1\leqslant i_1\leqslant \cdots\leqslant i_t\leqslant d$ such that 
\begin{equation}\label{eq:tilde-T}
f(\mathcal{W}^*,X)  = \sum_{t=0}^{k^L}\sum_{1\leqslant i_1\leqslant \cdots\leqslant i_t\leqslant d}X_{i_1}\cdots X_{i_t} \widetilde{\mathcal{T}}^{(t)}_{i_1,\dots,i_t},\quad\text{for all}\quad X\in S^d.
\end{equation}
Note  that
$$
X_{i_1}\cdots X_{i_t} = X_1^{\alpha_1}\cdots X_d^{\alpha_d},
$$
where $(\alpha_1,\dots,\alpha_d)$ is such that $\sum_{j=1}^d \alpha_j=t$, and $\alpha_j=|\{s\in[t]:i_s=j\}|$.

With this, we now deduce there exists 
$$
\widehat{\mathcal{T}}^{(t)}_{\alpha_1,\dots\alpha_d},\quad\text{with}\quad \sum_{j=1}^d \alpha_j=t, \alpha_j\in\mathbb{Z}_{\geqslant 0};\quad\text{and}\quad 0\leqslant t\leqslant k^L,
$$
such that
\begin{equation}\label{eq:hat-T} 
f(\mathcal{W}^*,X) = \sum_{t=0}^{k^L}\sum_{(\alpha_j: j\in[d]): \sum_{j=1}^d\alpha_j =t} \widehat{T}^{(t)}_{\alpha_1,\dots,\alpha_d}X_1^{\alpha_1}\cdots X_d^{\alpha_d},\quad\text{for all}\quad X\in S^d.
\end{equation}
Note that the inner sum above runs over $(\alpha_1,\dots,\alpha_d)\in\mathcal{F}_{t,t}$.

Now, recall that $\omega(S,k^L)=\min\{k^L,|S|-1\}$. In particular, in the case $|S|-1\leqslant k^L$, we now establish that the monomial $X_i^{\alpha_i}$ can be represented as a linear combination of $1,X_i,\dots,X_i^{\omega(S,k^L)}$, for every $\alpha_i\geqslant \omega(S,k^L)+1$. For this, fix such an $\alpha_i$; and let the support of $X$ be $S=\{x_1,\dots,x_{|S|}\}$. We claim there exists $\theta_0,\dots,\theta_{|S|-1}$ such that
$$
x^{\alpha_i} = \sum_{j=0}^{|S|-1}\theta_j x^j, \quad\text{for all}\quad x\in S.
$$
Such a vector $(\theta_0,\dots,\theta_{|S|-1})$ of coefficients indeed exists, since the Vandermonde matrix
$$
\begin{bmatrix}
1 & x_1 & \cdots & x_1^{|S|-1} \\1 & x_2 & \cdots & x_2^{|S|-1}  \\ \vdots &\vdots & \ddots &\vdots \\ 1 & x_{|S|} & \cdots & x_{|S|}^{|S|-1} 
\end{bmatrix}
$$
is invertible. 

Therefore, it follows that the monomial $X_i^{\alpha_i}$ for any $\alpha_i\geqslant \omega(S,k^L)+1$ can be represented as a linear combination of $1,X_i,\dots,X_i^{\omega(S,k^L)}$. Hence, we restrict ourselves further to the indices $(\alpha_1,\dots,\alpha_d)$ with $0\leqslant \alpha_j \leqslant \omega(S,k^L)$. We deduce that $0\leqslant t\leqslant k^L$ there exists $\mathcal{T}^{(t)}_{\alpha_1,\dots,\alpha_d}$ with $\sum_{j=1}^d\alpha_j =t$, $0\leqslant \alpha_j\leqslant \omega(S,k^L)$ (that is, $(\alpha_1,\dots,\alpha_d)\in\mathcal{F}_{\omega(S,k^L),t}$), such that 
$$
f(\mathcal{W}^*,X)=\sum_{t=0}^{k^L} \sum_{\alpha=(\alpha_j:j\in[d])\in\mathcal{F}_{\omega(S,k^L),t}} \mathcal{T}^{(t)}_{\alpha_1,\dots,\alpha_d} X_1^{\alpha_1}\cdots X_d^{\alpha_d} = \sum_{t=0}^{k^L}  \inner{\Xi^{(t)}\left( X \right),\mathcal{T}^{(t)}},
$$
for all $X\in S^d$, where for each $0\leqslant t\leqslant k^L$, 
\[
\inner{\Xi^{(t)}\left( X \right),\mathcal{T}^{(t)}} = \sum_{\alpha\in\mathcal{F}_{\omega(S,k^L),t}} \Xi_\alpha^{(t)}(X) \mathcal{T}^{(t)}_\alpha.
\]\qedhere
\end{proof}
Having established Lemma \ref{lem:poly2}, we now proceed with the proof of Theorem \ref{thm:the-main-for-poly}.
\begin{proof}{(of Theorem \ref{thm:the-main-for-poly})}
To set the stage, recall that $N$ is the number of samples available, $k$ is the degree of activation (at each node), and $L$ is the depth. Furthermore, for the rest of the proof we denote by $S=\mathrm{Support}(P)$ and in particular by definition all realizations of the input satisfy $X \in S^d$ and furthermore, per equations \eqref{eq:omega-Pt}, \eqref{eq:omega-Pt-P} it holds for any $T \in \mathbb{N},$ $\omega(S,T)=\omega(P,T).$

We now proceed with the proof. First, by the Lemma \ref{lem:poly2} since $X \in S^d$ such a neural network computes the function:
$$
Y_i=\sum_{t=0}^{k^L}\sum_{\alpha=(\alpha_j:j\in[d])\in\mathcal{F}_{\omega(S,k^L),t}} \Xi_\alpha^{(i,t)}\mathcal{T}_\alpha^{(t)} \quad\text{where}\quad \Xi_\alpha^{(i,t)}=\prod_{j=1}^d \left(X_j^{(i)}\right)^{\alpha_j}. 
$$
For each $i\in[N]$, define now the vectors
$$
\Xi^{(i)} = \left(\Xi^{(i,t)}_\alpha : \alpha\in\mathcal{F}_{\omega(S,k^L),t},0\leqslant t\leqslant k^L\right)\in\R^\ell,
$$
where 
$$
\ell =\sum_{t=0}^{k^L}\left|\mathcal{F}_{\omega(S,k^L),t}\right|.
$$
We now show,
\begin{equation}\label{eq:card-ell}
\ell\leqslant d^{2k^L}.
\end{equation}
Observe that an elementary counting argument  reveals that for any $t$,
 $|\mathcal{F}_{\omega(S,k^L),t}|\leqslant \binom{d+t-1}{t}$; and therefore \begin{equation}\label{eq:ell-upbd}
 \ell \leqslant 1+\sum_{t=1}^{k^L} \binom{d+t-1}{t}.
\end{equation}
Now, observe that $\binom{d+t-1}{t}$ is the cardinality of the set, $\{(i_1,\dots,i_t):1\leqslant i_1\leqslant \cdots \leqslant i_t\leqslant d\}$, and is trivially upper bounded by $d^t$. Hence, \begin{equation}\label{eq:ell-upbd2}
    \ell \leqslant \sum_{t=0}^{k^L}\binom{d+t-1}{t}\leqslant 1+\sum_{t=1}^{k^L} d^t \leqslant  1+k^L d^{k^L}\leqslant d^{2k^L}.
\end{equation}

Let now $\Xi\in\R^{N\times \ell }$ be the matrix whose $i^{th}$ row is $\Xi^{(i)}$. In the context of regression, this is the ``measurement matrix". Now, let $Y=(Y_1,\dots,Y_N)^T \in\R^N$, and
$$
\mathcal{T} = \left(\mathcal{T}_\alpha^{(t)}:\alpha\in\mathcal{F}_{\omega(S,k^L),t},0\leqslant t\leqslant k^L\right).
$$
In particular, the training data $(X^{(i)},Y_i)\in\R^d\times \R$, $1\leqslant i\leqslant N$ obeys the equation:
\begin{equation}\label{eq:data-obeys-reg}
Y=\Xi \mathcal{T}. 
\end{equation}
\subsubsection*{Proof of Part ${\rm (a)}$}
We now study Equation (\ref{eq:data-obeys-reg}):
 Observe that, 
    $$
    \Xi^T Y = (\Xi^T \Xi)\mathcal{T}. 
    $$
    Suppose now that, $N$ enjoys (\ref{eq:sample-complex-N1}),  that is,
    \begin{equation}\label{eq:n-main-poly-thm}
    N>C^4\left(P,k^L\right) d^{24 k^L},
        \end{equation}
    where $C(\cdot,\cdot)$ is the constant defined in Corollary \ref{coro:number}.  Define the expected covariance matrix
    $$
     \Sigma =\mathbb{E}\left[\Xi^{(1)}\left(\Xi^{(1)}\right)^T\right]\in\mathbb{R}^{\ell \times \ell}.
    $$
    Now, note that $\Xi^{(1)}$ is identical with the vector of multiplicities $\mathcal{X}$ of Corollary \ref{coro:number} where $k$ in Corollary \ref{coro:number} corresponds $k^L$ in here. In particular, Corollary \ref{coro:number} yields that
    $$
    \left(\frac{\sigma_{\max}(\Sigma)}{\sigma_{\min}(\Sigma)}\right)^4<C^4(P,k^L)d^{12k^L}.
    $$
    This, together with Equation (\ref{eq:card-ell}), then yields
    $$
    N>\ell ^2 \left(\frac{\sigma_{\max}(\Sigma)}{\sigma_{\min}(\Sigma)}\right)^4,
    $$

   We now establish that with high probability, $\Xi^T \Xi$ is invertible. The main tool is Theorem \ref{thm:vershy} for the concentration of spectrum of random matrices with i.i.d.\ non-isotropic rows. To apply this result, we prescribe the corresponding parameters as follows:
    \begin{center}
 \begin{tabular}{||c c||} 
 \hline
 Parameter  & Value \\ [0.5ex] 
 \hline\hline
 $m$  & $\ell $ \\
 \hline
 $t$  & $N^{1/8}$ \\
 \hline
 $\delta$ & $N^{-3/8}\ell ^{1/2}$ \\
 \hline
 $\gamma$ & $\max(\|\Sigma\|^{1/2}\delta,\delta^2)$
 \\
 \hline 
\end{tabular}
\end{center}
We now focus on this parameter assignment in more detail. First, note that since the data satisfies $X_j\in[-1,1]$ for every $j\in[d]$, and $\Sigma\in\R^{\ell\times \ell}$, $m$ can indeed be taken to be $\ell$.

We next show $\|\Sigma\|^{1/2}\delta>\delta^2$, and consequently, $\gamma =\max(\|\Sigma\|^{1/2}\delta,\delta^2)= \|\Sigma\|^{1/2}\delta$. Note that this is equivalent to establishing $\|\Sigma\|^{1/2}>\delta=N^{-3/8}\ell^{1/2}$, namely, it is equivalent to establishing
$$
N>\ell^{4/3}\|\Sigma\|^{-4/3}.
$$
In what follows below, recall that the constant $C(P,\cdot)$ is from Corollary \ref{coro:number}, which is defined through two constants $f(P,\cdot)$ (defined in \ref{dpk}) and $c(P,\cdot)$ (defined in \ref{cpk}). 

 We now use $\|\Sigma\|=\sigma_{\max}(\Sigma)\geqslant \sigma_{\min}(\Sigma)$,  the bound (\ref{minEgein2}), and the bound (\ref{eq:ell-upbd2}) to arrive at:
\begin{equation}\label{eq:aux-sigmamin-bd}
\sigma_{\min}(\Sigma)\geqslant c(P,k^L)\frac{1}{\sum_{i=0}^{k^L}\binom{d+i-1}{i}}\geqslant c(P,k^L)d^{-2k^L}.
\end{equation}
Consequently,
$$
\|\Sigma\|^{-4/3}\leqslant \sigma_{\min}(\Sigma)^{-4/3}\leqslant c(P,k^L)^{-4/3}d^{\frac{8}{3}k^L}.
$$
This, together with (\ref{eq:card-ell}) then yields
$$
\ell^{4/3}\|\Sigma\|^{-4/3}\leqslant c(P,k^L)^{-4/3} d^{\frac{16}{3}k^L}.
$$
Following Remark \ref{remark:on-order-between-c-C-f}, we have 
 $C(P,k^L)\geqslant c(P,k^L)^{-1}$. Therefore,
$$
\ell^{4/3}\|\Sigma\|^{-4/3}\leqslant C(P,k^L)^{4/3}d^{\frac{16}{3}k^L}.
$$
 We further have $C(P,k^L)\geqslant 1$ by Remark \ref{remark:on-order-between-c-C-f}. Since $N$ obeys (\ref{eq:n-main-poly-thm}), we deduce immediately that $N>\ell^{4/3}\|\Sigma\|^{-4/3}$.  Thus, $\gamma=\max(\|\Sigma\|^{1/2}\delta,\delta^2)=\|\Sigma\|^{1/2}\delta$.

We now claim
$$
\sigma_{\min}(\Sigma)>\gamma=\|\Sigma\|^{1/2}\delta=\|\Sigma\|^{1/2}N^{-3/8}\ell^{1/2}.
$$
This is equivalent to establishing
$$
N>\frac{\ell^{4/3}\|\Sigma\|^{4/3}}{\sigma_{\min}(\Sigma)^{8/3}}.
$$
Note that Equations (\ref{maxEgein2}) and (\ref{eq:ell-upbd2}) together yield
$$
\|\Sigma\|\leqslant f(P,k^L)\sum_{i=0}^{k^L}\binom{d+i-1}{i}\leqslant f(P,k^L)d^{2k^L}.
$$
This, together with Equations (\ref{eq:ell-upbd2}) and (\ref{eq:aux-sigmamin-bd}), then yield
$$
\frac{\ell^{4/3}\|\Sigma\|^{4/3}}{\sigma_{\min}(\Sigma)^{8/3}}<d^{\frac83 k^L} \frac{f(P,k^L)^{4/3}d^{\frac83 k^L}}{c(P,k^L)^{8/3}d^{-\frac{16}{3} k^L}} = \frac{f(P,k^L)^{4/3}}{c(P,k^L)^{8/3}}d^{\frac{32}{3}k^L}.
$$
Using now $\max(f(P,k^L),c^{-1}(P,k^L))\leqslant C(P,k^L)$ as noted in Remark \ref{remark:on-order-between-c-C-f}, we further obtain
$$
\frac{\ell^{4/3}\|\Sigma\|^{4/3}}{\sigma_{\min}(\Sigma)^{8/3}}<C(P,k^L)^4 d^{\frac{32}{3}k^L}.
$$
Since $C(P,k^L)\geqslant 1$ and $N$ enjoys (\ref{eq:n-main-poly-thm}), we deduce immediately that
$$
N> \frac{\ell^{4/3}\|\Sigma\|^{4/3}}{\sigma_{\min}(\Sigma)^{8/3}},
$$
and consequently $\sigma_{\min}(\Sigma)>\gamma$. Equipped with these facts, we now apply Theorem \ref{thm:vershy}:   
with probability at least $1-\ell  e^{-cN^{1/4}}=1-\exp(-c'N^{1/4})$ (where $c>c'>0$ are two absolute constants), it holds that:
    $$
    \left\|\frac1N \Xi^T \Xi - \Sigma\right\|\leqslant \gamma.
    $$
     Now, 
    $$
     \left\|\frac1N \Xi^T \Xi - \Sigma\right\|\leqslant \gamma \iff \forall v\in\R^{\ell}, \left|\|\frac{1}{\sqrt{N}}\Xi v\|_2^2 - v^T \Sigma v \right|\leqslant \gamma\|v\|_2^2,
    $$
    which implies, for every $v$ on the sphere $\mathbb{S}^{\ell-1} =\{v\in\R^{\ell}:\|v\|_2=1\}$,
    $$
    \frac1N \|\Xi v\|_2^2 \geqslant v^T \Sigma v-\gamma \Rightarrow \frac1N \inf_{v:\|v\|=1}\|\Xi v\|_2^2 \geqslant \inf_{v:\|v\|=1}v^T \Sigma v - \gamma,
    $$
    which, together with the Courant-Fischer variational characterization of the smallest singular value \cite{horn2012matrix}, implies
    \begin{equation}
        \sigma_{\min}(\Xi) \geqslant N(\sigma_{\min}(\Sigma)-\gamma)>0,
    \end{equation}
    with probability at least $1-\exp(-c'N^{1/4})$. In particular since we have established already $\sigma_{\min}(\Sigma)>\gamma$, it turns out with probability at least $1-\exp(-c'N^{1/4})$,  $\Xi^T \Xi$ is invertible.

    
    Observe now that
    $$
    \mathbb{P}\left(\exists \widehat{\mathcal{T}}\neq \mathcal{T}: \Xi\widehat{\mathcal{T}}=\Xi \mathcal{T}\middle\vert{\rm det}(\Xi^T \Xi)\neq 0\right) =0,
    $$
    where $\mathcal{T}$ is defined in Equation (\ref{eq:data-obeys-reg}). Now, let
   $$
    \widehat{\mathcal{T}} = (\Xi^T \Xi)^{-1}\Xi^T Y,
    $$
    be the ``output" of the T-OLS algorithm (Algorithm \ref{algo:tensor-non-poly}). Thus with probability at least $1-\exp(-c'N^{1/4})$,
    $$
    \widehat{\mathcal{T}}=
    \mathcal{T}.
    $$
    Now, let $x\in\R^d$ be arbitrary, and define the corresponding ``tensor"
    \begin{equation}\label{eq:datanin-tensoru}
    \bar{x} \triangleq  \left(x_\alpha:\alpha\in\mathcal{F}_{\omega(P,k),t},0\leqslant t\leqslant k^L\right)\in\R^\ell.
      \end{equation}
    We then deduce that with probability at least $1-\exp(-c'N^{1/4})$ over the randomness of $\{X^{(i)}:1\leqslant i\leqslant N\}$: it holds
    $$
    \ip{\widehat{\mathcal{T}}}{\bar{x}} = f(\mathcal{W}^*,x),
    $$
    for every $x\in\R^d$, where  $f(\mathcal{W}^*,x)$ is the function computed by the  network with polynomial activation, per Equation (\ref{eq:deepnn-main}).
    
    Since $\ip{\widehat{\mathcal{T}}}{\bar{x}}$ is equal to $\widehat{h}(x)$, the output of the T-OLS algorithm, the proof of Part ${\rm (a)}$ of Theorem \ref{thm:the-main-for-poly} is complete.  

\subsubsection*{Proof of Part ${\rm (b)}$}
Suppose now, that the data $X\in\R^d$ has jointly continuous coordinates.

We claim that provided $N$ satisfies (\ref{eq:sample-complex-N1-hat}), that is as soon as $N=\ell$, the matrix $\Xi\in\R^{N\times \ell}$ is invertible with probability one, where  
$$
\ell = \sum_{t=0}^{k^L}|\mathcal{F}_{t,t}| =\sum_{t=0}^{k^L}\binom{d+t-1}{t}.
$$
\begin{proposition}\label{prop:Vandermonde:invertible} 
With probability one over $X^{(i)}$, $1\leqslant i\leqslant N$,
$\Xi$ is invertible.
\end{proposition}
\begin{proof}{(of Proposition \ref{prop:Vandermonde:invertible})}
We use the auxiliary result, Theorem \ref{thm:auxiliary}. Note the following consequence of this result: let $X=(X_i:i\in[p])\in\R^p$ be a jointly continuous random vector with density $f$; and $P:\R^p\to \R$ is a polynomial where there exists an $(x_1',\dots,x_p')\in\R^p$ with $P(x_1',\dots,x_p')\neq 0$. Define $S=\{(x_1,\dots,x_p)\in\R^p:P(x_1,\dots,x_p)=0\}$. Then, 
$$
\mathbb{P}(P(X)=0)=\int_{S\subset \R^p}f(x_1,\dots,x_p)d\lambda(x_1,\cdots,x_p)=0,
$$
where $\lambda$ is the ($p$-dimensional) Lebesgue measure. In particular, for any jointly continuous $X\in\R^p$, and any non-vanishing polynomial $P$, $P(X)\neq 0$ almost everywhere.

Equipped with this, it then suffices to show ${\rm det}(\Xi)$ is not identically zero, when viewed as a polynomial in $X^{(i)}_j$, $i\in[N]$ and $j\in[d]$. We now prove this by providing a deterministic construction (for data $X^{(i)}$, $1\leqslant i\leqslant N$, and thus for $\Xi$) under which ${\rm det}(\Xi)\neq 0$. Let $p_1<\cdots<p_d$ be distinct primes. For each $1\leqslant i\leqslant N$, define
$$
X^{(i)}=(p_1^{i-1},p_2^{i-1},\dots,p_d^{i-1})\in\R^d.
$$
In particular, $X^{(1)}=(1,1,\dots,1)\in\R^d$, vector of all ones. This yields the first row $\mathcal{R}_1\in\R^N$ of $\mathcal{M}$ is the vector of all ones. Now, let $\mathcal{R}_2=(z_1,\dots,z_N)\in\R^N$ be the second row of $\mathcal{M}$, with $z_1=1$ (since a multiplicities set starts with $(0,0,\dots,0)$, corresponding to the term $X_1^0 X_2^0 \cdots X_d^0$). Notice that $z_i\in\mathbb{Z}^+$ for each $i\in[N]$. Moreover, for each $i\neq j$, $z_i\neq z_j$, from the fundamental theorem of arithmetic asserting the uniqueness of prime factorization. Equipped with this, $\Xi$ then becomes
$$
\Xi = \begin{bmatrix}1 & 1 & \cdots & 1 \\ z_1 & z_2 & \cdots & z_N \\ z_1^2 & z_2^2 & \cdots & z_N^2 \\ \vdots & \vdots & \vdots & \vdots \\ z_1^{N-1} & z_2^{N-1} & \cdots & z_N^{N-1} \end{bmatrix}.
$$
This is a Vandermonde matrix with ${\rm det}(\Xi)=\prod_{1\leqslant i<j\leqslant N}(z_i-z_j)\neq 0$. This completes the proof of the fact that $\mathbb{P}(\Xi)\neq 0)=1$.
\end{proof}
Having established that for $N=\ell$, $\Xi$ is invertible with probability one, we then set $\widehat{\mathcal{T}}=\Xi^{-1} Y$, and deduce, in a similar manner to the proof of Part ${\rm (a)}$, that with probability one,
$$
\mathcal{T}=\widehat{\mathcal{T}}.
$$
Now, let $x\in\R^d$ be any arbitrary data, and $\bar{x}$ be its ``tensorized" version as per (\ref{eq:datanin-tensoru}). We then obtain with probability one over the randomness in $\{X^{(i)}:1\leqslant i\leqslant N\}$, it holds
$$
\ip{\widehat{\mathcal{T}}}{\bar{x}}=f(\mathcal{W}^*,x).
$$
Finally, since $\ip{\widehat{\mathcal{T}}}{\bar{x}}$ is precisely the output $\widehat{h}(x)$ of the T-OLS algorithm, we conclude the proof of the Part ${\rm (b)}$ of Theorem \ref{thm:the-main-for-poly}.
\end{proof}
\subsection{Proof of Theorem \ref{thm:the-main}}\label{sec:pf-thm-main}
In this section, we establish Theorem \ref{thm:the-main}. We achieve this through the help of several auxiliary results, which may be of independent interest.  



We first establish that any such neural network with an admissible activation function can be approximated uniformly by a polynomial network, of the same depth and width.
\begin{theorem}\label{thm:deep-approx}
Consider a neural network of width $m^*$ and depth $L$, consisting of admissible activations $\sigma(\cdot)$ and planted weights $\mathcal{W}^*$; which for each $x\in\R^d$ computes $f(\mathcal{W}^*,x)$.  
Suppose $\|a^*\|_{\ell_1}\leqslant 1$, and $\|W_{k,j}^*\|_{\ell_1}\leqslant 1$, where $W_{k,j}^*$ is the $j^{th}$ row of $W_k^*$. Then, for every $\epsilon>0$ there exists a polynomial $P(\cdot)$, of degree $\varphi_\sigma(\epsilon/L)$, such that the following network
$$
\widehat{f}(\mathcal{W}^*,x) = (a^*)^T P(W_L^* P(W_{L-1}^*\cdots P(W_1^* X))\cdots),
$$
satisfies
$$
\sup_{x\in[-1,1]^d}|f(\mathcal{W}^*,x)-\widehat{f}(\mathcal{W}^*,x)|\leqslant \epsilon.
$$
\end{theorem}
\begin{proof}{(of Theorem \ref{thm:deep-approx})}
 Fix an $\epsilon>0$. Let $P(\cdot)$ be a polynomial of degree ${\rm deg}(P)=\varphi_\sigma(\epsilon/L)$ such that
$$
\sup_{z\in[-1,1]}|P(z)-\sigma(z)|\leqslant \epsilon/L \quad\text{and}\quad P([-1,1])\subseteq[0,1].
$$
Such a polynomial exists since $\sigma(\cdot)$ is admissible with order of approximation $\varphi_\sigma(\cdot)$. 
We first claim that the outputs of internal nodes are indeed at most one. The base case is easily verified, since, $|\ip{W_{1,i}}{x}|\leqslant \|W_{1,i}\|_{\ell_1}\leqslant 1$. For the inductive hypothesis, since $|\ip{W_{k,j}}{x}|\leqslant 1$ for any $k\in[L]$ and $j\in[m^*]$, since $P(
\cdot)$ is such that $P([-1,1])\in [0,1]$, we have the desired conclusion. 

Equipped with this, we now proceed with the inductive argument over the depth, in a manner similar to Telgarsky \cite{telgarsky2017neural}. Set $k\leqslant L$, and let 
$$
V_k = W_k\sigma(W_{k-1}\cdots \sigma(W_1 x)\cdots) \in\R^{m^*},\quad\text{and}\quad V_k'=W_k P(W_{k-1}\cdots P(W_1 x)\cdots) \in\R^{m^*}.
$$
Denoting the $i^{th}$ coordinate of $V_k$ by $V_k(i)$, we then get:
\begin{align*}
|\sigma(V_k(i)) - P(V_k'(i))| &\leqslant |\sigma(V_k(i))-\sigma(V_k'(i))|+ |\sigma(V_k'(i))-P(V_k'(i))| \\
&\leqslant |V_k(i)-V_k'(i)| +\epsilon/L,
\end{align*}
since $\sigma(\cdot)$ is $1-$Lipschitz per Definition \ref{def:admissible-activation} ${\rm (c)}$.  Note moreover that since $V_{k+1}(i)=\sigma(V_k(i))$ and $V_{k+1}'(i)=P(V_k'(i))$, we then get:
$$
|V_{k+1}(i)-V_{k+1}'(i)|\leqslant |V_k(i)-V_k'(i)|+\epsilon/L,\quad \forall i\in[m]\quad\text{and}\quad 1\leqslant k\leqslant L-1.
$$
In particular, this recursion, together with the base case, and $\|a^*\|_{\ell_1}\leqslant 1$, imply indeed that,
$$
\sup_{x\in[-1,1]^d}|f(\mathcal{W}^*,x)-\widehat{f}(\mathcal{W}^*,x)|\leqslant \epsilon.
$$
\qedhere
\end{proof}

Our next auxiliary result establishes the existence of an approximating polynomial $P(\cdot)$ for which the resulting approximation error is ``orthogonal" to any linear function of tensors of data.

\begin{proposition}\label{prop:ortho-poly}
Suppose that $X\in\R^d$ is a random vector drawn from an admissible distribution $P^{ \otimes d}$, 
 $\mathcal{C}$ is a $(d,M)-$multiplicities set per \eqref{def:multiplicities};  and $f:\R^d\to\R$ is a Borel-measurable function. Then, there exists coefficients $(q_\alpha:\alpha\in\mathcal{C})\in\R^{|\mathcal{C}|}$, and a polynomial
$$P(X) = \sum_{\alpha\in\mathcal{C}}q_\alpha X_\alpha=\sum_{\alpha=(\alpha_j:j\in[d])\in\mathcal{C}}q_\alpha \prod_{i=j}^d X_j^{\alpha_j},
$$
such that
$$
\mathbb{E}\left[X_\alpha\left(f(X)-P(X)\right)\right]=0,\quad\quad\forall \alpha\in\mathcal{C}.
$$
\end{proposition}
In our application, $f(X)$ will be the function computed by neural network (per Equation (\ref{eq:deepnn-main})). 
\begin{proof}{(of Proposition \ref{prop:ortho-poly})}
The condition one needs to ensure is:
$$
\E{X_\alpha f(X)} = \sum_{\alpha'\in\mathcal{C}} q_{\alpha'} \E{X_{\alpha'} X_{\alpha}}, \quad \forall \alpha\in\mathcal{C}.
$$
Let $\mathcal{C} = \{\alpha(1),\dots,\alpha(M)\}$ (where $\alpha(i)=(\alpha_j(i):1\leqslant j\leqslant d)$), and $\alpha(1)$ is the $d-$dimensional vector of all zeroes). We then need to ensure that the following linear system is solvable:
$$
\underbrace{\begin{bmatrix}
\E{X_{\alpha(1)}^2} & \E{X_{\alpha(1)}X_{\alpha(2)}} & \cdots & \E{X_{\alpha(1)}X_{\alpha(M)}} \\\E{X_{\alpha(2)}X_{\alpha(1)}} & \E{X_{\alpha(2)}^2} & \cdots & \E{X_{\alpha(2)}X_{\alpha(M)}}  \\ \vdots & \vdots & \cdots & \vdots \\ \E{X_{\alpha(M)}X_{\alpha(1)}} & \E{X_{\alpha(M)}X_{\alpha(2)}} & \cdots & \E{X_{\alpha(M)}^2}
\end{bmatrix}}_{\mathcal{M}}\cdot \begin{bmatrix} q_{\alpha(1)}\\ q_{\alpha(2)}\\ \vdots \\ q_{\alpha(M)}\end{bmatrix} = \begin{bmatrix} \E{X_{\alpha(1)}f(X)}\\  \E{X_{\alpha(2)}f(X)}\\ \vdots \\  \E{X_{\alpha(M)}f(X)}\end{bmatrix}.
$$
We now establish that $\mathcal{M}$ is invertible. Since $X\in\R^d$ has i.i.d.\ coordinates with finite moments; Theorem \ref{cond:main_2} yields that the smallest singular value of this matrix is bounded away from zero, which, in turn, yields $\mathcal{M}$ is indeed invertible.  
\end{proof}


Having established the existence of an approximating polynomial orthogonal to data, its error performance is the subject of the next proposition.
\begin{proposition}\label{prop:ortho-poly-error}
Let $X\in \R^d$ be a random vector, $f:\R^d\to\R$ is a Borel-measurable function, $\mathcal{C}$ be a multiplicities set (per \eqref{def:multiplicities}), $P(X)=\sum_{\alpha\in\mathcal{C}}q_\alpha X_\alpha$ with $X_\alpha=\prod_{i=1}^d X_i^{\alpha_i}$ (for $\alpha=(\alpha_1,\dots,\alpha_d)$, with the property that 
$$
\E{X_\alpha(f(X)-P(X))}=0,
$$
for any $\alpha\in\mathcal{C}$. Then, for any $Q(X) = \sum_{\alpha\in\mathcal{C}}\beta_\alpha X_\alpha$, it holds that:
$$
\E{(f(X)-Q(X))^2} \geqslant \E{(f(X)-P(X))^2}.
$$
\end{proposition}
Intuitively speaking, what this results says is that the polynomial that is 'orthogonal' to linear functions of data performs the best among all 'linear estimators'. 
\begin{proof}{(of Proposition \ref{prop:ortho-poly-error})}
Note that, $f(X)-Q(X)=f(X)-P(X)+P(X)-Q(X)$. In particular, we then have:
\begin{align*}
\E{(f(X)-Q(X))^2} & = \E{f(X)-P(X)}^2 + 2\E{(f(X)-P(X))(P(X)-Q(X))} + \E{(P(X)-Q(X))^2} \\
&= \E{(f(X)-P(X))^2}  + \E{(P(X)-Q(X))^2}\\
&\geqslant \E{(f(X)-P(X))^2},
\end{align*}
where the second equality follows due to the fact that $P(X)-Q(X) = \sum_{J\in\mathcal{C}}(\alpha_J-\beta_J)X_J$, and that, 
$$
\E{(f(X)-P(X))\left(\sum_{\alpha\in \mathcal{C}}(q_\alpha-\beta_\alpha)X_\alpha\right)} = \sum_{\alpha\in\mathcal{C}} (q_\alpha-\beta_\alpha) \E{X_\alpha(f(X)-P(X))}=0,
$$
using the orthogonality property of $P$. 
\end{proof}

We now establish the existence of a multivariate polynomial ``supported" on an appropriate multiplicities set, which is orthogonal to data and has small approximation error; by combining Propositions \ref{prop:ortho-poly}, \ref{prop:ortho-poly-error}; together with the existential result of Theorem \ref{thm:deep-approx}.
\begin{theorem}\label{thm:approximate-deep-with-poly}
Suppose the data $X\in\R^d$ follows an admissible distribution; $\sigma(\cdot)$ is an admissible activation with order of approximation $\varphi_\sigma(\cdot)$; and $f(\mathcal{W}^*,X)$ be the function computed by a neural network of depth $L$, where the planted weights $\mathcal{W}^*$ of the network satisfy: $\|a^*\|_{\ell_1}\leqslant 1$; $\|W_{p,j}^*\|_{\ell_1}\leqslant 1$,  for any $p\in[L]$ and any $j$. Fix an $\epsilon>0$. 
Let $M=\varphi_\sigma^L(\sqrt{\epsilon/L^2})\in\mathbb{Z}^+$, and let $\mathcal{C}$ be the $(d,M)-$multiplicities set.
Then, there exists coefficients $(q_\alpha^*:\alpha\in\mathcal{C})\in \R^{|\mathcal{C}|}$, and a polynomial 
$$
P(X) = \sum_{\alpha\in\mathcal{C}}q_\alpha^* X_\alpha = \sum_{\alpha=(\alpha_j:j\in[d])\in\mathcal{C}} q_\alpha^*\cdot  \prod_{j=1}^d X_j^{\alpha_j},
$$
such that:
\begin{itemize}
    \item[(a)] $\E{X_\alpha(f(\mathcal{W}^*,X)-P(X))}=0$ for every $\alpha\in\mathcal{C}$.
    \item[(b)] $\E{(f(\mathcal{W}^*,X)-P(X))^2}\leqslant \epsilon$.
\end{itemize}
\end{theorem}
\begin{proof}{(of Theorem \ref{thm:approximate-deep-with-poly})}
Fix an $\epsilon>0$. Theorem \ref{thm:deep-approx} then yields there exists a neural network computing $\widehat{f}(\mathcal{W}^*,X)$ and consisting only of polynomial activations with degree 
$\varphi_{\sigma}(\sqrt{\epsilon}/L)$ such that
$$
|\widehat{f}(\mathcal{W}^*,X)-f(\mathcal{W}^*,X)|\leqslant \sqrt{\epsilon},\quad \forall X\in[-1,1]^d.
$$
In particular, as a consequence of Lemma \ref{lem:poly2}, $\widehat{f}(\mathcal{W}^*,X)$ can be perceived as a polynomial
$$
\widehat{f}(\mathcal{W}^*,X)\triangleq P'(X)=\sum_{\alpha\in\mathcal{C}}q_\alpha^* X_\alpha =\sum_{\alpha=(\alpha_j:j\in[d])\in\mathcal{C}}q_\alpha^* \prod_{j=1}^d (X_j)^{\alpha_j}, 
$$
supported on a $(d,M)-$multiplicities set $\mathcal{C}$, where $M=\varphi_\sigma^L(\sqrt{\epsilon/L^2})\in\mathbb{Z}^+$. We claim that $
|\mathcal{C}| \leqslant d^{2M}$. Indeed, observe that
$$
|\mathcal{F}_{\omega(P,k),t}|\leqslant \binom{d+t-1}{t} \Rightarrow |\mathcal{C}|\leqslant 1+\sum_{t=1}^{M} \binom{d+t-1}{t}.
$$
Now, observe that $\binom{d+t-1}{t}$ is the cardinality of the set, $\{(i_1,\dots,i_t):1\leqslant i_1\leqslant \cdots \leqslant i_t\leqslant d\}$, and is trivially upper bounded by $d^t$. Hence, $|\mathcal{C}|\leqslant 1+\sum_{t=1}^{M} d^t \leqslant  1+M d^{M}\leqslant d^{2M}$. 

Now, let $P(X)=\sum_{J\in\mathcal{C}} \alpha_J X_J$ be a polynomial such that,
$$
\E{X_\alpha \left(f(\mathcal{W}^*,X)- P(X)\right)}=0,\quad \forall \alpha \in \mathcal{C}.
$$
Such a $P(\cdot)$ exists due to Proposition \ref{prop:ortho-poly}. Finally, using Proposition \ref{prop:ortho-poly-error}, we immediately get:
$$
\mathbb{E}\left[\left(f(\mathcal{W}^*,X)-P(X)\right)^2\right] \leqslant \E{\left(f(\mathcal{W}^*,X)-P'(X)\right)^2}\leqslant \epsilon,
$$
using the fact that $|f(\mathcal{W}^*,X)-Q(X)|\leq\sqrt{\epsilon}$ for any $X\in\R^d$.
\end{proof}

We are now ready to prove Theorem \ref{thm:the-main}. 
\begin{proof}{(of Theorem \ref{thm:the-main})}
Let $M=\varphi_\sigma^L(\sqrt{\epsilon/4L^2})\in\mathbb{Z}^+$. Recall that this is the degree of the approximation stated in the theorem. Using Theorem \ref{thm:approximate-deep-with-poly}, we obtain that there exists a 
$(d,M)-$multiplicities set $\mathcal{C}$ as per \eqref{def:multiplicities} such that, and coefficients $(q_\alpha^*)_{\alpha\in\mathcal{C}}$ such that 
$$
Y_i = \sum_{\alpha\in\mathcal{C}} q_\alpha^* X_\alpha^{(i)} + \mathcal{N}^{(i)},\quad\quad  \E{(\mathcal{N}^{(i)})^2}\leqslant \frac{\epsilon}{4}, \quad\text{and}\quad  \E{X_{\alpha}^{(i)}\mathcal{N}^{(i)}}=0,
$$
for every $\alpha\in\mathcal{C}$ and every $i\in[N]$, where for any $\alpha=(\alpha_1,\alpha_2,\dots,\alpha_d)$:
$$
X_\alpha^{(i)}=\prod_{j=1}^d \left(X_j^{(i)}\right)^{\alpha_j}.
$$
An  exact  similar counting argument, as  in  the proof of Theorem \ref{thm:the-main-for-poly}, then yields
\begin{equation}\label{eq:c-card-main-thm}
    |\mathcal{C}|\leqslant 1+\sum_{t=1}^M  \binom{d+t-1}{t}  \leqslant 1+Md^M\leqslant d^{2M}.
\end{equation}

Let $\Xi\in\R^{N\times |\mathcal{C}|}$ to be the matrix whose $i^{th}$ row is $(X_\alpha^{(i)}:\alpha\in\mathcal{C})\in\R^{|\mathcal{C}|}$; $Y=(Y_1,\dots,Y_N)^T\in\R^N$, $\mathcal{A}^* = (q_\alpha^* :\alpha\in\mathcal{C})\in \R^{|\mathcal{C}|}$; and $\mathcal{N}=(\mathcal{N}^{(1)},\dots,\mathcal{N}^{(N)})^T\in \R^N$. Observe that $\Xi$ consists of i.i.d.\ rows.  We then have
\begin{equation}\label{eq:MAINMAINMAIN}
Y=\Xi \mathcal{A}^* + \mathcal{N}.
\end{equation}
We now observe the following:
\begin{itemize}
\item[(a)] $(\mathcal{N}^{(i)})_{i\in [N]}$ is a  random vector with i.i.d.\ coordinates.
\item[(b)] For any {\em fixed} $\alpha\in\mathcal{C}$, $(X_\alpha^{(i)})_{i\in[N]}$ is an i.i.d.\ sequence of random variables.
\end{itemize}
We now assume the sample size $N$ satisfies (\ref{eq:main-sample-complexity}), that is
\begin{equation}\label{eq:main-sample-complexity-forthm}
N>\frac{2^{24}}{\epsilon^6}d^{96M}C^{18}(P,M).
\end{equation}
In this regime, we study the spectrum of random matrix $\Xi^T \Xi$ with associated per-row of $\Xi$ covariance matrix $\Sigma$. Note that $\Sigma$ simply corresponds to the covariance matrix $\E{\mathcal{X}\mathcal{X}^t}$ defined in Theorem \ref{cond:main_2}. The main concentration of measure tool is the Theorem \ref{thm:vershy} for the concentration of the spectrum of random matrices with i.i.d. non-isotropic rows. The parameter setting tailored to our scenario is as follows:
    \begin{center}
 \begin{tabular}{||c c||} 
 \hline
 Parameter  & Value \\ [0.5ex] 
 \hline\hline
 $m$  & $|\mathcal{C}|$ \\
 \hline
 $t$  & $N^{1/8}$ \\
 \hline
 $\delta$ & $N^{-3/8}|\mathcal{C}|^{1/2}$ \\
 \hline
 $\gamma$ & $\max(\|\Sigma\|^{1/2}\delta,\delta^2)$
 \\
 \hline 
\end{tabular}
\end{center}
We now focus on the parameter setting in detail. Observe first that
since the data satisfies $X_j\in[-1,1]$ for every $j\in[d]$, and $\Sigma\in\R^{|\mathcal{C}|\times |\mathcal{C}|}$, $m$ can indeed be taken to be $|\mathcal{C}|$. We next establish
$$
\|\Sigma\|^{1/2}\delta>\delta^2,
$$
and consequently,
$$
\gamma =\max(\|\Sigma\|^{1/2}\delta,\delta^2)=\|\Sigma\|^{1/2}\delta.
$$
For this, it suffices to ensure
$$
\|\Sigma\|^{1/2}>N^{-3/8}|\mathcal{C}|^{1/2}\iff N>|\mathcal{C}|^{4/3}\|\Sigma\|^{-4/3}.
$$
Similar to the proof of Theorem \ref{thm:the-main-for-poly}, we start by recalling that the constant $C(P,\cdot)$ is introduced in Corollary \ref{coro:number}; and is defined in terms of two constants $f(P,\cdot)$ (defined in \ref{dpk}) and $c(P,\cdot)$ (defined in \ref{cpk}). The relations between these constants we use below are recorded in Remark \ref{remark:on-order-between-c-C-f}.

Using $\|\Sigma\|=\sigma_{\max}(\Sigma)\geqslant \sigma_{\min}(\Sigma)$, the observation that $\Sigma=\E{\mathcal{X}\mathcal{X}^t}$ for $\mathcal{X}$ defined in Theorem \ref{cond:main_2}, the  bound (\ref{minEgein2}), and (\ref{eq:c-card-main-thm}); we arrive at:
\begin{equation}\label{eq:aux-sigmamin-bd2}
\sigma_{\min}(\Sigma)\geqslant c(P,M)\frac{1}{\sum_{i=0}^{M}\binom{d+i-1}{i}}\geqslant c(P,M)d^{-2M}>0.
\end{equation}
Consequently,
$$
\|\Sigma\|^{-4/3}\leqslant \sigma_{\min}(\Sigma)^{-4/3}\leqslant c(P,M)^{-4/3}d^{\frac{8}{3}M}.
$$
This, together with (\ref{eq:c-card-main-thm}), then yields
$$
|\mathcal{C}|^{4/3}\|\Sigma\|^{-4/3}\leqslant c(P,M)^{-4/3} d^{\frac{16}{3}M}.
$$
Using now Remark \ref{remark:on-order-between-c-C-f} we have  $C(P,M)\geqslant c(P,M)^{-1}$. Thus,
$$
|\mathcal{C}|^{4/3}\|\Sigma\|^{-4/3}\leqslant C(P,M)^{4/3} d^{\frac{16}{3}M}
$$
Remark \ref{remark:on-order-between-c-C-f} further yields $C(P,M)\geqslant 1$. Since $N$ enjoys (\ref{eq:main-sample-complexity-forthm}), we deduce immediately that $N>\ell^{4/3}\|\Sigma\|^{-4/3}$.  Thus, $\gamma=\|\Sigma\|^{1/2}\delta^2$.

We next claim
$$
\frac12\sigma_{\min}(\Sigma)>\gamma=\|\Sigma\|^{1/2}N^{-3/8}|\mathcal{C}|^{1/2}.
$$
This is equivalent to establishing
$$
N>2^{8/3}\frac{|\mathcal{C}|^{4/3}\|\Sigma\|^{4/3}}{\sigma_{\min}(\Sigma)^{8/3}}
$$
Note that Equations   (\ref{maxEgein2}) and (\ref{eq:c-card-main-thm}) together  yield
\begin{equation}\label{eq:uppa}
\|\Sigma\|\leqslant f(P,M)\sum_{i=0}^{M}\binom{d+i-1}{i}\leqslant f(P,M)d^{2M}.
\end{equation}
Now, (\ref{eq:c-card-main-thm}), (\ref{eq:aux-sigmamin-bd2}), and (\ref{eq:uppa}) together yield
$$
\frac{\ell^{4/3}\|\Sigma\|^{4/3}}{\sigma_{\min}(\Sigma)^{8/3}}<d^{\frac83 M} \frac{f(P,M)^{4/3}d^{\frac83 M}}{c(P,M)^{8/3}d^{-\frac{16}{3} M}} = \frac{f(P,M)^{4/3}}{c(P,M)^{8/3}}d^{\frac{32}{3}M}.
$$
Using now $\max\left(f(P,M),c^{-1}(P,M)\right)\leqslant C(P,M)$ as noted in Remark \ref{remark:on-order-between-c-C-f}, we further obtain
$$
\frac{\ell^{4/3}\|\Sigma\|^{4/3}}{\sigma_{\min}(\Sigma)^{8/3}}<C(P,M)^4 d^{\frac{32}{3}M}.
$$
Since $C(P,M)\geqslant 1$ by Remark \ref{remark:on-order-between-c-C-f} and $N$ enjoys (\ref{eq:main-sample-complexity-forthm}), we therefore conclude
$$
N> \frac{\ell^{4/3}\|\Sigma\|^{4/3}}{\sigma_{\min}(\Sigma)^{8/3}},
$$
and consequently
\begin{equation}\label{eq:sigma-min-admissible-mainthm}
\frac12\sigma_{\min}(\Sigma)>\gamma.
\end{equation} 
We now apply Theorem \ref{thm:vershy}. With probability at least $1-|\mathcal{C}|e^{-cN^{1/4}}$, it holds that:
    $$
    \left\|\frac1N \Xi^T \Xi - \Sigma\right\|\leqslant \gamma.
    $$
    Here, $c>0$ is an absolute constant. Now,
    $$
     \left\|\frac1N \Xi^T \Xi - \Sigma\right\|\leqslant \gamma \iff \forall v\in\R^{|\mathcal{S}|}, \left|\|\frac{1}{\sqrt{N}}\Xi v\|_2^2 - v^T \Sigma v \right|\leqslant \gamma\|v\|_2^2,
    $$
    which implies, for every $v$ on the sphere $\mathbb{S}^{|\mathcal{S}|-1}=\{v\in\mathbb{S}^{|\mathcal{S}}:\|v\|_2=1\}$,
    $$
    \frac1N \|\Xi v\|_2^2 \geqslant v^T \Sigma v-\gamma \Rightarrow \frac1N \inf_{v:\|v\|=1}\|\Xi v\|_2^2 \geqslant \inf_{v:\|v\|=1}v^T \Sigma v - \gamma.
    $$
    Now, using the Courant-Fischer variational characterization of the smallest singular value \cite{horn2012matrix} and   (\ref{eq:sigma-min-admissible-mainthm}), we obtain
    \begin{equation}\label{eq:dddddd}
        \sigma_{\min}(\Xi) \geqslant N(\sigma_{\min}(\Sigma)-\gamma)>\frac{N}{2}\sigma_{\min}(\Sigma),
    \end{equation}
    with probability at least $1-|\mathcal{C}|e^{-cN^{1/4}}=1-\exp(-c'N^{1/4})$, where $c'>0$ is a positive  absolute constant smaller than $c$.
We now set:
$$
\widehat{\mathcal{A}} = (\Xi^T \Xi)^{-1}\Xi^T Y
$$
which is the ``output" of the T-OLS algorithm (namely, $\widehat{\mathcal{A}}$ is the ordinary least squares (OLS) estimator). Note that we have by (\ref{eq:MAINMAINMAIN})
\begin{equation}\label{eq:widehat-A-main-thm}
\widehat{\mathcal{A}}= \mathcal{A}^* + (\Xi^T \Xi)^{-1}\Xi^T \mathcal{N}.
\end{equation}
Here, we denote
$$
\mathcal{\widehat{A}} = (\widehat{q}_\alpha:\alpha\in\mathcal{C}).
$$
Now, let $X^{({\rm fr})}\in\R^d$ be a (fresh) sample. Our goal is then to control the generalization error:
$$
\mathbb{E}_{X^{({\rm fr})}}\left[\left|f(\mathcal{W}^*,X^{({\rm fr})}) - \sum_{\alpha\in\mathcal{C}}\widehat{q}_\alpha X_\alpha^{({\rm fr})}\right|^2 \right] = \mathbb{E}_{X^{({\rm fr})}}\left[\left|\sum_{\alpha\in\mathcal{C}}q_\alpha^* X_\alpha^{(fr)}+\mathcal{N}^{({\rm fr})}- \sum_{\alpha\in\mathcal{C}}\widehat{q}_\alpha X_\alpha^{({\rm fr})}\right|^2 \right].
$$
Here, we take the expectations with respect to a ``new sample" (highlighted with a superscript) drawn from the same distribution generating the training data $X^{(i)}$, $1\leqslant i\leqslant N$. 

Now, define $\mathcal{X}^{({\rm fr})} = (X_\alpha^{({\rm fr})}:\alpha\in\mathcal{C})\in\R^{|\mathcal{C}|}$ where $$
X_\alpha^{{\rm (fr)}}=(\prod_{j=1}^d X_j^{{\rm (fr)}})^{\alpha_j},$$
and note $\sum_{\alpha\in\mathcal{C}} q_\alpha^* X_\alpha^{(fr)} = \ip{\mathcal{A}^*}{\mathcal{X}^{({\rm fr})}}$. In particular, the goal is to bound:
$$
\mathbb{E}_{X^{({\rm fr})}}\left[\left(\ip{\mathcal{A}^*-\widehat{\mathcal{A}}}{\mathcal{X}^{({\rm fr})}} +\mathcal{N}^{(fr)} \right)^2 \right].
$$
Observe first that it holds
\begin{equation}\label{eq:error-bd}
\mathbb{E}_{X^{({\rm fr})}}\left[\left(\ip{\mathcal{A}^*-\widehat{\mathcal{A}}}{\mathcal{X}^{({\rm fr})}} +\mathcal{N}^{(fr)} \right)^2 \right] \leqslant  2\mathbb{E}_{X^{({\rm fr})}}\left[\ip{\mathcal{A}^*-\widehat{\mathcal{A}}}{\mathcal{X}^{({\rm fr})}}^2 \right]+2\mathbb{E}_{X^{({\rm fr})}}\left[\left(\mathcal{N}^{({\rm fr})}\right)^2\right],
\end{equation}
by the elementary inequality, $2(a^2+b^2)\geqslant (a+b)^2$; 
where, the first term, 
$$
\mathbb{E}_{X^{({\rm fr})}}\left[\ip{\mathcal{A}^*-\widehat{\mathcal{A}}}{\mathcal{X}^{({\rm fr})}}^2 \right],
$$
is the {\em training error} which we will establish can be driven to zero having accessed to sufficiently many samples; and $\E{(\mathcal{N}^{({\rm fr})})^2}$ is  {\em approximation error} term, which is well-controlled from the choice of the approximating polynomial.

Using (\ref{eq:widehat-A-main-thm}), we have
$$
\mathbb{E}_{X^{({\rm fr})}}\left[\ip{\widehat{\mathcal{A}}-\mathcal{A}^*}{\mathcal{X}^{({\rm fr})}}\right]=\mathbb{E}_{X^{({\rm fr})}}\left[\ip{(\Xi^T \Xi)^{-1}\Xi^T \mathcal{N}}{\mathcal{X}^{({\rm fr})}}^2\right].
$$
Now we let 
\begin{equation}\label{eq:vii}
v=(\Xi^T \Xi)^{-1}\Xi^T \mathcal{N},
\end{equation}
for convenience, and obtain:
$$
\mathbb{E}_{X^{({\rm fr})}}\left[\ip{(\Xi^T \Xi)^{-1}\Xi^T \mathcal{N}}{\mathcal{X}^{({\rm fr})}}^2 \right]\leqslant \|v\|_2^2 \sup_{x\in\R^{|\mathcal{C}|}:\|x\|_2=1} \mathbb{E}_{X^{({\rm fr})}}\left[\ip{v}{\mathcal{X}^{({\rm fr})}}^2\right] = \|v\|_2^2 \cdot \sigma_{\max}(\Sigma),
$$
where 
$$
\Sigma = \mathbb{E}[\mathcal{X}^{({\rm fr})}(\mathcal{X}^{({\rm fr})})^T]\in\mathbb{R}^{|\mathcal{C}|\times|\mathcal{C}|},
$$
using the Courant-Fischer characterization of the largest eigenvalue, together with the fact that the singular values of a positive semidefinite matrix coincide with its eigenvalues \cite{horn2012matrix}.

We now bound $\|v\|_2$. Note first that, since the matrix norms are submultiplicative, we have:
$$
\|v\|_2 = \|(\Xi^T \Xi)^{-1}\Xi^T \mathcal{N}\|_2 \leqslant \|(\Xi^T \Xi)^{-1}\|_2 \|\Xi^T \mathcal{N}\|_2.
$$

\paragraph{Bounding $\|\Xi^T \mathcal{N}\|_2$:} Fix an $\alpha\in\mathcal{C}$, and note that, $(\Xi^T \mathcal{N})_\alpha = \sum_{j=1}^N X_\alpha^{(j)} \mathcal{N}^{(j)}$. Recall that $\E{X_\alpha^{(j)}\mathcal{N}^{(j)}} = 0$ where the expectation is with respect to the randomness in $j^{th}$ sample, $X^{(j)}\in\R^d$, and thus for each $\alpha\in\mathcal{C}$, $(\Xi^T \mathcal{N})_\alpha$ is a sum of $N$ i.i.d.\ zero-mean random variables. 
    
   We now claim that for any fixed $\alpha\in\mathcal{C}$ the family $(X_\alpha \mathcal{N})_{\alpha\in\mathcal{C}}$ of random variables are almost surely bounded. More precisely, we claim:
\begin{claim}\label{claim:bdd-X_N}
$$
|X_\alpha \mathcal{N}| \leqslant 1+|\mathcal{C}|\frac{1}{\sigma_{\min}(\Sigma)},
$$
almost surely, where
$$
\Sigma = \E{\mathcal{X}^{({\rm fr})}(\mathcal{X}^{({\rm fr})})^T}\in \R^{|\mathcal{C}|\times|\mathcal{C}|}.
$$
\end{claim}
\begin{proof}{(of Claim \ref{claim:bdd-X_N})}

To establish this, observe that since $X\in[-1,1]^d$, it holds that $|X_\alpha|\leqslant 1$; thus it suffices to ensure
$$
|\mathcal{N}|\leqslant 1+|\mathcal{C}|\sigma_{\min}(\Sigma)^{-1}.
$$
Recall now the ``orthogonalization" procedure as per Proposition \ref{prop:ortho-poly}:  $f(X)$ in the statement of Proposition \ref{prop:ortho-poly} in our case is the function that the network computes, that is, $f(\mathcal{W}^*,X)$ (given in (\ref{eq:deepnn-main})).

Notice now that $|X_\alpha f(\mathcal{W}^*,X))|\leqslant 1$ almost surely since $\|a^*\|_{\ell_1}\leqslant 1$, and $\sup_{x\in[-1,1]}|\sigma(x)|\leqslant 1$ (since $|\sigma(x)|\leqslant 1$ on $[-1,1]$ per Definition \ref{def:admissible-activation}). 

Now the vector $\mathcal{A}^*$ of coefficients  is obtained by solving the linear system $\Sigma \mathcal{A}^* =F$ where $F=(\E{X_\alpha f(\mathcal{W}^*,X))}:\alpha\in\mathcal{C})\in\R^{|\mathcal{C}|}$,  and by the discussion above $\|F\|_\infty\leqslant 1$. Consequently, $\|F\|_2\leqslant |\mathcal{C}|^{1/2}$. With this,
$$
\mathcal{A}^* = \Sigma^{-1}F \Rightarrow \|\mathcal{A}^*\|\leqslant \|\Sigma^{-1}\|\cdot \|F\| \leqslant \frac{1}{\sigma_{\min}(\Sigma)}|\mathcal{C}|^{1/2}.
$$
Finally since 
$$
\mathcal{N} =f(X) - \sum_{\alpha\in\mathcal{C}}q_\alpha^* X_\alpha,
$$
we have
$$
|\mathcal{N}|\leqslant  |f(X)|+\left|\sum_{\alpha\in\mathcal{C}}q_\alpha^* X_\alpha\right|\leqslant 1+|\mathcal{C}|\sigma_{\min}(\Sigma)^{-1},
$$
using triangle inequality. Thus the proof of Claim \ref{claim:bdd-X_N} is complete.
\end{proof}

   We now use Theorem \ref{thm:vershy-2} with \begin{equation}\label{eq:K-upper-bd-thmmain}
       K=2|\mathcal{C}|\sigma_{\min}(\Sigma)^{-1}.
       \end{equation} 
   We have: 
    $$
    \mathbb{P}\left[\left|\sum_{i=1}^N \frac{1}{N^{11/12}}X_\alpha^{(i)}\mathcal{N}^{(i)}\right|\geqslant 1\right]\leqslant e\cdot \exp\left(-c'\frac{N^{5/12}}{4|\mathcal{C}|^2 \sigma_{\min}(\Sigma)^{-2}}\right),
    $$
    with $c'$ being some absolute constant. 
    Using now the union bound over the coordinates of $\Xi^T \mathcal{N}\in\R^{|\mathcal{C}|}$, we obtain
    $\|\Xi^T \mathcal{N} \|_\infty \leqslant N^{11/12}$, and thus $
    \|\Xi^T \mathcal{N}\|_2 \leqslant N^{11/12}|\mathcal{C}|$ with probability at least
    \begin{equation}\label{eq:hpbp}
    1-|\mathcal{C}|e \exp\left(-c'\frac{N^{5/12}}{4|\mathcal{C}|^2 \sigma_{\min}(\Sigma)^{-2}}\right).
    \end{equation}
 We  now claim
 $$
 N^{1/6}>\frac14\frac{|\mathcal{C}|^2}{\sigma_{\min}(\Sigma)^2}.
 $$
 To  see this, recall first that we have $|\mathcal{C}|\leqslant d^{2M}$ by (\ref{eq:c-card-main-thm}).  Now, (\ref{eq:aux-sigmamin-bd2}) yields
 $$
 \sigma_{\min}(\Sigma)^{-2}\leqslant c(P,M)^{-2}d^{-4M}.
 $$
 Using now $c(P,M)^{-1}\leqslant C(P,M)$ as noted in Remark \ref{remark:on-order-between-c-C-f}, and the bounds above, we conclude
 $$
 \frac14 \frac{|\mathcal{C}|^2}{\sigma_{\min}(\Sigma)^2}\leqslant \frac14 d^{8M}C(P,M)^{2}.
 $$
 Since $N$ enjoys (\ref{eq:main-sample-complexity-forthm}),  we indeed have
 \begin{equation}\label{eq:simplify-whpbd}
 N^{1/6}>\frac14 \frac{|\mathcal{C}|^2}{\sigma_{\min}(\Sigma)^2}.
  \end{equation}
Using (\ref{eq:simplify-whpbd}), the high probability bound in (\ref{eq:hpbp}) simplifies, and we obtain that for some absolute constant $c''>0$, it holds that with probability at least $1-\exp\left(-c''N^{1/4}\right)$,
\begin{align}\label{eq:aha-bu-bir}
    \|\Xi^T \mathcal{N}\|_2 \leqslant N^{11/12}|\mathcal{C}|.
\end{align} 
\paragraph{Bounding $\|(\Xi^T \Xi)^{-1}\|_{2}$:} Let $A$ be any matrix $A$. Note that, $\|A^{-1}\|= \sigma_{\min}(A)^{-1}$. Indeed, taking the singular value decomposition $A=U\Sigma V^T$, and observing, $A^{-1}=(V^T)^{-1}\Sigma^{-1} U^{-1}$ we obtain $\|A^{-1}\| = \max_i (\sigma_i(A))^{-1} = \sigma_{\min}(A)^{-1}$. This, together with (\ref{eq:dddddd}), yields:
   \begin{equation}\label{eq:aha-bu-iki}
    \|(\Xi^T \Xi)^{-1}\| \leqslant \frac{2}{N\sigma_{\min}(\Sigma)},
      \end{equation}
    with probability at least $1-\exp(-c'N^{1/4})$.
    
We now combine the assertions of Equations (\ref{eq:aha-bu-bir}) and (\ref{eq:aha-bu-iki}). It holds that, 
\begin{align*}
    \|(\Xi^T \Xi)^{-1}\Xi^T \mathcal{N}\|_2^2 &\leqslant \|(\Xi^T \Xi)^{-1}\|_2^2 \cdot \|\Xi^T \mathcal{N}\|_2^2 \\ &\leqslant \underbrace{\frac{4}{N^2\sigma_{\min}(\Sigma)^2}}_{\text{ from } (\ref{eq:aha-bu-iki})}\cdot \underbrace{N^{11/6}|\mathcal{C}|^2}_{\text{ from } (\ref{eq:aha-bu-bir})}  = \frac{4}{N^{1/6}}\frac{|\mathcal{C}|^2}{\sigma_{\min}(\Sigma)^2}.
\end{align*}
For $v$ is defined in (\ref{eq:vii}), we thus have
\begin{align*}
\mathbb{E}_{X^{({\rm fr})}}\left[\ip{\mathcal{A}^*-\widehat{\mathcal{A}}}{\mathcal{X}^{(fr)}}^2 \right]=\mathbb{E}_{X^{({\rm fr})}}\left[\ip{(\Xi^T \Xi)^{-1}\Xi^T \mathcal{N}}{\mathcal{X}^{({\rm fr})}}^2 \right]
&\leqslant \|v\|_2^2 \sup_{x\in\R^{|\mathcal{C}|}:\|x\|_2=1} \underbrace{\mathbb{E}_{X^{(fr)}}\left[\ip{v}{\mathcal{X}^{(fr)}}^2\right]}_{=\sigma_{\max}(\Sigma) }\\
&\leqslant \frac{4}{N^{1/6}}\frac{\sigma_{\max}(\Sigma)|\mathcal{C}|^2}{\sigma_{\min}(\Sigma)^2},
\end{align*}
with probability at least
$$
1-\exp(-c'N^{1/4})-\exp(-c''N^{1/4}) = 1-\exp(-cN^{1/4}),
$$
for some absolute constant $c>0$.
We now claim
$$
\frac{4}{N^{1/6}}\frac{\sigma_{\max}(\Sigma)|\mathcal{C}|^2}{\sigma_{\min}(\Sigma)^2}\leqslant \frac{\epsilon}{4}.
$$
This is equivalent to showing
$$
N>\frac{2^{24}}{\epsilon^6}|\mathcal{C}|^{12}  \frac{\sigma_{\max}(\Sigma)^6}{\sigma_{\min}(\Sigma)^{12}}.
$$
Now, $|\mathcal{C}|^{12} \leqslant d^{24M}$ by (\ref{eq:c-card-main-thm}). Furthermore, $\sigma_{\max}(\Sigma)^6\leqslant f(P,M)^6 d^{12M}$ by (\ref{eq:uppa})  and $\sigma_{\min}(\Sigma)^{-12}\leqslant c(P,M)^{-12}d^{24M}$ by (\ref{eq:aux-sigmamin-bd2}). These, together with the fact $\max(f(P,M),c^{-1}(P,M))\leqslant C(P,M)$, as per Remark \ref{remark:on-order-between-c-C-f}, yield that
$$
\frac{2^{24}}{\epsilon^6}|\mathcal{C}|^{12}  \frac{\sigma_{\max}(\Sigma)^6}{\sigma_{\min}(\Sigma)^{12}}< \frac{2^{24}}{\epsilon^6}d^{24M}f(P,M)^6 d^{12M} c(P,M)^{-12}d^{24M}\leqslant \frac{2^{24}}{\epsilon^6} C^{18}(P,M)d^{60M}.
$$ 
Since $N$ enjoys (\ref{eq:main-sample-complexity-forthm}), we deduce
$$
N>\frac{2^{24}}{\epsilon^6}|\mathcal{C}|^{12}  \frac{\sigma_{\max}(\Sigma)^6}{\sigma_{\min}(\Sigma)^{12}},
$$
and consequently
$$
\frac{4}{N^{1/6}}\frac{\sigma_{\max}(\Sigma)|\mathcal{C}|^2}{\sigma_{\min}(\Sigma)^2}\leqslant \frac{\epsilon}{4}.
$$
Thus, with probability at least $1-\exp(-c'N^{1/4})$ over the randomness in $\{X^{(i)}:1\leqslant i\leqslant  N\}$, it holds:
$$
\mathbb{E}_{X^{({\rm fr})}}\left[\ip{\mathcal{A}^*-\widehat{\mathcal{A}}}{\mathcal{X}^{(fr)}}^2 \right]\leqslant \frac{\epsilon}{4}.
$$
Furthermore,
$$
\mathbb{E}_{X^{({\rm fr})}}\left[\left(\mathcal{N}^{({\rm fr})}\right)^2\right]\leqslant \epsilon/4
$$
from the way the multiplicities set $\mathcal{C}$ is constructed. Finally, using (\ref{eq:error-bd}), we deduce that   with probability at least $1-\exp(-c'N^{1/4})$ over the randomness of $\{X^{(i)}:1\leqslant i\leqslant N\}$ it is true that
$$
\mathbb{E}_{X^{({\rm fr})}}\left[\left|f(\mathcal{W}^*,X^{({\rm fr})}) - \sum_{\alpha\in\mathcal{C}}\widehat{q}_\alpha X_\alpha^{({\rm fr})}\right|^2 \right]=\mathbb{E}_{X^{({\rm fr})}}\left[\left(\ip{\mathcal{A}^*-\widehat{\mathcal{A}}}{\mathcal{X}^{({\rm fr})}} +\mathcal{N}^{(fr)} \right)^2 \right]\leqslant \epsilon.
$$
This concludes the proof of Theorem \ref{thm:the-main}.
\end{proof}

\subsection{Proofs of Generalization  Results}

\subsubsection{Proof of Theorem \ref{thm:overly-general-poly}}\label{sec:pf-thm:overly-general-poly}
\begin{proof}
Such interpolating weights indeed exist, since $\widehat{m}\geqslant m^*$. The proof is almost identical to that of Theorems \ref{thm:demyst-overparam-regular}, \ref{thm:overparam:ReLU-case}, hence we only provide a very brief sketch.
\begin{itemize}
    \item[(a)] In this case, we are in the setting of ${{\rm (a)}}$: we construct $\widehat{h}(\cdot)$, as in Theorem \ref{thm:the-main-for-poly}, and notice that with probability at least $1-\exp(-c'N^{1/4})$, it holds:
    $$
    \widehat{h}(x) = f_{\mathcal{N}_1}(\mathcal{W}^*,x),
    $$
    for every $x\in\R^d$.
    Furthermore, interpolation also implies $Y_i=f_{\mathcal{N}_2}(\widehat{\mathcal{W}},X^{(i)})$. Running now Theorem \ref{thm:the-main-for-poly} ${\rm (a)}$, this time however treating as if the data comes from $\mathcal{N}_2$, we then deduce the same estimator $\widehat{h}(\cdot)$ constructed above satisfies  also:
    $$
    \widehat{h}(x) = f_{\mathcal{N}_2}(\widehat{\mathcal{W}},x),
    $$
    for every $x\in\R^d$, with probability at least $1-\exp(-c'N^{1/4})$. Combining these, and using union bound, we then obtain that with probability at least
    $$
    1-2\exp(-c'N^{1/4}),
    $$
    over the randomness of $\{X^{(i)}:1\leqslant i\leqslant N\}$, it holds:
    $$
    f_{\mathcal{N}_1}(\mathcal{W}^*,x)=f_{\mathcal{N}_2}(\mathcal{\widehat{W}},x),
    $$
    for every $x\in\R^d$.
    \item[(b)] Under joint continuity assumption on the entries of $X^{(i)}$ we now are in the setting of Theorem \ref{thm:the-main-for-poly} ${\rm (b)}$, from which the conclusion follows using the exact same outline as  above.
\end{itemize}
\end{proof}

\subsubsection{Proof of Theorem \ref{thm:demyst-overparam-regular}}\label{sec:pf-main-generalization}
\begin{proof}
We first establish that the global optimum of the empirical risk minimization problem being considered is zero: we construct a $\widehat{\mathcal{W}}$ such that:
\begin{itemize}
    \item[(a)] $\widehat{a}\in\R^{\widehat{m}}$ satisfy $\widehat{a}_i = a_i^*$ for every $i\in[m]$, and $\widehat{a}_i = 0$  for $m+1\leqslant i\leqslant \widehat{m}$. 
    \item[(b)] Fix any $k\in[L]$, and consider the weight matrix $\widehat{W}_k\in\R^{\widehat{m}\times \widehat{m}}$. Now, for every $i,j\in[m^*]$, set
    $$
    (\widehat{W}_k)_{i,j}=  (W^*_k)_{i,j}.
    $$
    All  the remaining entries are set to zero.
\end{itemize}
This network interpolates the data (and in fact, it essentially turns off the $\widehat{m}-m$ many nuisance nodes per layer), that is,
$$
Y_i = f_{\mathcal{N}_2}(\widehat{\mathcal{W}},X_i),
$$
for every $1\leqslant i\leqslant N$. Hence, any optimum of the empirical risk minimization problem must necessarily have zero cost. 


 Let $\epsilon>0$ be a target accuracy level. We now apply the setting of Theorem \ref{thm:the-main}, with $M=\varphi_\sigma(\sqrt{\epsilon^2/16L^2})^L\in\mathbb{Z}^+$, to both networks, and in particular construct a $(d,M)-$multiplicities set $\mathcal{C}$. We then construct a matrix $\Xi$, like in the proof of Theorem \ref{thm:the-main}, and construct the ordinary least squares estimator $\widehat{h}(\cdot)$, that is:
$$
\widehat{h}(X) = \left(\Xi^T \Xi \right)^{-1} \Xi^T Y \mathcal{X},
$$
where for any $X\in\R^d$, $\mathcal{X}$ corresponds to its ``tensorized" version,
$$
\mathcal{X}_\alpha = \prod_{j=1}^d (X_j)^{\alpha_j},
$$
for every $\alpha=(\alpha_1,\dots,\alpha_d)\in\mathcal{C}$. Note now that the target accuracy $\epsilon$, depth $L$, as well as the dimension $d$ of data are the same for both the ``teacher" network $\mathcal{N}_1$,  and the ``student" network $\mathcal{N}_2$. Thus we deduce
$$
\E{\left(\widehat{h}(X) - f_{\mathcal{N}_1}(\mathcal{W}^*,X)\right)^2}\leqslant \frac{\epsilon}{4},
$$
with probability at least $1-\exp(-c'N^{1/4})$; and
$$
\E{\left(\widehat{h}(X) - f_{\mathcal{N}_2}(\widehat{\mathcal{W}},X)\right)^2}\leqslant \frac{\epsilon}{4},
$$
again with probability at least $1-\exp(-c'N^{1/4})$, provided $N$ enjoys the sample complexity bound (\ref{eq:sample-cpmplex-4}). 
Now, triangle inequality yields
$$
|f_{\mathcal{N}_1}(\mathcal{W}^*,X) - f_{\mathcal{N}_2}(\widehat{\mathcal{W}},X)|\leqslant |f_{\mathcal{N}_1}(\mathcal{W}^*,X)-\widehat{h}(X)|+|f_{\mathcal{N}_2}(\widehat{\mathcal{W}},X)-\widehat{h}(X)|.
$$
Combining this, with the elementary inequality $2(a^2+b^2)\geqslant (a+b)^2$ and union bound, we obtain that with probability at least $1-2\exp(-c'N^{1/4})$, it holds
$$
 \E{\left(f_{\mathcal{N}_1}(\mathcal{W}^*,X) - f_{\mathcal{N}_2}(\widehat{\mathcal{W}},X)\right)^2}\leqslant \epsilon.
$$
Finally, replacing $c'$ be a smaller positive constant $c''$ we conclude that the same bound holds with probability at least $1-\exp(-c''N^{1/4})$. 

\end{proof}

\subsubsection{Proof of Theorem \ref{thm:overparam:ReLU-case}}\label{sec:proof-of-overparam:ReLU-case}
\begin{proof}
Recall that the weights of the ``teacher" network are assumed to satisfy
$$
\|W_j^*\|_{\ell_1}\leqslant 1\quad\text{and}\quad \|a^*\|_{\ell_1}\leqslant 1.
$$
Furthermore, by Theorem \ref{thm:interpolant-norm-poly-bd}, we may also suppose that the weights of the interpolating wider ``student" network satisfy:
\begin{equation}\label{eq:auxilarr} 
\|\widehat{a}\|_{\ell_1} \leqslant d^{\kappa +1}2^{\kappa+1}\|a^*\|_{\ell_1}\leqslant d^{\kappa +1}2^{\kappa+1} \triangleq B_d,
\end{equation}
and
$$
\|\widehat{W}_j\|_{\ell_2}=1/\sqrt{d}\Rightarrow \|\widehat{W}_j\|_{\ell_1}\leqslant 1,\quad \forall j\in[\widehat{m}],
$$
with probability $1-\exp(-cN^{1/3})$, provided $N>\exp(3d\log d)$. Now, let $\epsilon>0$ be the target level of error. The rest of the proof is similar to the proof of Theorem \ref{thm:demyst-overparam-regular}. 

We apply the setting of Theorem \ref{thm:the-main} with one modification: note that by Theorem \ref{thm:interpolant-norm-poly-bd}, $\|\widehat{a}\|_{\ell_1}\leqslant B_d$ with high probability. 

Now, if the approximation error (obtained by replacing the activation with an appropriate polynomial) is $\delta$ at each node, then the overall error at the output will be $\delta\|\widehat{a}\|_{\ell_1}$. To control this, we ensure each node makes an error of at most $\sqrt{\epsilon}/2B_d$. For this, it suffices to select $M=\varphi_\sigma(\sqrt{\epsilon/32B_d^2})\in\mathbb{Z}^+$, where $\varphi_\sigma(\cdot)$ is the order of approximation (per Definition \ref{def:approximation-order}), for activation $\sigma(\cdot)$, where the choice of constant $32$ is for convenience. 

Having selected this, we now proceed, in the exact same way, as in the proof of Theorem \ref{thm:demyst-overparam-regular}, and obtain that there is an estimator $\widehat{h}(\cdot)$, such that,
$$
\E{\left(\widehat{h}(X) - f_{\mathcal{N}_1}(a^*,W^*,X)\right)^2}\leqslant \frac{\epsilon}{4},
$$
with probability at least $1-\exp(-c'N^{1/4})$; and
$$
\E{\left(\widehat{h}(X) - f_{\mathcal{N}_2}(\widehat{a},\widehat{W},X)\right)^2}\leqslant \frac{\epsilon}{4},
$$
again with probability at least $1-\exp(-c'N^{1/4})$, provided $N$ enjoys the sample complexity bound (\ref{eq:sample-complex-N6}). 
Now, triangle inequality yields
$$
|f_{\mathcal{N}_1}(a^*,W^*,X) - f_{\mathcal{N}_2}(\widehat{a},\widehat{W},X)|\leqslant |f_{\mathcal{N}_1}(a^*,W^*,X)-\widehat{h}(X)|+|f_{\mathcal{N}_2}(\widehat{a},\widehat{W},X)-\widehat{h}(X)|.
$$
Combining this, with the elementary inequality $2(a^2+b^2)\geqslant (a+b)^2$ and union bound, we obtain that with probability at least $1-2\exp(-c'N^{1/4})$ over the randomness of $\{X^{(i)}:1\leqslant i\leqslant N\}$, it holds
$$
 \E{\left(f_{\mathcal{N}_1}(a^*,W^*,X) - f_{\mathcal{N}_2}(\widehat{a},\widehat{W},X)\right)^2}\leqslant \epsilon.
$$
Finally, replacing $c'$ by a smaller positive constant $c''$ we conclude that the same bound holds with probability at least $1-\exp(-c''N^{1/4})$. 
\end{proof}

\subsection{Proof of Theorem \ref{thm:interpolant-norm-poly-bd}}\label{sec:pf-thm-interpolant-norm}
We start with an auxiliary result.
We establish that the set $\{X^{(i)}:1\leqslant i\leqslant N\}$ of random vectors contains elements arbitrarily close to any member of the standard basis of $\R^d$; provided it has sufficient cardinality.
\begin{theorem}\label{thm:random-covering}
Let $X^{(i)}\in\R^d$, $1\leqslant i\leqslant N$ be an i.i.d.\ sequence of random vectors drawn from an admissible distribution, where the coordinates of $X^{(i)}$ have a density bounded away from zero. Let $\mathcal{B}=\{e_i\}_{i=1}^d$ be the standard basis for $\R^d$, $-\mathcal{B}=\{-e_i\}_{i=1}^d$, and $\mathcal{S}=\mathcal{B}\cup (-\mathcal{B})$. Then, 
$$
\mathbb{P}\left(\bigcap_{v\in\mathcal{S}}\bigcup_{i=1}^N\left\{ \|X^{(i)}-v\|_2<\frac{1}{4d}\right\}\right)\geqslant 1-\exp(-cN^{1/3}),
$$ 
for all $d$ large (where $c$ is an absolute constant), provided $N$ satisfies (\ref{eq:sample-complex-N5-star}), that is
$$
N>\exp(3d\log d).
$$
\end{theorem}
\begin{proof}{(of Theorem \ref{thm:random-covering})}
Let $V$ be a vector having the same distribution as the data $X^{(i)}$, where the coordinates of $V$ have a density bounded away from zero by $B$. Define
$$
q_1(\epsilon) = \mathbb{P}\left(\left\|V-e_i\right\|_2<\epsilon\right)\quad \text{and}\quad q_2(\epsilon) = \mathbb{P}\left(\left\|V+e_i\right\|_2<\epsilon\right).
$$
Note that since $V$ has i.i.d. coordinates, the choice of $i$ above is immaterial, hence we suppose $i=1$ for simplicity. Introduce now the non-negative integer-valued random variables $\{N_j:1\leqslant j\leqslant d\}$ and $\{M_j:1\leqslant j\leqslant d\}$, where
$$
N_j =\left|\left\{i\in[n]:\left\|X^{(i)}-e_j\right\|_2<\frac{1}{4d}\right\}\right|\quad\text{and}\quad M_j =\left|\left\{i\in[n]:\left\|X^{(i)}+e_j\right\|_2<\frac{1}{4d}\right\}\right|.
$$
We now observe that
$$
\mathcal{E}\triangleq \bigcap_{v\in\mathcal{S}}\bigcup_{i=1}^N\left\{ \|X^{(i)}-v\|_2<\frac{1}{4d}\right\} = \bigcap_{j=1}^d \{N_j\geqslant 1\}\cap \{M_j\geqslant 1\}.
$$
Note that,
$$
\mathbb{P}(N_j =0) = \left(1-q_1\left(\frac{1}{4d}\right)\right)^N \quad\text{and}\quad \mathbb{P}(M_j=0) = \left(1-q_2\left(\frac{1}{4d}\right)\right)^N.
$$
Observe that using union bound,
$$
\mathbb{P}(\mathcal{E})\geqslant 1-d\left(1-q_1\left(\frac{1}{4d}\right)\right)^N -d\left(1-q_2\left(\frac{1}{4d}\right)\right)^N
$$
We now lower bound $q_1(\cdot)$ and $q_2(\cdot)$. To that end, let the random vector $V$ be $V=(V_1,\dots,V_d)\in\R^d$, and $\epsilon>0$ be a small constant. Then, 
\begin{align*}
    \left\{\left\|V-e_1\right \|_2^2<\epsilon^2\right\}  &= \left\{(V_1-1)^2+\sum_{i=2}^d V_i^2<\epsilon^2\right\}\\
    &\supset \left\{|V_1-1|<\epsilon/\sqrt{d}\right\}\cap \bigcap_{i=2}^d \left\{|V_i|<\epsilon/\sqrt{d}\right\}.
\end{align*}
Observe now that, 
$$
\mathbb{P}\left(|V_1-1|<\epsilon/\sqrt{d}\right)\geqslant B\epsilon/\sqrt{d},
$$
and
$$
\mathbb{P}\left(|V_i|<\epsilon/\sqrt{d}\right)\geqslant 2B\epsilon/\sqrt{d},
$$
for all $i\geqslant 2$. Using the independence of coordinates $V_1,\dots,V_d$, we then obtain $q_1(\epsilon)\geqslant (B\epsilon/\sqrt{d})^d$, therefore,
$$
q_1\left(\frac{1}{4d}\right) \geqslant (\frac{B}{4}d^{-3/2})^d = \exp\left(-\frac32d\log(d)+d\log \frac{B}{4}\right)\geqslant \exp(-2d\log(d)),
$$
for all $d$ large. An exact same analysis reveals also that
$$
q_2\left(\frac{1}{4d}\right)\geqslant \exp(-2d\log(d)),
$$
for $d$ large. Thus,
$$
\mathbb{P}(\mathcal{E}) \geqslant 1-2d\left(1-e^{-2d\log d}\right)^N.
$$
We now control $\mathbb{P}(\mathcal{E})$, provided $N>\exp(3d\log(d))$. Note that in this regime, $e^{2d\log d}=O(N^{2/3})$. Using now Taylor expansion $\log(1-x)=-x+o(x)$, we have:
\begin{align*}
    \mathbb{P}(\mathcal{E})&\geqslant 1-\exp\left(\log(2d)+N\log\left(1-e^{-2d\log d}\right)\right) \\
    &=1-\exp\left(\log(2d)-Ne^{-2d\log d} +o\left(Ne^{-2d\log d}\right)\right) \\
    & = 1-\exp(-cN^{1/3}),
\end{align*}
for some absolute constant $c$, provided that $d$ is large, concluding the proof.
\end{proof}
Having proven Theorem \ref{thm:random-covering}, we are now ready to prove Theorem \ref{thm:interpolant-norm-poly-bd}. 
\begin{proof}{(of Theorem \ref{thm:interpolant-norm-poly-bd})}
We first establish one can indeed assume without loss of generality $\|\widehat{W}_j\|_2=1/\sqrt{d}$, for every $j\in[\widehat{m}]$. Since $\sigma(\cdot)$ is positive homogeneous with constant $\kappa>0$, it holds:
$$
\widehat{a}_j \sigma(\langle \widehat{W}_j,X\rangle) = 
\widehat{a}_j \theta^{-\kappa} \sigma(\langle \theta\widehat{W}_j,X\rangle),
$$
for every $\theta>0$. If $\widehat{W}_j\neq 0$, then set $\theta = d^{-1/2}(\|\widehat{W}_j\|_2)^{-1}$. If $\widehat{W}_j=0$, then set $\widehat{a_j}=0$, and let $\widehat{W}_j$ to be any arbitrary vector with $\ell_2-$norm $d^{-1/2}$. Hence, in the remainder, we assume $\|\widehat{W}_j\|_2 = d^{-1/2}$ for every $j\in[\widehat{m}]$.  


Now, let $\mathcal{B}=\{e_i\}_{i=1}^d$ be the standard basis for $\R^d$, $-\mathcal{B}=\{-e_i\}_{i=1}^d$; and $\mathcal{S}=\mathcal{B}\cup (-\mathcal{B})$, as in the setting of Theorem \ref{thm:random-covering}, and
$$
\mathcal{E} \triangleq \bigcap_{v\in\mathcal{S}} \bigcup_{i=1}^N \left\{\left\|X^{(i)}-v\right\|_2<\frac{1}{4d}\right\}.
$$
Since $N$ satisfies (\ref{eq:sample-complex-N5-star}),  it holds by the Theorem \ref{thm:random-covering} that
$$
\mathbb{P}\left(\mathcal{E}\right)\geqslant 1-\exp(-cN^{1/3}),
$$
for all large $d$. In the remainder assume we are on this high probability event $\mathcal{E}$. 

Observe that, on the event $\mathcal{E}$, there exists $X_{n(1)},\dots,X_{n(2d)}$ with
$$
\{n(1),\dots,n(2d)\}\subset \{1,2,\dots,N\},
$$
such that $\|X_{n(i)}-e_i\|<\frac{1}{4d}$ for each $i$, where $e_i=-e_{i-d}$ for $d+1\leqslant i\leqslant 2d$.

We now establish that for every $j\in[\widehat{m}]$, there exists a $v\in\mathcal{S}$ such that $\langle \widehat{W_j},v\rangle \geqslant  1/d$. To see this, we study the basis expansion $\widehat{W_j} = \sum_{i=1}^d \theta_{ij}e_i$.  Now, we have $1/d=\|\widehat{W_j}\|_{\ell_2}^2 = \sum_{i=1}^d |\theta_{ij}|^2$. Consequently, there is an $i_0$ such that $|\theta_{i_0,j}|\geqslant 1/d$. Thus, 
$$
\max\left\{\ip{\widehat{W_j}}{v_{i_0}},\ip{\widehat{W_j}}{-v_{i_0}}\right\}\geqslant 1/d.
$$
Now, let $v\in\mathcal{S}$ with \begin{equation}\label{eq:v-inner-W_j}
    \inner{\widehat{W_j},v}\geqslant 1/d.
\end{equation} Since we are on the event $\mathcal{E}$, there is a data $X$ for which \begin{equation}\label{eq:X-minus-v}
\|X-v\|_2 \leqslant \frac{1}{4d}.
\end{equation}We next establish a lower bound on  $\inner{\widehat{W_j},X}$. First, observe that using $\|X-v\|\leqslant\frac{1}{4d}$, we have
$$
1=\|v\|\leqslant \|X-v\| + \|X\|,
$$ 
and thus
\begin{equation}\label{eq:norm-bd-on-X}
    \|X\|\geqslant 1-\frac{1}{4d}.
\end{equation}
Next, combining the facts $\|\widehat{W}_j\|_2^2=d^{-1}$, $\|v\|_2=1$, and (\ref{eq:v-inner-W_j}), we obtain
$$
\|\widehat{W_j}-v\|_2^2 = 1/d+1-2\inner{\widehat{W_j},v}\leqslant 1-1/d.
$$
Consequently
\begin{equation}\label{eq:w_j-v}
\|\widehat{W_j}-v\|_2 \leqslant \sqrt{1-1/d}.
\end{equation}
Combining (\ref{eq:X-minus-v})) and (\ref{eq:w_j-v}); and applying 
triangle inequality, we conclude:
$$
\|\widehat{W_j}-X\|\leqslant \frac{1}{4d}+\sqrt{1-1/d}.
$$
We now claim
\begin{equation}\label{eq:keep-in-mind}
\inner{\widehat{W_j},X}\geqslant 1/2d.
\end{equation}
Observe first that,
\begin{equation}\label{eq:helper31}
    \|\widehat{W}_j-X\|_2^2 \leqslant \frac{1}{16d^2}+1-1/d+\frac{1}{2d}\sqrt{1-1/d}.
\end{equation}
Now, we have
\begin{equation}\label{eq:helper3131}
\|\widehat{W}_j-X\|_2^2 = \|\widehat{W}_j\|_2^2+\|X\|_2^2-2\ip{\widehat{W}_j}{X}\geqslant d^{-1}+ (1-\frac{1}{2d}+\frac{1}{16d^2})-2\ip{\widehat{W}_j}{X},
\end{equation}
where we used the inequality (\ref{eq:norm-bd-on-X}).  Combining the inequalities (\ref{eq:helper31}) and (\ref{eq:helper3131}), we obtain
$$
\ip{\widehat{W}_j}{X}\geqslant \frac1d -\frac{1}{4d}-\frac{1}{4d}\sqrt{1-\frac1d}.
$$
Since $\sqrt{1-1/d}<1$, we thus conclude 
$$
\ip{\widehat{W}_j}{X}\geqslant \frac1d-\frac{1}{4d}-\frac{1}{4d}=\frac{1}{2d},
$$
as claimed.

Using now the fact $\|W_j^*\|_{\ell_1}\leqslant 1$, we obtain $|\langle W_j^*,X_i\rangle|\leqslant 1$. This, together with the facts $a_j^*\geqslant 0$, and $\sigma(\cdot)$ being non-decreasing on $[0,\infty)$, then yield,
$$
\sigma(1)\|a^*\|_{\ell_1}\geqslant \sum_{j=1}^{m^*} a_j^* \sigma(\ip{W_j^*}{X_i})=\sum_{j=1}^{\widehat{m}}\widehat{a_j}\sigma(\langle \widehat{W_j},X_i\rangle),
$$
that is,
\begin{equation}\label{eq:useful}
\sigma(1)\|a^*\|_{\ell_1}\geqslant \sum_{j=1}^{\widehat{m}}\widehat{a_j}\sigma(\langle \widehat{W_j},X_i\rangle),
\end{equation}
for every $i\in[N]$.
As recorded in Equation (\ref{eq:keep-in-mind}), on the event $\mathcal{E}$, it holds that for every $j\in[\widehat{m}]$, there exists an $i\in[2d]$, such that
\begin{equation}\label{eq:inner-prod-lower-bd}
\inner{\widehat{W_j},X_{n(i)}}\geqslant 1/2d.
\end{equation}
Combining Equations (\ref{eq:useful}) and (\ref{eq:inner-prod-lower-bd}) and summing over $1\leqslant i\leqslant 2d$, we then have:
$$
2d\cdot \sigma(1)\cdot \|a^*\|_{\ell_1}\geqslant \sum_{i=1}^{2d}\sum_{j=1}^{\widehat{m}}\widehat{a_j}\sigma(\inner{\widehat{W_j},X_{n(i)}}) = \sum_{j=1}^{\widehat{m}} \widehat{a_j}\underbrace{\left(\sum_{i=1}^{2d}\sigma(\inner{\widehat{W_j},X_{n(i)}})\right)}_{\geqslant \sigma(1/2d)},
$$
and thus,
$$
\|\widehat{a}\|_{\ell_1}\leqslant 2d \cdot 
\sigma(1)\cdot \sigma(1/2d)^{-1}\|a^*\|_{\ell_1}.
$$
Finally, using the fact that $\sigma(\cdot)$ is positive homogeneous with constant $\kappa>0$, we have:
$$
 \sigma(1/2d)  = 2^{-\kappa}d^{-\kappa}\sigma(1).
$$
Hence, we obtain that with probability at least $1-\exp(-cN^{1/3})$ over $X^{(i)}$, $1\leqslant i\leqslant N$ (where $c>0$ is an absolute constant), it holds:
$$
\|\widehat{a}\|_{\ell_1}\leqslant d^{\kappa+1}2^{\kappa+1}\|a^*\|_{\ell_1}.
$$
\end{proof}

\bibliographystyle{amsalpha}
\bibliography{bibliography}

\end{document}